\documentclass{article}

\PassOptionsToPackage{numbers, compress}{natbib}
\usepackage[main, final]{neurips_2026_arxiv}

\usepackage[utf8]{inputenc} 
\usepackage[T1]{fontenc}    
\usepackage{hyperref}       
\usepackage{url}            
\usepackage{booktabs}       
\usepackage{amsfonts}       
\usepackage{nicefrac}       
\usepackage{microtype}      
\usepackage{xcolor}         

\usepackage{amsmath}
\usepackage{amssymb}
\usepackage{mathtools}
\usepackage{amsthm}
\usepackage{xspace}
\usepackage{enumitem}
\usepackage{subcaption}
\usepackage{wrapfig}
\usepackage{microtype}
\usepackage{graphicx}
\usepackage{booktabs}
\usepackage[textsize=tiny]{todonotes}
\usepackage{algorithm}
\usepackage{algorithmic}

\usepackage[capitalize,noabbrev]{cleveref}
\theoremstyle{plain}
\newtheorem{theorem}{Theorem}[section]

\theoremstyle{definition}

\theoremstyle{remark}

\hypersetup{
	colorlinks=true,
	linkcolor=blue,
	urlcolor=orange,
	citecolor=blue,
}

\crefname{proposition}{prop.}{props.}
\Crefname{proposition}{Proposition}{Propositions}

\DeclareMathOperator*{\argmin}{arg\,min}

\newcommand{\prvm}{\textsc{Svm}\xspace}
\newcommand{\dsrl}{\textsc{Dsrl}\xspace}
\newcommand{\resrl}{\textsc{Residual RL}\xspace}
\newcommand{\fhat}{\widehat{f}}
\newcommand{\algname}{\textsc{Svm}\xspace}
\newcommand{\arxiv}[1]{}
\newcommand{\loose}{\looseness=-1}

\newtoggle{arxiv}
\toggletrue{arxiv}

\title{Learning Process Rewards via Success Visitation Matching for Efficient RL}

\author{%
  Raymond Tsao\thanks{Correspondence to: Raymond Tsao \{\texttt{r112358@berkeley.edu}\} and Andrew Wagenmaker \{\texttt{ajwagen@berkeley.edu}\}} \\
  UC Berkeley
  \And
  Andrew Wagenmaker \\
  UC Berkeley 
  \And
  Sergey Levine \\
  UC Berkeley 
}

\begin{document}

\maketitle

\begin{abstract}
  In many modern applications of reinforcement learning (RL), the natural reward for a task of interest is inherently \emph{sparse}: a reward of 0 is given everywhere except when the task is completed, when a reward of +1 is given. Training a policy to maximize such a sparse reward requires solving a challenging credit assignment problem, leading to slow or ineffective RL improvement. We propose a simple approach to transform a sparse \emph{outcome} reward into a dense \emph{process} reward. Our approach relies on training a discriminator to distinguish between previous successful and unsuccessful episodes, and using this discriminator to incentivize the RL-learned policy to match the state-action visitations of successful episodes, while avoiding those of unsuccessful episodes. By incentivizing the policy to match the visitations over all states, not just those that correspond to task success, this reward provides dense feedback on whether progress is being made towards task completion, and, we show, provably achieves this without changing the optimal policy. Focusing on finetuning of robotic control policies, we demonstrate that our approach leads to significantly faster RL finetuning performance on both simulated and real-world manipulation tasks, as compared to simply maximizing the sparse outcome reward.\loose
  
  Website: \url{https://success-visitation-matching.github.io}
\end{abstract}


\section{Introduction}

Efficiently training a policy to maximize a scalar reward function via reinforcement learning (RL) is critical to achieving high-quality performance in modern AI systems.
To ensure that maximizing the reward corresponds to completing the task of interest, modern applications of RL often rely
on \emph{outcome} rewards that assign a reward of 0 until the task of interest is successfully completed, at which point a reward of +1 is given. While such outcome rewards have enabled recent progress both in LLM ``reasoning'' (the ``reinforcement learning with verifiable rewards'' setting) \citep{guo2025deepseek,team2025kimi}, and the successful application of RL to improving robotic control policies \citep{mark2024policy, luo2024serl, chen2025conrft, wagenmaker2025steering, xiao2025self}, they are inherently \emph{sparse}---reward is only received if the task is successfully accomplished. Learning to maximize such a sparse reward requires solving a difficult credit assignment problem---we must determine which actions, perhaps taken significantly before any reward is received, lead to  reward---often resulting in slow or ineffective RL improvement.

To mitigate the challenge of maximizing a sparse reward and enable faster RL improvement, we might hope our reward would correspond not just to task \emph{completion}, but to task \emph{progress}, rewarding the agent at intermediate states if they are making progress towards completing the task, and reducing the challenge of credit assignment. While obtaining such \emph{process rewards} has been an active area of research, it has proven challenging to develop rewards that both provide an accurate estimate of task progress, and ensure that the maximizing policy also effectively completes the task. Indeed, existing approaches for specifying process rewards (``reward shaping'') typically require side knowledge about the nature of the task \citep{trott2019keeping,wang2024rl} or significant human hand-engineering or supervision \citep{andrychowicz2020learning, smith2022walk}, greatly limiting their application to domains such as robotics where side information is often not available and hand-engineering rewards is time-consuming and expensive.

\begin{figure*}[t]
  \centering
  \includegraphics[width=\textwidth]{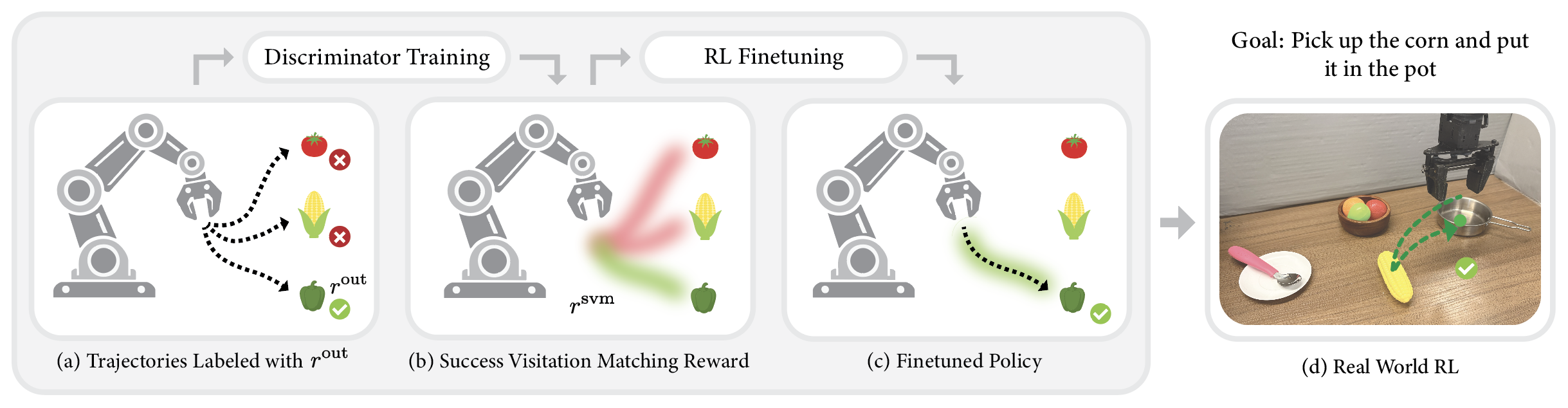}
  \caption{(a) Given trajectories labeled with sparse outcome reward $r^{\mathrm{out}}$, we train a discriminator to distinguish between state-actions in successful vs. unsuccessful trajectories. (b) This yields a dense process reward---the success visitation matching reward $r^{\mathrm{svm}}$---that rewards visiting state-actions likely to be in successful trajectories. (c) Finetuning a policy with RL on a combination of the process reward and outcome reward enables fast RL improvement. (d) Our approach scales to real-world robotic control settings, enabling fast real-world RL.}
  \label{fig:paper_figure}
  \vspace{-1em}
\end{figure*}

In this work, we propose a simple method to transform a sparse outcome reward into a dense process reward that provides intermediate feedback on whether progress is being made towards accomplishing the task. Our approach, which we refer to as \emph{success visitation matching} (\algname) process rewards, is fully automated, does not require additional human engineering or side knowledge, 
and provably ensures the policy that maximizes the process reward also maximizes the original outcome reward. \algname process rewards are motivated by the insight that 
\emph{training a policy to match the state-action visitations of observed trajectories that receive outcome reward of +1 will lead the policy to also achieve outcome reward of +1},
and produces a process reward by computing a step-level signal indicating whether the current state-action is likely under the visitations of successful trajectories or unlikely under the visitations of unsuccessful trajectories, thereby encouraging the learned policy to match successful behaviors and avoid unsuccessful ones. 
As we can estimate the visitation likelihood at all states,
not just states where the task has been completed, this reward provides  dense feedback on whether progress is being made towards the goal,
guiding the learner to task completion. See \Cref{fig:paper_figure} for an overview of our approach.

We first show that \algname process rewards are provably correct, in that an optimal policy for the \algname process reward is also optimal for the original sparse outcome reward. 
Focusing on RL improvement of pretrained robotic control policies, we then show that \algname rewards can enable significantly faster and more stable convergence of standard RL finetuning algorithms in both simulation and the real world. In particular, \algname rewards enable roughly $2\times$ faster converge of RL finetuning on VLAs ($\pi_0$ \citep{black2024pi_0}), and significantly faster convergence of RL finetuning on real-world multi-task diffusion policies, in many cases enabling learning on tasks where RL with only sparse outcome rewards fails to learn. Furthermore, instantiating \algname rewards simply requires training a discriminator over observed successful and unsuccessful trajectories, and using the log-probability induced by the discriminator to shape the reward, making it an easy-to-implement addition to existing RL training pipelines.

\iftoggle{arxiv}{}{\vspace{-0.5em}}
\section{Related Work}
\iftoggle{arxiv}{}{\vspace{-0.5em}}

We highlight the most relevant related work here; see Appendix \ref{sec:additional_related} for additional related work.

\textbf{Reward shaping for RL.}
Designing rewards to enable more efficient RL improvement---reward shaping---has a long history, going back to the seminal work of \citet{ng1999policy}. \arxiv{A variety of approaches to reward shaping have been proposed, based on the sources of information available, and it has also been shown theoretically that reward shaping can reduce the sample complexity of RL \citep{gupta2022unpacking}.} Approaches include utilizing human feedback to design a reward \citep{christiano2017deep,palan2019learning,park2022surf,hejna2023few}, adapting the reward based on uncertainty to induce exploration \citep{osband2016deep,tang2017exploration,pathak2017curiosity,burda2018exploration}, and shaping the reward based on a given distance-to-goal function \citep{trott2019keeping}.
These approach all require access to information we do not assume (human feedback or distance functions) or target different objectives (exploration). 
Several other approaches to reward shaping have considered a  setting similar to ours, where the learner is not given access to any information other than a success signal. In particular, \citet{ma2024highly, adamczyk2025bootstrapped} rely on an outcome reward, and utilize the $Q$-function or future success probability to shape the reward, \citet{memarian2021self} apply preference learning over trajectories to obtain a dense reward that is consistent with the ordering on trajectories induced by the outcome reward, \citet{stadie2020learning} treat the reward as a hyperparameter and apply hyperparameter optimization approaches to it, and \citet{ma2024reward} apply RL to design the reward itself.
In addition, \citet{fu2018variational,li2021mural} shape the reward with a success classifier, yet instead of assuming access to an outcome reward, assume access to a set of successful states and, furthermore, these works as well as \citet{durugkar2021adversarial} simply train a success classifier on whether a state is successful (i.e. in effect they aim to train an outcome reward), rather than whether a state-action is likely to lead to a successful state.
We will show that our approach significantly outperforms these methods.\loose

\textbf{Rewards in robotics.}
While many recent works in RL for robotics rely on only sparse outcome rewards \citep{mark2024policy, nakamoto2024steering, luo2024serl, luo2025precise, chen2025conrft, wagenmaker2025steering, ankile2025imitation, yuan2024policy, julg2025refined, xiao2025self},
significant attention has also been given to obtaining denser rewards to enable more efficient RL. Existing work has shown success engineering such rewards by hand \citep{lee2019learning, andrychowicz2020learning, smith2022walk} and learning them from human demonstration videos \citep{chen2021learning, shao2021concept2robot, ma2022vip} or large-scale robot demonstrations \citep{ma2023liv,alakuijala2024video,yang2024rank2reward,chen2025sarm,zhai2025vision,zhang2025rewind}. Other approaches include utilizing internet-pretrained vision-language models (VLMs) to provide both outcome and process rewards \citep{du2023vision,rocamonde2023vision,baumli2023vision,sontakke2023roboclip,lee2024affordance,venuto2024code,yang2024robot,wang2024rl,ma2024vision,venkataraman2024real,tan2025robodopamine,luu2025enhancing,lee2026roboreward,liang2026robometer}. While showing promise, these approaches all rely on additional sources of supervision---large-scale demonstrations, internet pretraining, human annotations or engineering---which are not in general available. In contrast, our approach operates under standard RL assumptions, only assuming access to an outcome reward and no additional sources of supervision.\loose

\textbf{RL for robotics.}
While the focus of this work is on designing more effective rewards rather than more effective RL algorithms, the utility of such rewards in robotics is only made possible by significant progress over the last several years in robotic RL. Early works on RL for robotics include \citet{riedmiller2009reinforcement, levine2016end, levine2018learning}, while more recently we have seen RL successfully applied to learn powerful policies for both locomotion \citep{smith2022legged, smith2022walk} and manipulation \citep{zhu2020ingredients, mendonca2024continuously, luo2024serl, luo2025precise}. With the advent of robot foundation models \citep{team2024octo,kim2024openvla,black2024pi_0,bjorck2025gr00t}, much attention has been given to RL approaches that aim to finetune such pretrained policies \citep{zhang2024grape, mark2024policy, nakamoto2024steering, ren2024diffusion, chen2025conrft, hu2025flare, ankile2025imitation, wagenmaker2025steering, dong2025matters, guo2025improving, lu2025vla, liu2025can, yuan2024policy, julg2025refined, amin2025pi, dong2025expo, xiao2025self}. In this work, we consider, in particular, RL finetuning approaches that modify the sampling process of the pretrained policy \citep{wagenmaker2025steering}, and that learn a residual policy on top of the pretrained policy \citep{johannink2018residual,ankile2025imitation, yuan2024policy, julg2025refined, dong2025expo, xiao2025self}; we will show that \algname improves the performance of both of these RL approaches.


\newcommand{\rout}{r^{\mathrm{out}}}
\newcommand{\rshape}{r^{\mathrm{proc}}}
\newcommand{\pist}{\pi^\star}
\newcommand{\pitilst}{\widetilde{\pi}^\star}
\newcommand{\Qtil}{\widetilde{Q}}
\newcommand{\Exp}{\mathbb{E}}
\newcommand{\cbar}{\bar{c}}
\newcommand{\pitil}{\widetilde{\pi}}
\newcommand{\stil}{\widetilde{s}}
\newcommand{\Qp}{\bar{Q}}
\newcommand{\pisub}{\pi^-}
\newcommand{\Dkl}[2]{\mathrm{KL}(#1 \parallel #2)}
\newcommand{\cJkl}{\mathcal{J}_{\mathrm{vm}}}
\newcommand{\cJdrs}{\mathcal{J}_{\mathrm{vm}}}
\newcommand{\cJ}{\mathcal{J}}
\newcommand{\cM}{\mathcal{M}}
\newcommand{\Vout}{V_{\mathrm{out}}}
\newcommand{\rdrs}{r^{\mathrm{vm}}}
\newcommand{\pipt}{\pi_{\mathrm{pre}}}
\newcommand{\Dpos}{\mathfrak{D}^+}
\newcommand{\Dneg}{\mathfrak{D}^-}
\newcommand{\frakD}{\mathfrak{D}}
\newcommand{\cS}{\mathcal{S}}
\newcommand{\cA}{\mathcal{A}}
\newcommand{\Dst}{f^\star}
\newcommand{\Dhat}{\widehat{f}}
\newcommand{\frakDE}{\mathfrak{D}^E}
\newcommand{\Prob}{\mathbb{P}}
\newcommand{\piplus}{\pi^+}
\newcommand{\pimin}{\pi^-}
\newcommand{\rclip}{\bar{r}}
\newcommand{\bbP}{\mathbb{P}}
\newcommand{\rproc}{\rshape}
\newcommand{\whatp}{\widehat{w}^+}
\newcommand{\whatm}{\widehat{w}^-}
\newcommand{\bbI}{\mathbb{I}}
\newcommand{\clip}{\mathrm{clip}}
\newcommand{\rvm}{r^{\mathrm{svm}}}

\iftoggle{arxiv}{}{\vspace{-0.5em}}
\section{Preliminaries}\label{sec:prelim}
\iftoggle{arxiv}{}{\vspace{-0.5em}}

We consider 
Markov decision processes (MDPs) denoted by a tuple $\cM = (\cS,\cA, P, P_{\mathrm{init}}, H, \rout)$ where $\cS$ is the set of states, $\cA$ the set of actions, $P : \cS \times \cA \rightarrow \triangle_{\cS}$ the transition probabilities, $P_0 \in \triangle_{\cS}$ the initial state distribution, $H$ the horizon, and $\rout : \cS \rightarrow \{ 0, 1 \}$ the reward. 
We let $\rout$ denote the true \emph{outcome reward}, and assume that it is 0 at all states except for states that correspond to task completion (``successful'' states), when it is 1.
Each \emph{episode} proceeds by sampling a state $s_1 \sim P_{\mathrm{init}}$, selecting an action $a_0$, transitioning to state $s_2 \sim P(s_1, a_1)$ and so on for $H$ steps, at which point the episode terminates. We are primarily interested in task completion and so assume that, once a reward of 1 is reached in an episode (a ``successful'' episode), the MDP enters a terminal state that has a reward of 0 for the remainder of the episode.
We let $\pi : \cS \rightarrow \triangle_{\cA}$ denote a policy, a mapping from states to actions. For any $\pi$, we denote $\Exp^\pi[\cdot ]$ the expectation over episodes induced by deploying $\pi$ in our environment, and
$\cJ(\pi) := \Exp^\pi[\sum_{h=1}^{H} \rout(s_h)]$
the expected reward collected by $\pi$ on $\cM$ (equivalently, the probability that $\pi$ succeeds).
Our goal is to learn a policy that maximizes $\cJ(\pi)$, and we say a policy $\piplus$ is an optimal policy if $\cJ(\piplus) = \max_\pi \cJ(\pi)$.

\iftoggle{arxiv}{}{\vspace{-0.5em}}
\section{Success Visitation Matching Process Rewards}\label{sec:method}
\iftoggle{arxiv}{}{\vspace{-0.5em}}

We wish to obtain a process reward that (a) preserves the optimal policies under the outcome reward $\rout$---so a policy which maximizes the process reward also maximizes $\rout$---while (b) providing dense, step-level feedback that guides the learner to successful states. In this section we describe our approach and show that it provably satisfies these desiderata.

\subsection{Visitation Matching Preserves Outcome Reward Maximization}
Our choice of process reward is based on a simple principle: given past observations from our environment, reward the learner for visiting states that previously led to success and penalize the learner for visiting states that previously led to failure.

To formalize this, assume we have access to a collection of episodes from our environment, $\frakD$, and let $\Dpos$ denote the successful episodes (episodes that have non-zero reward) and $\Dneg$ the unsuccessful episodes (episodes that have zero reward) in $\frakD$. Furthermore, let $\whatp_h(s,a)$ (resp. $\whatm_h(s,a)$) denote an estimate of the state-action visitation distribution at step $h$ for episodes in $\Dpos$ (resp. $\Dneg$)---that is, $\whatp_h(s,a)$ is an estimate of the density of state-actions at step $h$ of successful episodes.
Given this, we propose the following reward:
\begin{align}\label{eq:rew_def}
    \rvm_h(s,a) := \rout(s) + \lambda \cdot \clip_\beta \left ( \log \frac{\whatp_h(s,a)}{\whatm_h(s,a)} \right ) \ \ \ \text{for} \ \ \ \clip_\beta(x) := \begin{cases} \beta \cdot \mathrm{sign}(x) & |x| > \beta \\
    x & |x| \le \beta 
    \end{cases}
\end{align}
and $\lambda, \beta > 0$. Here we define $\log 0 := -\infty$ and $\log 0/0 := -\infty$. We refer to $\rvm$ as the \emph{success visitation matching} (\algname) process reward for reasons that will become clear shortly. 

$\rvm$ has two key components. First, it includes the original outcome reward, $\rout(s)$, rewarding the learner for reaching successful states. Second, it includes the difference in log-probabilities of the state-action visitations of previous successful and unsuccessful episodes. 
In particular, note that $\log \frac{\whatp_h(s,a)}{\whatm_h(s,a)}$ is positive when $\whatp_h(s,a) > \whatm_h(s,a)$---when $(s,a)$ is more likely to have been visited in a previous successful than unsuccessful episode---and negative otherwise. In addition to incentivizing the learner to reach successful states, then, $\rvm$ also incentivizes the learner to visit states likely to lead to success and avoid states likely to lead to failure.
The following result shows that, in deterministic environments, any policy that maximizes $\rvm$ also maximizes $\rout$.

\begin{theorem}\label{thm:proc_consistent}
    Assume our environment has deterministic transitions (but potentially stochastic initial state distribution) and $\cS$ and $\cA$ are countable. Let $\frakD$ be any set of episodes collected in our environment and $\whatp/\whatm$ the count-based visitation estimates:
     $\textstyle \whatp_h(s,a) := \frac{1}{|\Dpos|} \cdot \sum_{(s',a') \in \Dpos_h} \bbI \{ (s',a') = (s,a) \},$
    where $\Dpos_h$ is all observations at step $h$ of $\Dpos$, and $\whatm$ is defined similarly with respect to $\Dneg$.
    Then for any $\lambda, \beta > 0$, any policy maximizing $\rvm_h(s,a)$  also maximizes $\rout(s)$.
\end{theorem}

We provide the proof of \Cref{thm:proc_consistent} in Appendix \ref{sec:proof}. We emphasize that this result holds for \emph{any} previous set of episodes $\frakD$, regardless of which policy collected them, and that it holds for the \emph{empirical} visitation estimates---it does not require access to oracle estimates of the visitations. \Cref{thm:proc_consistent} relies on the fact that, if we have previously visited some $(s,a)$ and succeeded from that $(s,a)$ (so that $\whatp_h(s,a) >0$), then we will be able to succeed from that $(s,a)$ again. Thus, increasing the reward at such states incentivizes the learner to visit states where future success is likely, which is aligned with the original objective of reaching successful states, preserving policy optimality. Critically, however, this reward now provides a dense progress signal---rather than only rewarding success, the learner is rewarded for \emph{making progress towards success} by visiting states the previously led to success. As we will see, this leads to significant sample efficiency gains without sacrificing final performance. We note that, while \Cref{thm:proc_consistent} requires that our environment has deterministic transitions, as we show in \Cref{sec:exp_real}, our approach scales effectively to real-world settings that are not strictly deterministic. Furthermore, the requirement that $\cS$ and $\cA$ are countable is purely for convenience of the proof, and the result can be extended to uncountable spaces by representing the visitation estimates with appropriate function approximators. \loose

\textbf{Success visitation matching process reward as KL-regularized RL.}
We next show that the policy optimization landscape induced by the \prvm process reward admits a close connection to a form of KL-regularized RL.
Taking $\beta \rightarrow \infty$, so that $\rvm$ is no longer clipped, we have that the expected \prvm reward collected by some policy $\pi$ is:\loose
\begin{align*}
 \textstyle \Exp^\pi[\sum_{h=1}^H \rvm_h(s_h,a_h)]  & \textstyle = \cJ(\pi) + \lambda \cdot \sum_{h=1}^H \Exp^\pi [ \log \frac{\whatp_h(s_h,a_h)}{\whatm_h(s_h,a_h)} ] \\
 & \textstyle = \cJ(\pi) + \lambda \cdot \sum_{h=1}^H \Exp^\pi [ \log \frac{w^\pi_h(s_h,a_h)}{\whatm_h(s_h,a_h)} - \log \frac{w^\pi_h(s_h,a_h)}{\whatp_h(s_h,a_h)} ]
\end{align*}
for $w_h^\pi(s,a) = \mathbb{P}^\pi[s_h = s, a_h = a]$ the state-action visitations for policy $\pi$. However, $\Exp^\pi [ \log \frac{w^\pi_h(s_h,a_h)}{\whatm_h(s_h,a_h)} ] = \Dkl{w^\pi_h}{\whatm_h}$, so
\begin{align*}
    \textstyle \Exp^\pi[\sum_{h=1}^H \rvm_h(s_h,a_h)] = \cJ(\pi) + \lambda \cdot \sum_{h=1}^H [ \Dkl{w^\pi_h}{\whatm_h} - \Dkl{w^\pi_h}{\whatp_h}].
\end{align*}
In other words, without clipping, $\rvm$ can be seen as inducing a policy optimization objective that is a combination of the original outcome reward objective, and a term encouraging the state-action visitations to have low KL divergence between the distribution of successful episodes---to \emph{match} the visitations of successful episodes---and high KL divergence with unsuccessful episodes.
Critically, this provides dense, step-level feedback: rather than only being rewarded for succeeding, the learner is rewarded for staying close to successful episodes and avoiding unsuccessful episodes at each step in an episode. Furthermore, by lightly regularizing this objective with clipping, \Cref{thm:proc_consistent} shows that this objective has the same optimal policy as the original outcome reward objective.\loose

\subsection{Reinforcement Learning with Success Visitation Matching Process Rewards}
Our ultimate goal is to enable fast RL improvement. The above discussion shows that the \algname reward provides a dense reward signal without changing the optimal policy. In this section we show how to estimate $\rvm$ in large state spaces and
present our full approach combining the \algname process reward with RL improvement.

\textbf{Estimating visitation ratios in large state spaces.}
\Cref{thm:proc_consistent} relies on the count-based visitation estimate natural in countable state spaces, yet this visitation estimator does not scale to real-world settings with large state-spaces. Here we show how we can compute $\rvm$ in such more general settings.\loose

Note that $\rvm$ in \eqref{eq:rew_def} only requires knowing the \emph{ratio} of the visitation densities. 
Critically, to estimate the ratio of densities, it often suffices to simply solve a classification problem rather than explicitly estimating the full density. To see this, consider some distributions $P$ and $Q$, and let
\begin{align*}
\Dst & := {\textstyle \argmin_{f}} \ \Exp_{X \sim P}[\log f(X)] + \Exp_{X \sim Q}[\log (1 - f(X))].
\end{align*}
Assuming $f$ is chosen from an expressive enough function class, $\Dst(x) = \frac{P(x)}{P(x) + Q(x)}$, so $\frac{\Dst(x)}{1 - \Dst(x)} = \frac{P(x)}{Q(x)}$. In other words, to estimate the ratio of visitations $\frac{P(x)}{Q(x)}$, it suffices to simply train a discriminator on observations from $P$ and $Q$, and use the ratio of values the discriminator assigns, avoiding the need for explicit density estimation.
In our setting, we can therefore compute $\rvm$ by first fitting
\begin{align}\label{eq:Dhat}
\begin{split}
\Dhat_h & := {\textstyle \argmin_{f}} \ \Exp_{(s,a) \sim \Dpos_h}[\log f(s,a)]  + \Exp_{(s,a) \sim \Dneg_h}[\log (1 - f(s,a))],
\end{split}
\end{align}
and then using $\frac{\Dhat_h(s,a)}{1-\Dhat_h(s,a)}$ in place of $\frac{\whatp_h(s,a)}{\whatm_h(s,a)}$ in \eqref{eq:rew_def}. This allows us to easily scale the computation of $\rvm$ to more general settings by simply training a classifier on past observations.

\textbf{RL with success visitation matching process rewards.}
We present our full approach in \Cref{alg:rdrs_rl}. \Cref{alg:rdrs_rl} alternates between updating the policy with RL to maximize the \algname reward, and refining the discriminator $\Dhat$ on new data. For every episode, it adds the episode collected to the positive examples $\Dpos$ if it contains outcome reward $>0$, and otherwise to the negative examples $\Dneg$, and trains $\Dhat$ to distinguish the state-actions in these datasets. \Cref{alg:rdrs_rl} optionally takes as input an initial policy $\pipt$ (for example, a pretrained policy to finetune) and, if such a policy is given, collects an initial set of rollouts from this policy in order to initialize $\Dhat$. While we will demonstrate \Cref{alg:rdrs_rl} can be effectively instantiated with several commonly used approaches for robotic RL (\dsrl, residual RL, and \textsc{Rlpd}), in principle any RL algorithm can be used.

\begin{algorithm}[t]
\caption{Reinforcement Learning with \prvm Process Reward}
\label{alg:rdrs_rl}
\begin{algorithmic}[1]
\STATE \textbf{input:} $\rvm$ weight $\lambda$, $\rvm$ clipping $\beta$, pretrained $\pipt$ (optional), initial rollouts $N_0$ (optional)\loose
\STATE Collect $N_0$ episodes with $\pipt$, set $\Dpos$ to successful episodes, $\Dneg$ to all others
\STATE Initialize discriminator $\Dhat_h$ on $\Dpos_h$ and $\Dneg_h$, following \eqref{eq:Dhat}
\STATE Initialize $\pi_1$ to $\pi_0$ or random policy
\FOR{$t = 1, 2, 3, \ldots $}
\STATE Run $\pi_t$ for one episode, add episode to $\Dpos$ if it is successful, otherwise to $\Dneg$
\STATE Update $\Dhat$ on $\Dpos$ and $\Dneg$, following \eqref{eq:Dhat}
\STATE Update $\pi_t$ to $\pi_{t+1}$ by maximizing reward
$\begin{aligned}
\rvm_h(s,a) = \rout(s) + \lambda \cdot \clip_\beta \left (  \log \tfrac{\Dhat_h(s,a)}{1-\Dhat_h(s,a)} \right )
\end{aligned}$
\ENDFOR
\end{algorithmic}
\end{algorithm}

\textbf{Connection to adversarial inverse RL.}
Adversarial inverse RL is an imitation learning approach that seeks to learn a reward function which, when maximized with RL, leads to a policy that has visitations matching that of a given set of expert demonstrations $\frakDE$ \citep{finn2016connection, ho2016generative, fu2017learning}.
In particular, the approach of \citet{fu2017learning} trains a discriminator $\Dhat$ as we do in \eqref{eq:Dhat}, yet replacing $\Dpos$ with $\frakDE$, the expert demonstrations, and $\Dneg$ with all episodes collected online. It then trains a policy with RL to maximize the reward
$\log \frac{\Dhat(s,a)}{1 - \Dhat(s,a)}$,
and alternates between updating $\Dhat$ (keeping $\Dpos$ fixed but expanding $\Dneg$ to include all episodes collected during online RL), and the RL policy maximizing this reward.\loose

Note that this form of reward is precisely the form of the \prvm process reward, with the key difference that, instead of being given a set of expert demonstrations, we set $\Dpos$ to observed episodes with positive outcome reward and $\Dneg$ to observed episodes with zero outcome reward. Thus, \prvm can alternatively be seen as applying adversarial inverse RL to obtain a dense process reward in the setting where outcome rewards are observed, and replacing the demonstrations with episodes that obtain positive outcome reward. 
Critically, though, \prvm continues to update $\Dpos$ online, using the observed outcome reward to determine which episodes to place in $\Dpos$, whereas adversarial inverse RL assumes access to a fixed set of positive examples. As we show in \Cref{sec:experiments}, this continual update is necessary to achieving effective performance. 
\Cref{thm:proc_consistent}, furthermore, shows that this approach leads to a policy that also maximizes the original outcome reward.


\iftoggle{arxiv}{}{\vspace{-0.5em}}
\section{Experiments}\label{sec:experiments}
\iftoggle{arxiv}{}{\vspace{-0.5em}}

Finally, we seek to understand whether \prvm rewards enable effective RL improvement in practice, in particular in robotic control domains. We aim to investigate (a) if \prvm process rewards enable faster RL improvement than running RL only on a sparse outcome reward and (b) if the optimal policy under the \prvm reward also maximizes the outcome reward.

\newcommand{\pipre}{\pi_{\mathrm{pre}}}

\subsection{Experiment Details}

We implement RL with \prvm rewards as described in Algorithm~\ref{alg:rdrs_rl}, with the only modification that we train a single discriminator for all steps $h$ together rather than an individual discriminator for each step. For all simulated experiments, results are averaged across three random seeds, and error bars denote one standard error. We provide further details on our experimental setup below; see \Cref{sec:exp_details} for additional details, and \Cref{sec:exp_individual_results} for per-task results from aggregated curves.

\textbf{RL algorithm.} We focus primarily on the RL finetuning setting, and test \prvm with two different commonly utilized RL finetuning algorithms:  \dsrl ~\citep{wagenmaker2025steering} and \resrl ~\citep{johannink2018residual,ankile2025imitation, yuan2024policy}. Both approaches assume access to a pretrained policy $\pipre$. \dsrl applies specifically to settings in which $\pipre$ is a diffusion or flow policy, and learns a ``noise'' policy $\pi_{\dsrl}(z \mid s)$ that selects the input noise to the pretrained policy's denoising process.
\resrl applies to any type of pretrained policy, and learns an additive correction to the action produced by the pretrained policy. That is, at each state the pretrained policy first produces an action $a_0 = \pipre(s)$, and the residual policy $\pi_{\text{res}}(s, a_0)$ then predicts a correction, yielding the final action: $a = a_{0} + \gamma \cdot \pi_{\text{res}}(s, a_{0})$. For both algorithms, we roll out the pretrained policy $N_0$ episodes on the task prior to finetuning, and utilize the rollouts to initialize both the discriminator $\widehat{f}$, as described in Algorithm~\ref{alg:rdrs_rl}, and the replay buffer for RL training. In \Cref{sec:exp_rlpd}, we also apply \prvm to settings where a set of demonstrations are given rather than a pretrained policy, and consider the \textsc{Rlpd} \citep{ball2023efficientonlinereinforcementlearning} algorithm.

\begin{wrapfigure}{r}{0.25\textwidth}
  \vspace{-4mm}
  \centering

  \newcommand{\sceneimgw}{0.48\linewidth}
  \newcommand{\sceneimggap}{0.02\linewidth}

  \includegraphics[width=\sceneimgw]{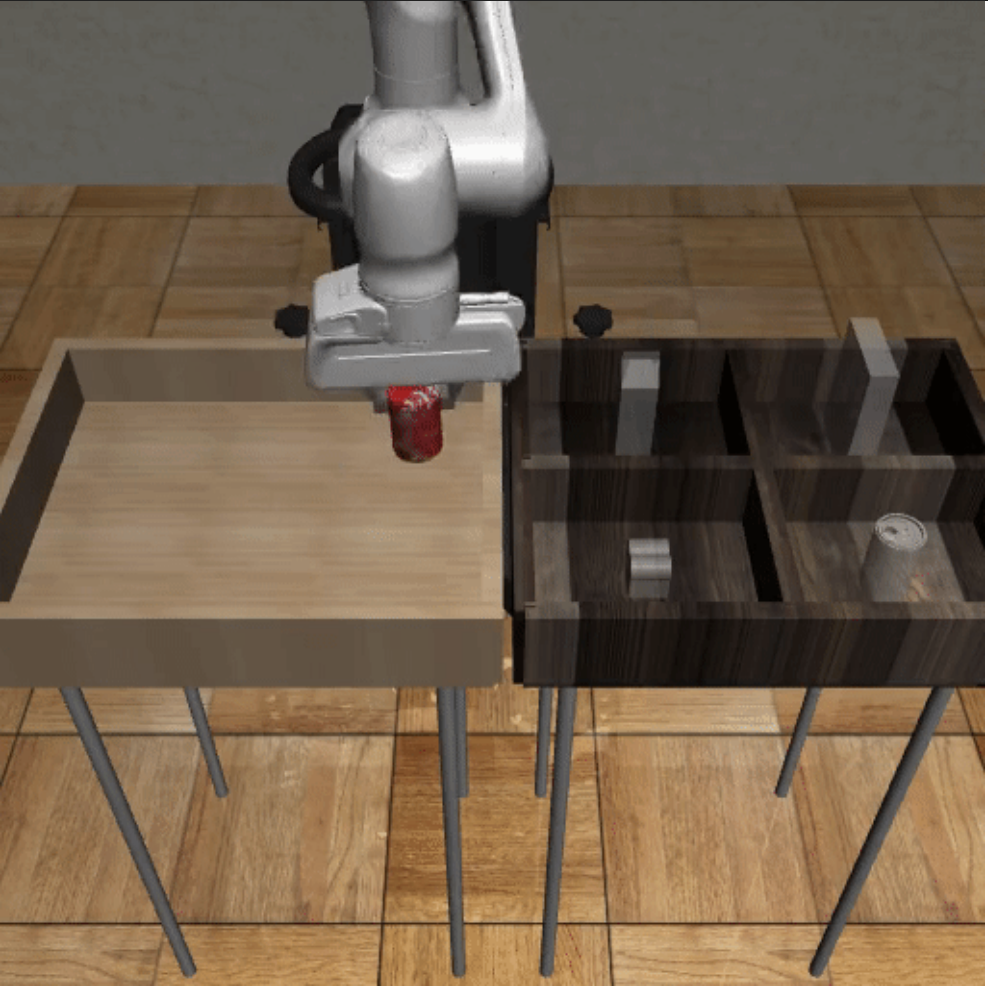}%
  \hspace{\sceneimggap}%
  \includegraphics[width=\sceneimgw]{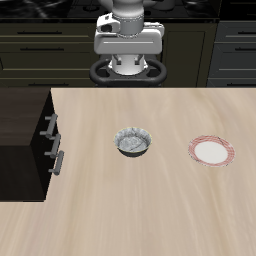}

  \vspace{0.06em}

  \includegraphics[width=\sceneimgw]{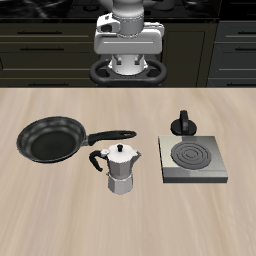}%
  \hspace{\sceneimggap}%
  \includegraphics[width=\sceneimgw]{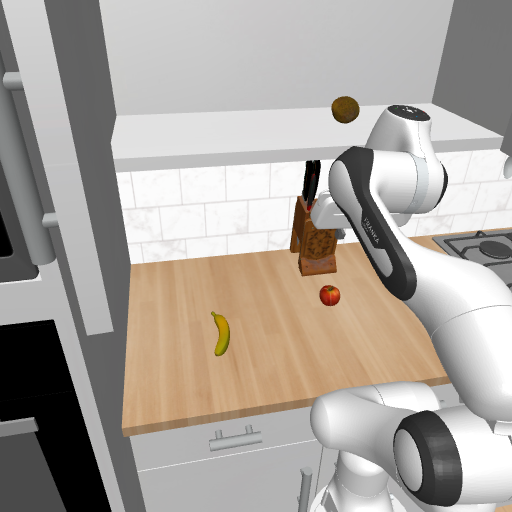}
  \caption{\texttt{Robomimic}, \texttt{LIBERO}, and \texttt{Robocasa} scenes.}
    \vspace{-1.5em}
  
  \label{fig:scenes}
\end{wrapfigure}
\textbf{Environments.} For our RL finetuning experiments, we evaluate our method on the \texttt{LIBERO-90}~\citep{liu2023liberobenchmarkingknowledgetransfer} and  \texttt{RoboCasa}~\citep{nasiriany2024robocasalargescalesimulationeveryday} benchmarks, and in the real world on the WidowX 250 6-DoF robot arm. 
The \texttt{LIBERO-90} benchmark is an image-based simulated robotic manipulation benchmark consisting of 90 total tasks distributed across 20 scenes. We focus primarily on \texttt{Kitchen Scenes 1-3}, comprising a total of 16 tasks. In addition, we evaluate on \texttt{RoboCasa}, an image-based simulated manipulation benchmark with visually diverse, realistic kitchen environments. We use the \texttt{PnPCounterToCab} task suite, where the robot must pick an object from the counter and place it inside a cabinet, and consider 3 such tasks.
For our experiments learning from demonstrations, we consider the \texttt{Robomimic} benchmark \citep{mandlekar2021matterslearningofflinehuman}, focusing on the \texttt{Can} and \texttt{Square} tasks. See \Cref{fig:scenes} for examples of the \texttt{LIBERO}, \texttt{RoboCasa}, and \texttt{Robomimic} scenes used in our experiments.
All three benchmarks provide a sparse reward signal, taking value $+1$ upon successful task completion and $0$ otherwise, which we utilize as the outcome reward. For the WidowX experiments, a human provides a sparse outcome reward on successful completion of the task.\loose

\begin{figure*}[t]
  \centering
  \includegraphics[width=0.85\textwidth]{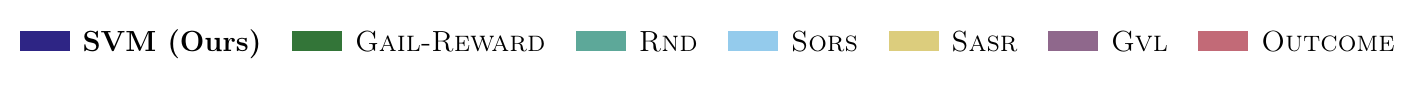}
  \vspace{0.5em}

  \begin{minipage}[t]{0.425\textwidth}
    \centering
    \includegraphics[width=.365\linewidth]{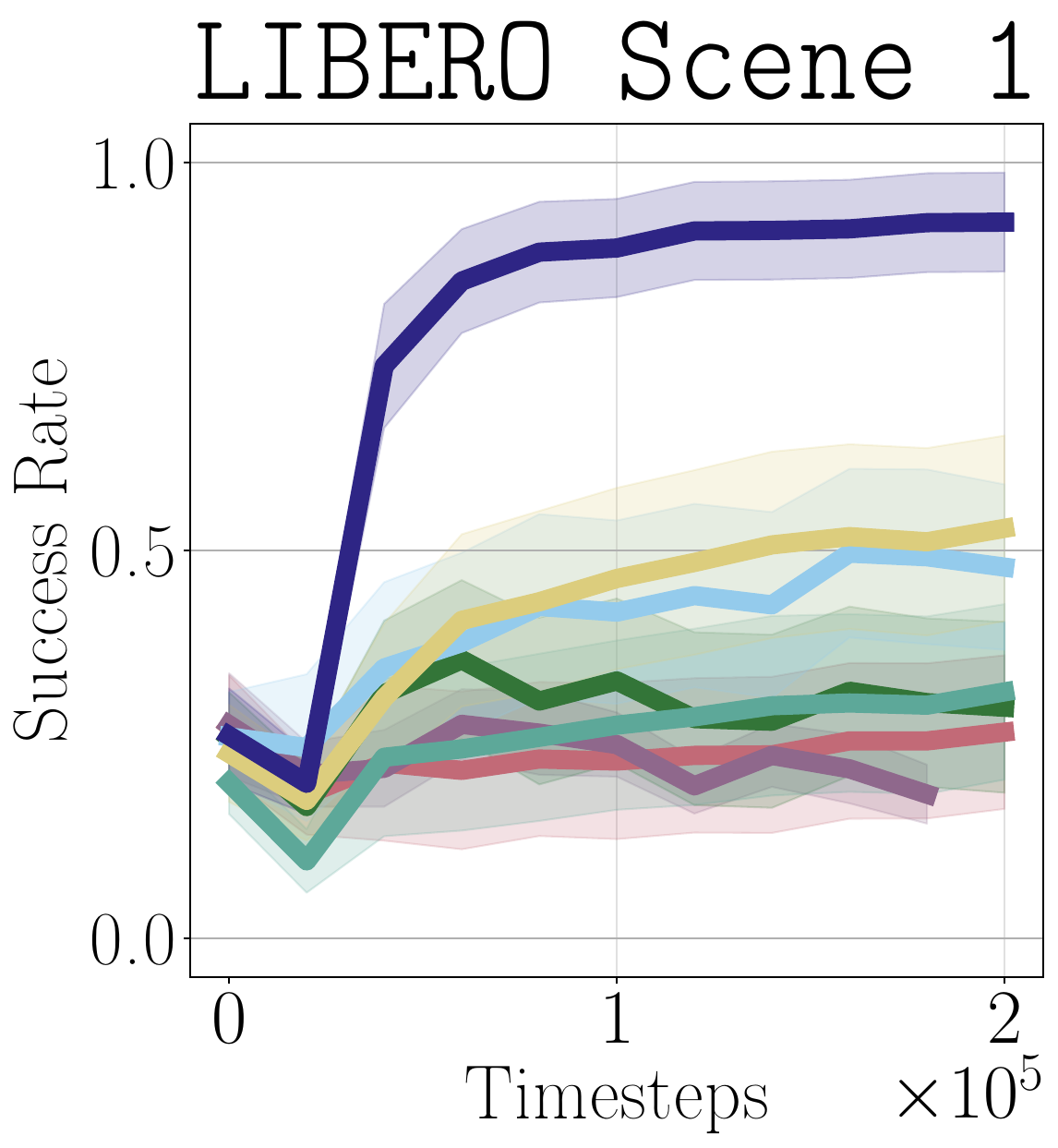}%
    \includegraphics[width=.305\linewidth]{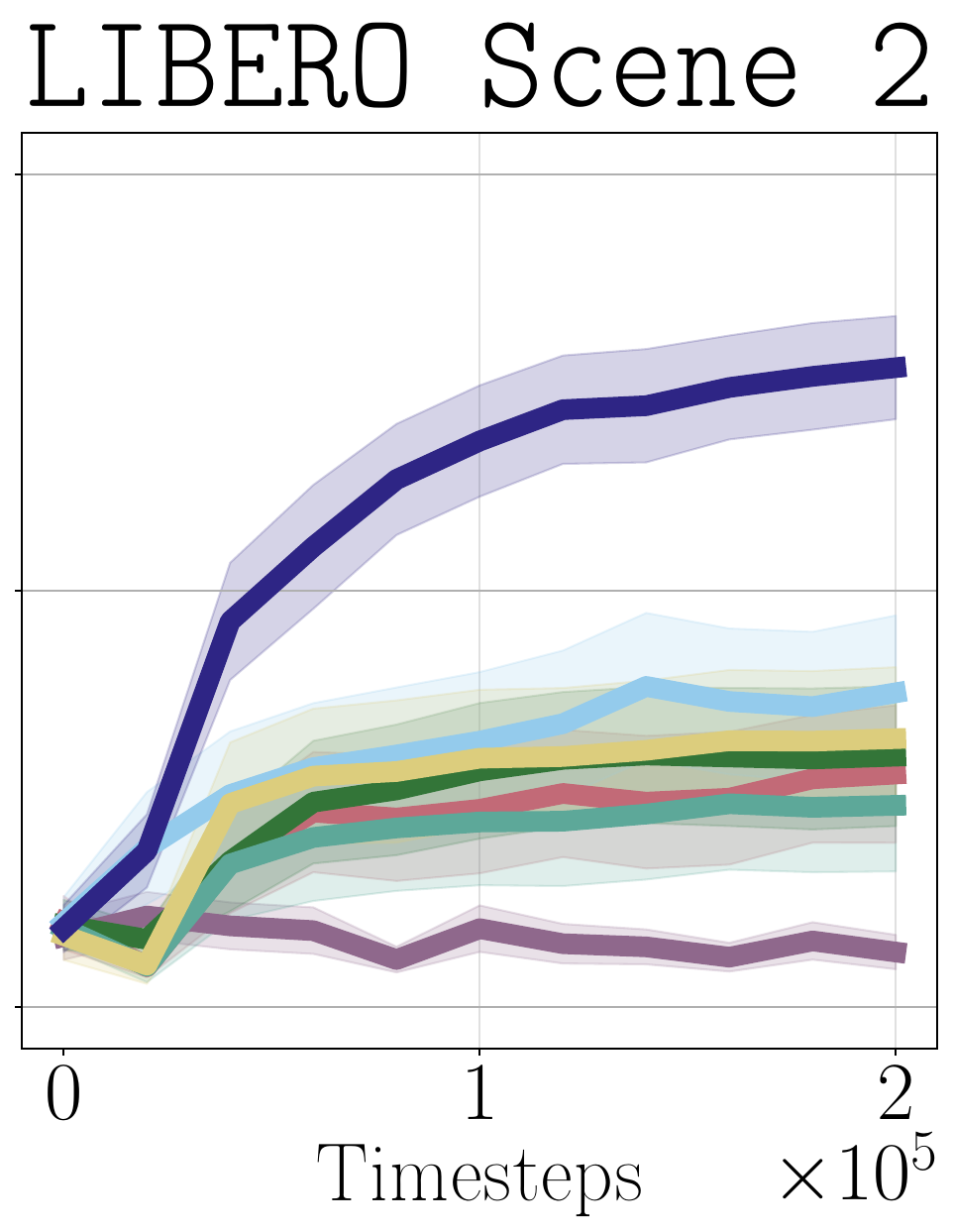}%
    \includegraphics[width=.305\linewidth]{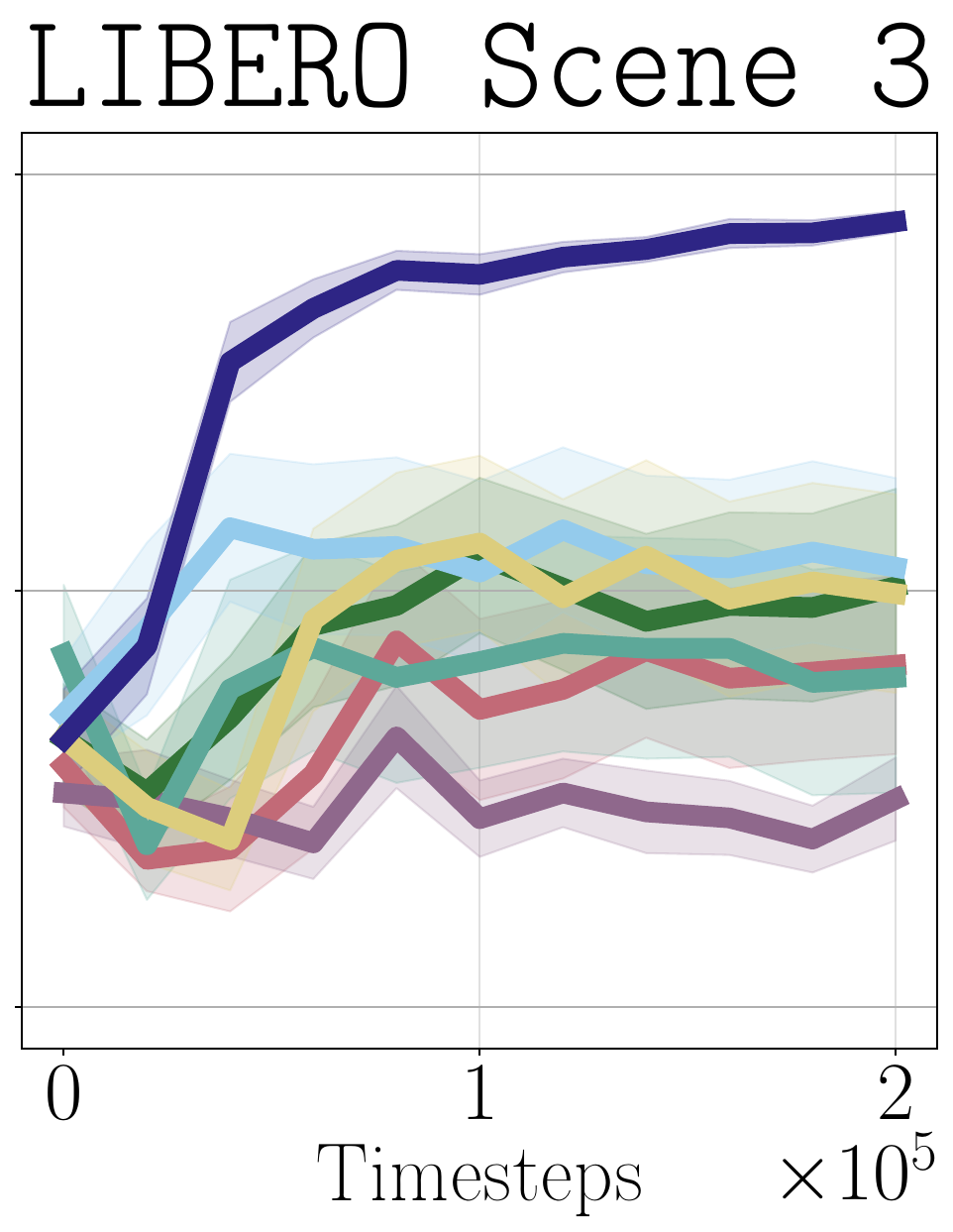}
    \caption{Aggregated results of RL finetuning with \dsrl and \prvm on \texttt{LIBERO-90 Scenes 1-3} (16 tasks).}
    \label{fig:libero_dsrl}
  \end{minipage}
  \hfill
  \begin{minipage}[t]{0.555\textwidth}
    \centering
    \includegraphics[width=.28\linewidth]{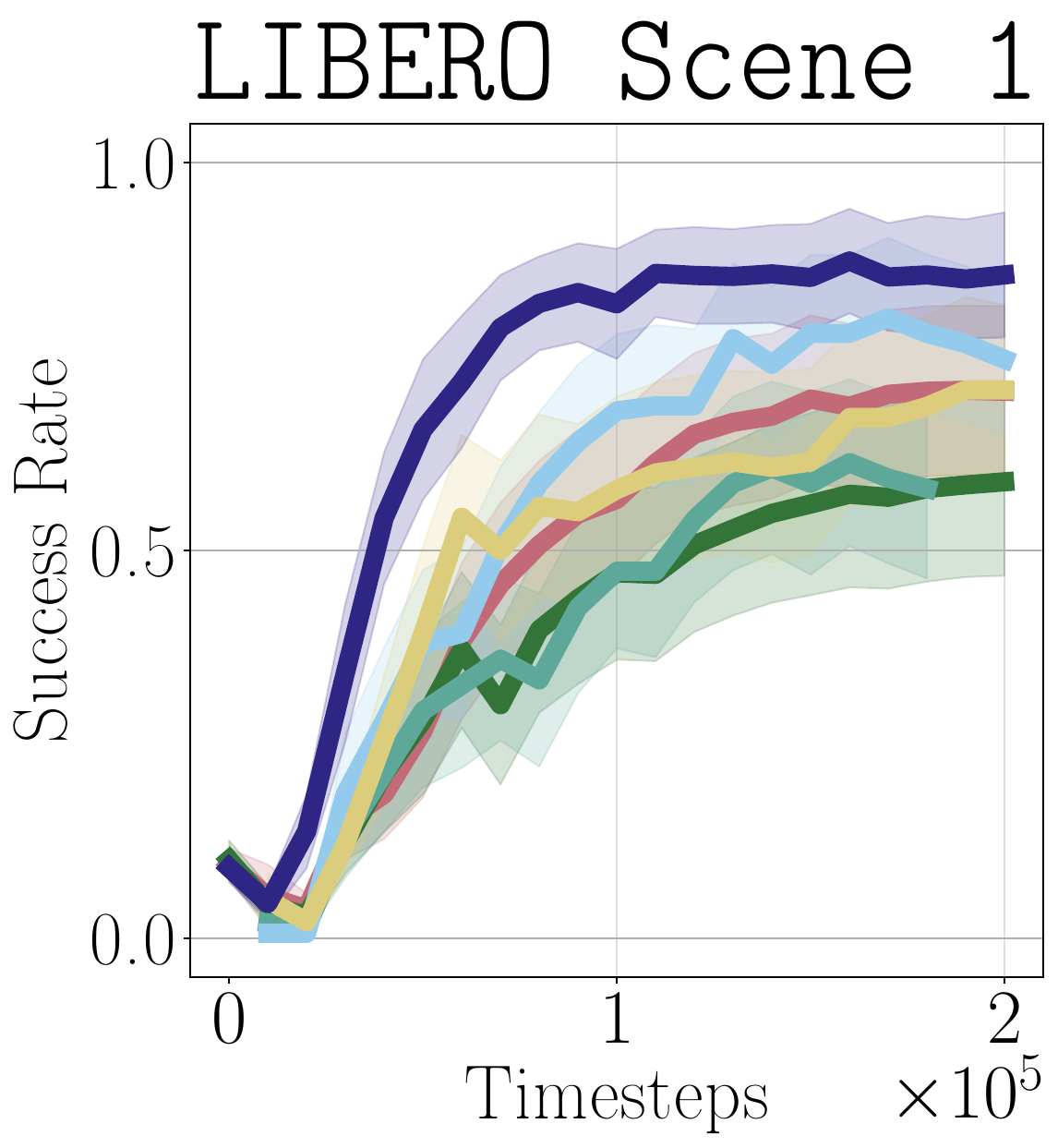}%
    \includegraphics[width=.235\linewidth]{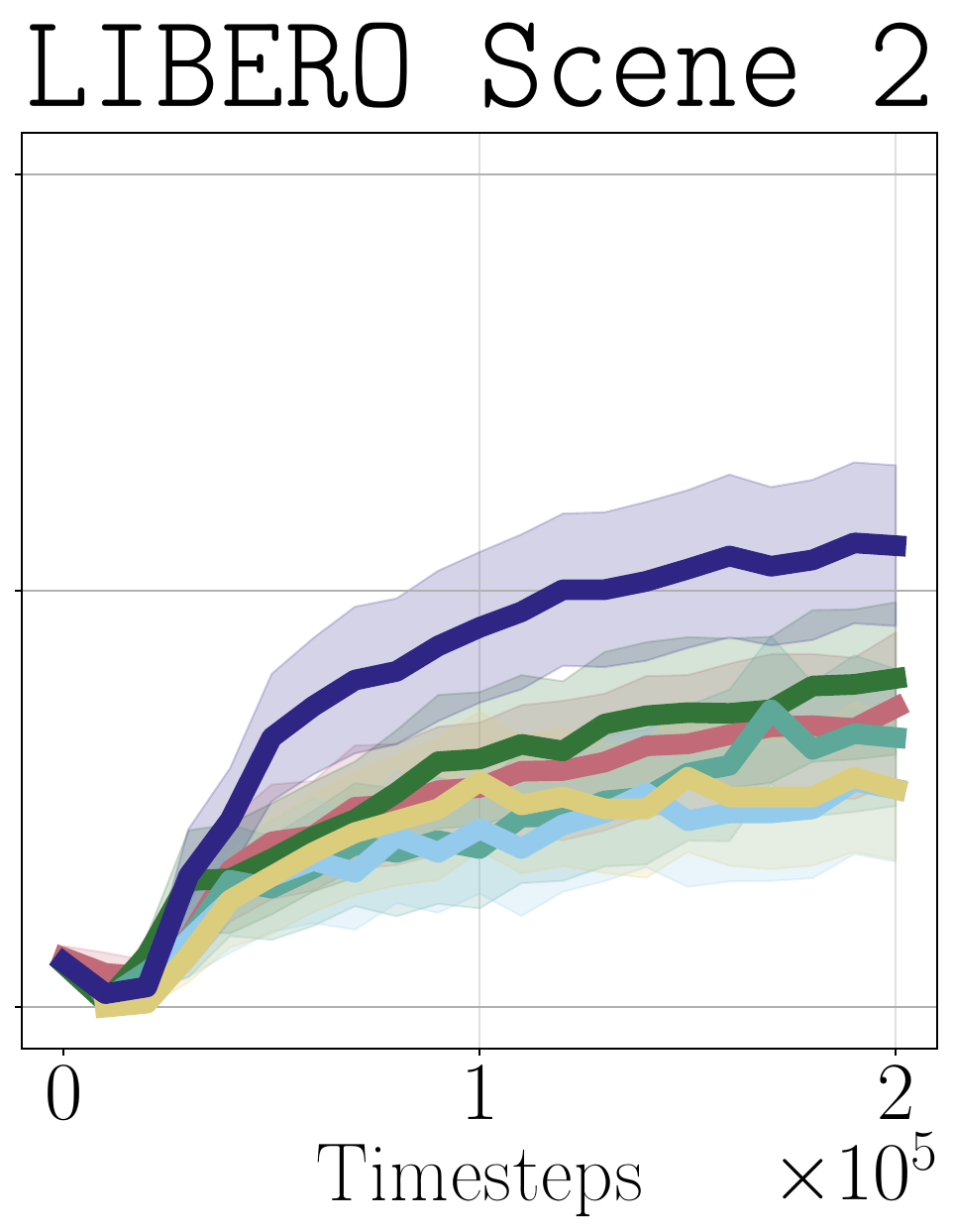}%
    \includegraphics[width=.235\linewidth]{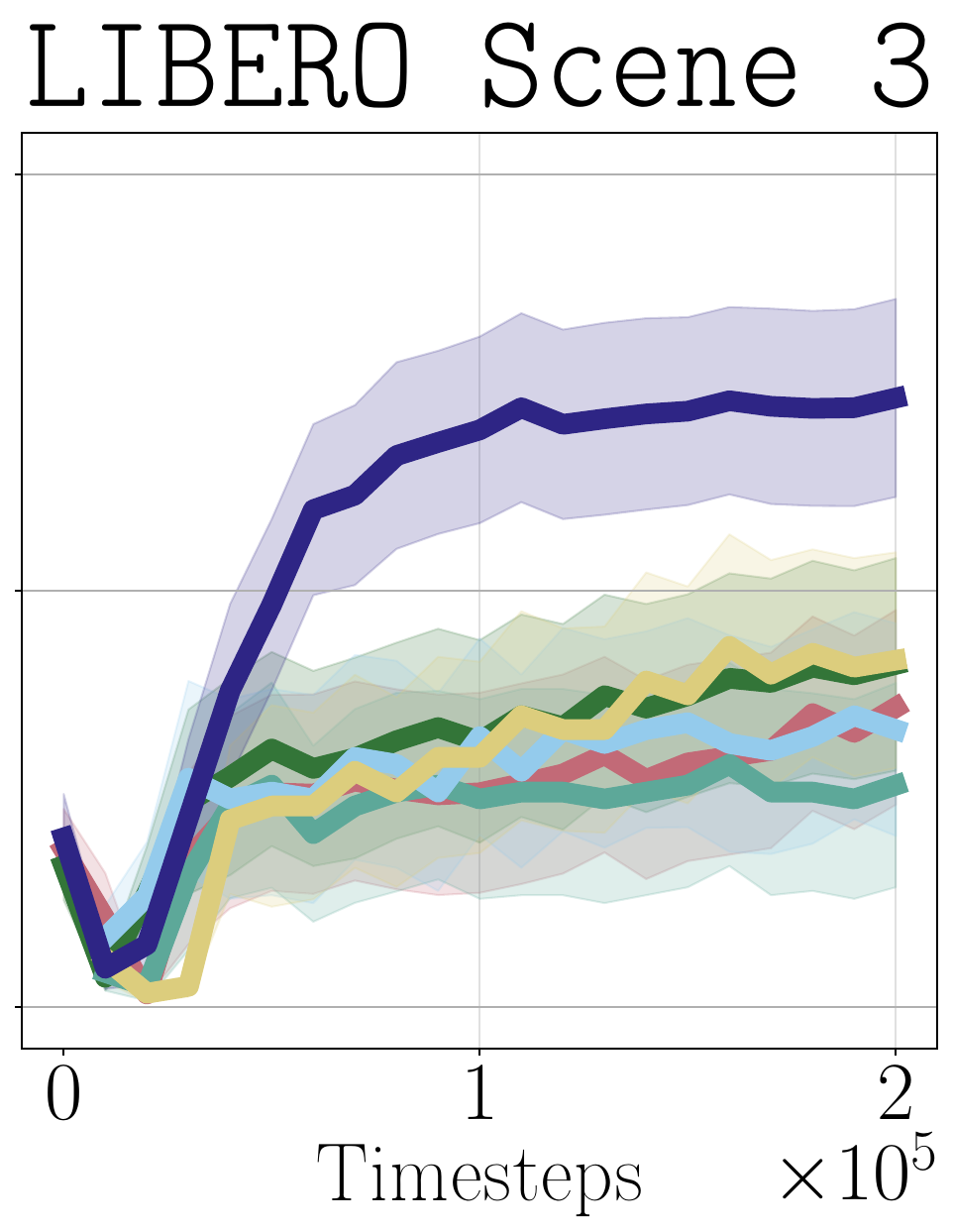}%
    \includegraphics[width=.243\linewidth]{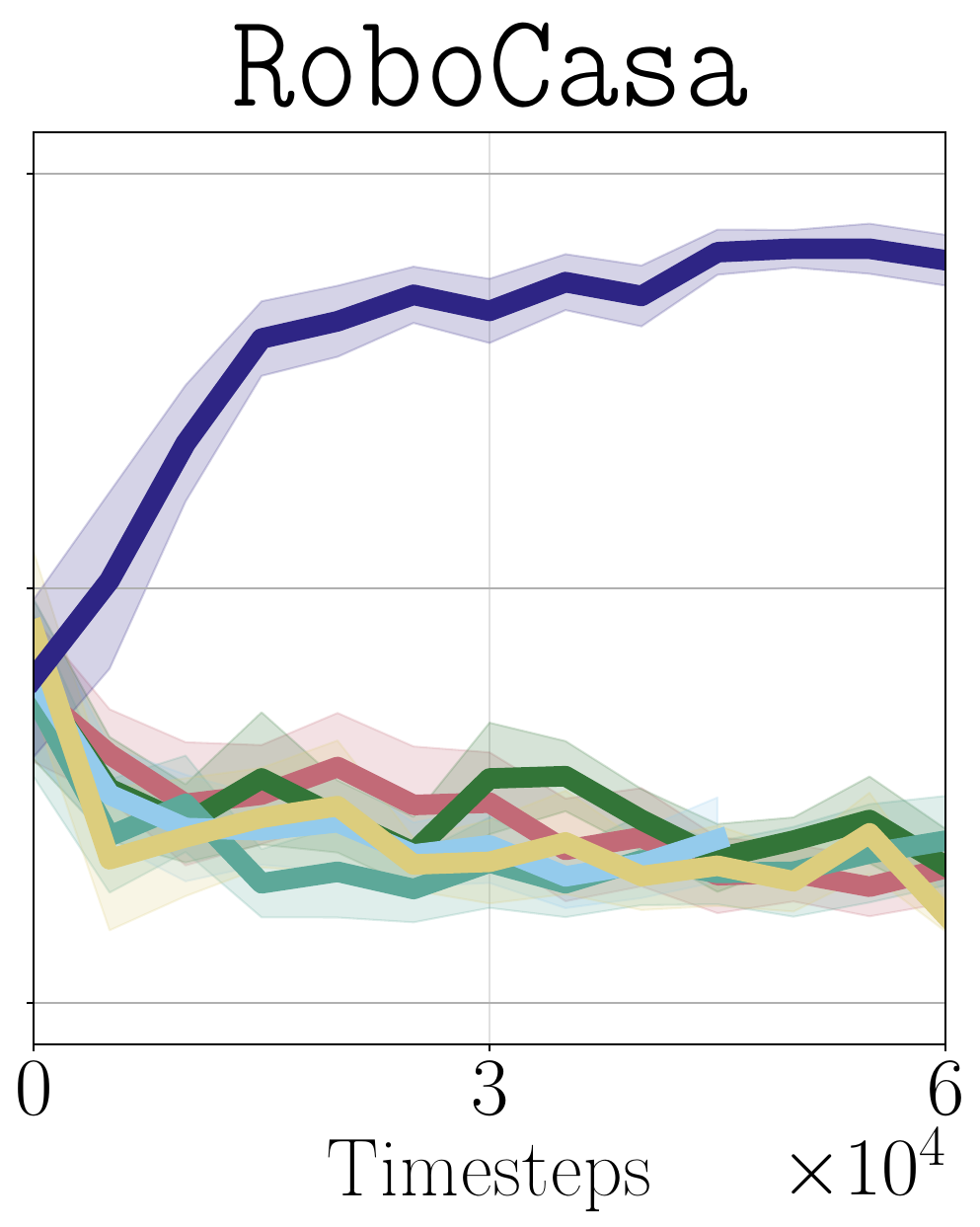}
    \caption{Aggregated results of RL finetuning with \resrl and \prvm on \texttt{LIBERO-90 Scenes 1-3} (16 tasks) and \texttt{Robocasa} (3 tasks).}
    \label{fig:libero_residual}
  \end{minipage}
  \vspace{-1.5em}
\end{figure*}

\newcommand{\gail}{\textsc{Gail}\xspace}
\newcommand{\gailr}{\textsc{Gail-reward}\xspace}

\textbf{Baselines.}
We compare \prvm rewards to several other reward shaping methods, focusing primarily on approaches that only require access to an outcome reward and no other  information:
\begin{itemize}[leftmargin=*]
    \item \textsc{Sors}~\citep{memarian2021self}: Trains a reward function using the Bradley-Terry preference model to produce an ordering on trajectories consistent with the discounted return received in each trajectory. 
    \item \textsc{Sasr}~\citep{ma2024highly}: Estimates how likely a state is to be in a successful trajectory by applying kernel density estimation to the observed trajectories. Utilizes the estimated density to shape the reward. 
     \item \gail~\citep{ho2016generative}: Uses a discriminator to distinguish a provided set of expert demonstrations from policy rollouts. Note that \gail does not assume access to an outcome reward. 
     We consider two variants of \gail. In \Cref{sec:exp_rlpd}, where we assume access to a set of expert demonstrations, we run \gail exactly as described in \citet{ho2016generative}, and do not utilize the outcome reward at all. In other results, where we do not assume access to a set of demonstrations, we adapt \gail to our setting by making two modifications: (i) we treat successful transitions from the initial $N_0$ policy rollouts as expert demonstrations and all online policy rollouts as negative examples, and (ii) we combine the \gail reward with the outcome reward during finetuning. This approach, then, can be seen as a variant of \prvm rewards, but where instead of updating the positive examples online with the successful trajectories collected, we use only the successful trajectories in the initial rollouts, simulating the role of a fixed set of demonstrations. We refer to this variant as \gailr.   
    \item \textsc{Rnd}~\citep{burda2018exploration}: Provides an intrinsic reward based on prediction error in a fixed random network in order to induce exploration.
    \item \textsc{Gvl}~\citep{ma2024vision}: Uses a frozen VLM (\texttt{gemini-3.1-pro-preview}) to estimate task progress from shuffled video frames; we use the predicted task progress rate as the reward. Due to budget constraints, we only run this baseline on \dsrl \texttt{LIBERO} settings.
    \item \textsc{Outcome}: Utilizes the outcome reward with no shaping.
\end{itemize}
For all baselines (with the exception of \gail, as described above), we maximize the shaped reward combined with the outcome reward.

\subsection{\prvm Process Rewards Enable Efficient RL Finetuning}

We first consider the performance of RL with \prvm rewards in the RL finetuning regime on the \texttt{LIBERO} and \texttt{RoboCasa} benchmarks. For the \texttt{LIBERO} tasks, we set $\pipre$ to a diffusion transformer policy~\citep{dasari2024ingredients} pretrained with behavioral cloning on the entire \texttt{LIBERO-90} suite using the provided demonstrations. For the \texttt{RoboCasa} tasks, we pretrain a BC-Transformer following the official codebase~\citep{nasiriany2024robocasalargescalesimulationeveryday} on the \texttt{PnPCounterToCab} task suite using the provided human demonstrations.

Figures \ref{fig:libero_dsrl} and \ref{fig:libero_residual} report the RL finetuning performance on \texttt{LIBERO} (aggregated by scene) and \texttt{RoboCasa} (aggregated by task). Across all three \texttt{LIBERO} kitchen scenes and both RL finetuning algorithms (\dsrl and \resrl), \prvm rewards consistently improve both the final converged success rate and the sample efficiency, as compared to prior reward shaping methods and performance without shaping. In particular, we see that with \dsrl, \prvm leads to a significantly higher final success rate after $2 \cdot 10^5$ timesteps and much faster convergence than other approaches.
With \resrl, \prvm rewards either require $2 \times$ or more fewer samples to converge, or achieves a significantly higher success rate than baselines. Furthermore, for both RL finetuning approaches, the final performance achieved by \prvm rewards consistently reaches over 80\%, demonstrating that \prvm rewards do not significantly diminish the ability of the final policy to maximize the outcome reward, while leading to much faster convergence.

These trends continue to hold on \texttt{RoboCasa}---across all three tasks, \prvm substantially improves RL finetuning, quickly achieving a success rate of approximately $90\%$, while other methods fail to enable RL improvement at all. Note that we only consider \resrl on \texttt{RoboCasa} as the official \texttt{RoboCasa} policy is not a diffusion or flow policy, which is required for \dsrl.\loose

\begin{figure*}[t]
  \centering
  \includegraphics[width=0.7\textwidth]{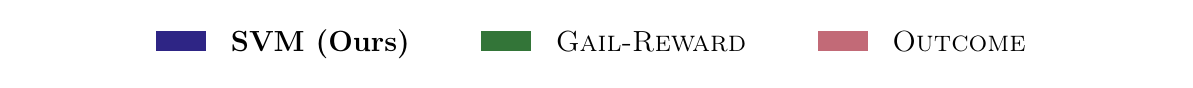}
  \vspace{0.2em}

  \begin{minipage}[t]{0.45\textwidth}
   \centering
    \includegraphics[width=.368\linewidth]{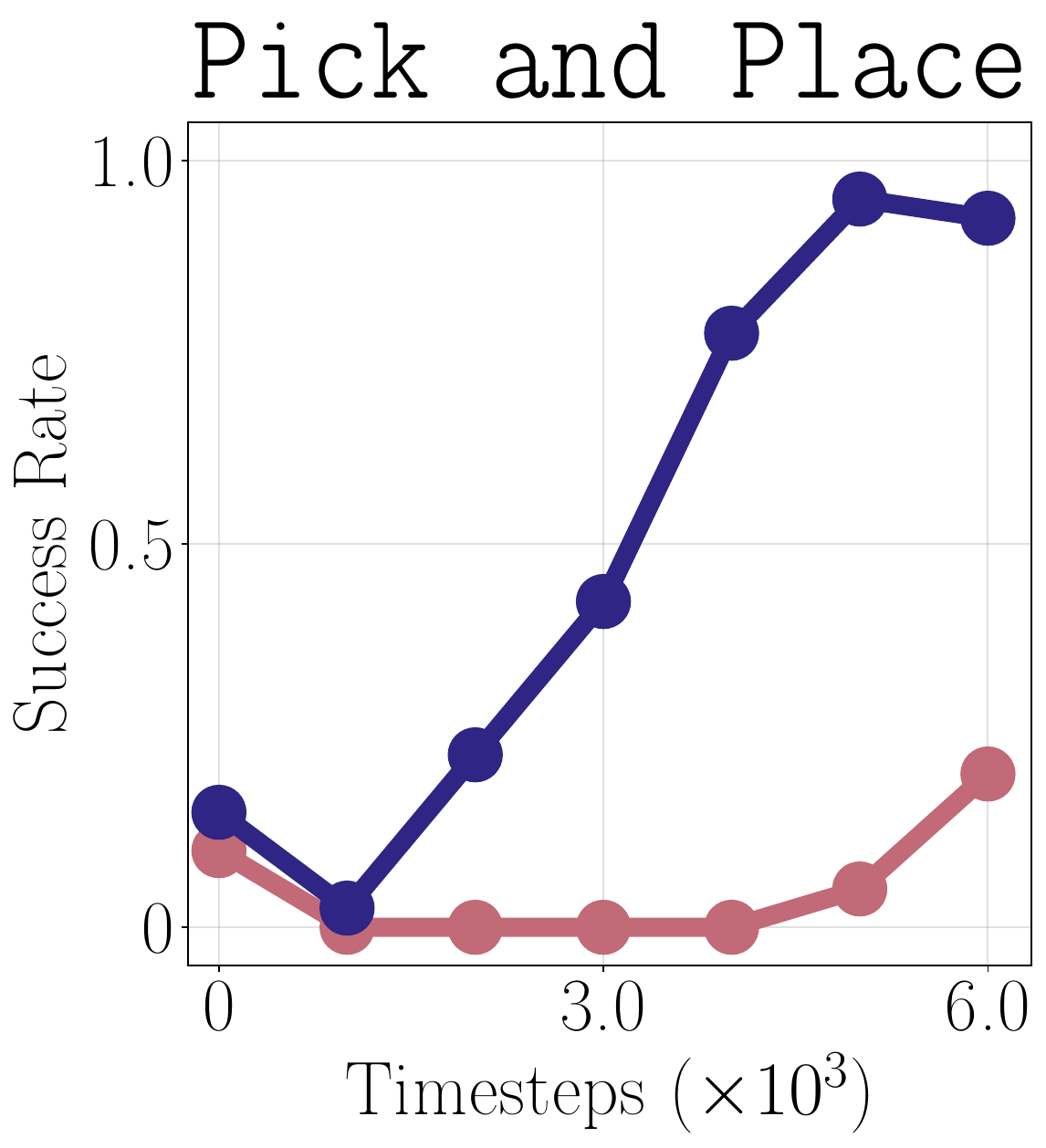}\hfill
    \includegraphics[width=.31\linewidth]{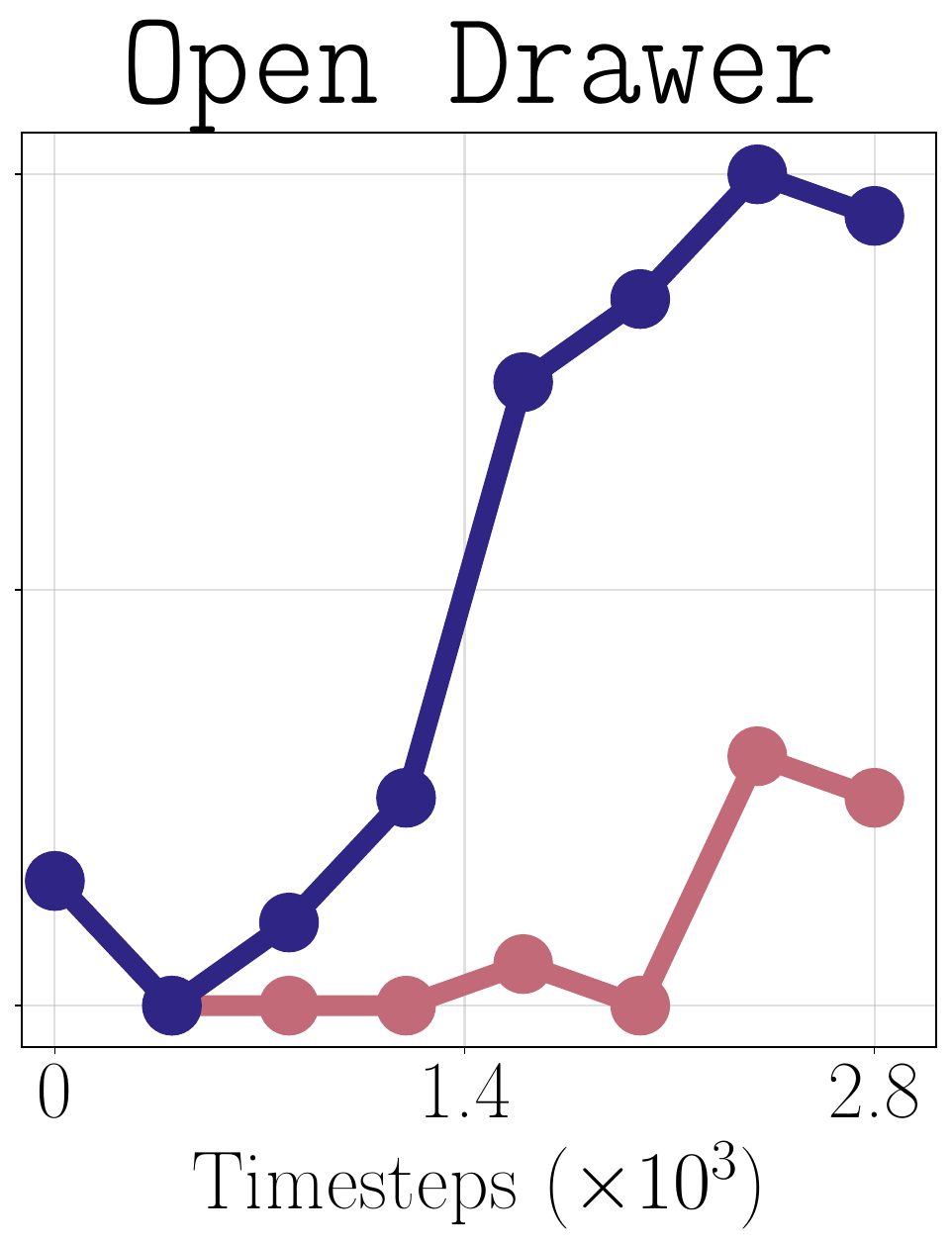}\hfill
    \includegraphics[width=.31\linewidth]{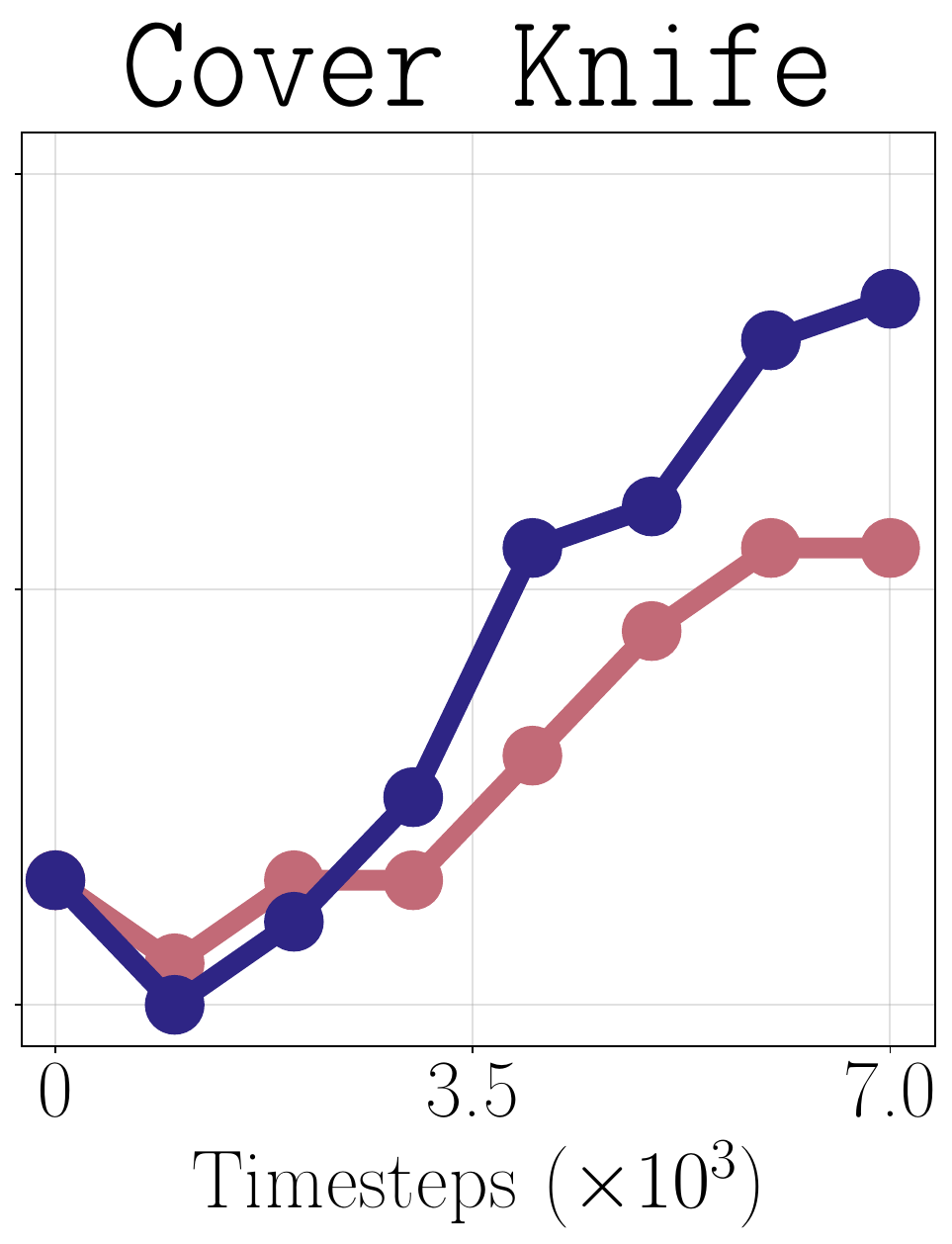}
    \phantomcaption
    \caption{RL finetuning with \dsrl and \prvm reward on 3 real-world tasks on WidowX robot arm (Figure \ref{fig:widowx_exp_scene}).}
    \label{fig:widowx_exp}
  \end{minipage}
  \hfill
  \begin{minipage}[t]{0.205\textwidth}
    \centering
    \includegraphics[width=0.823\linewidth]{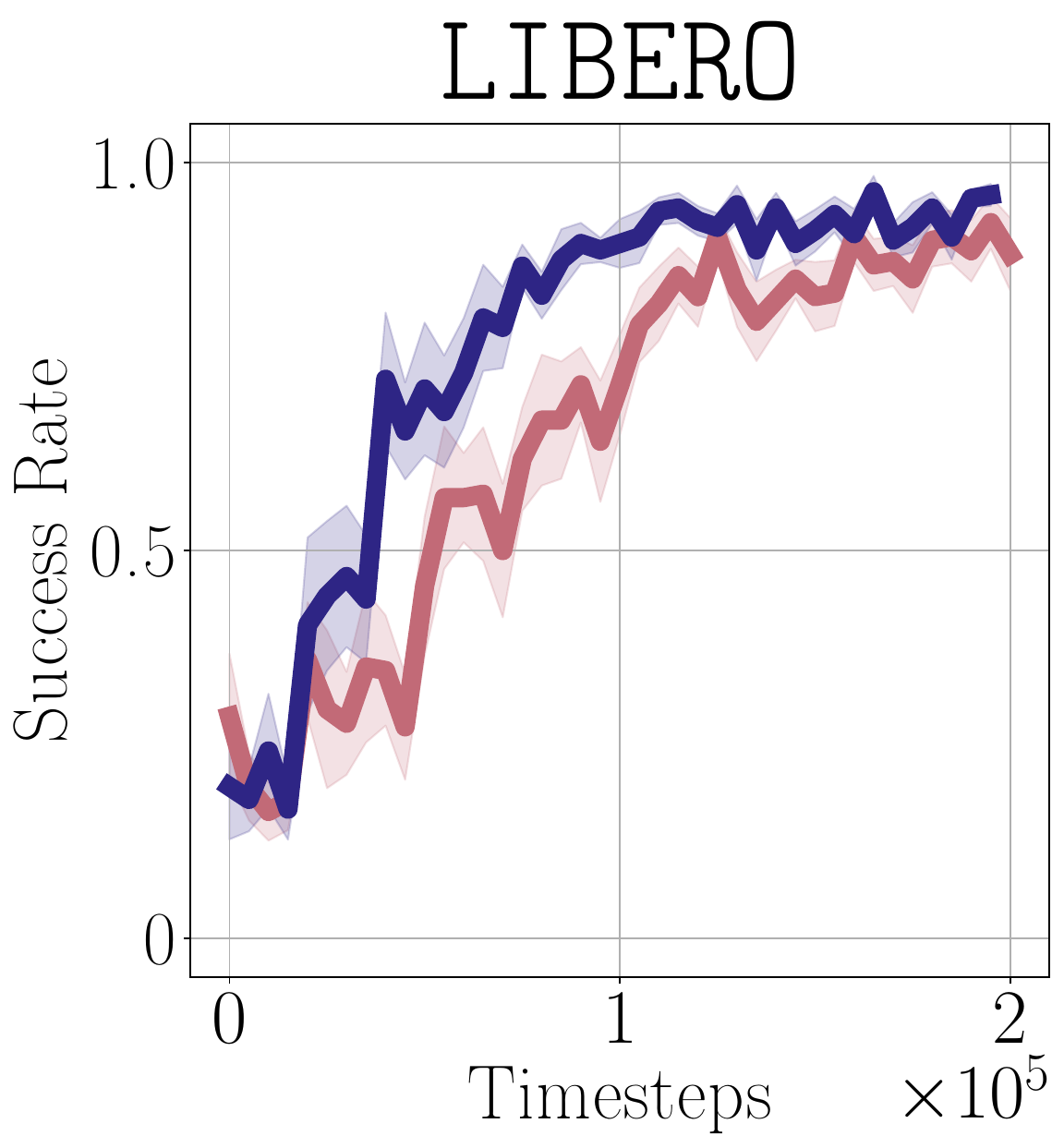}
  \caption{
    RL finetuning of $\pi_0$ with \dsrl and \prvm.
  }
  \label{fig:dsrl_pi0_exp}
  \end{minipage}
  \hfill
  \begin{minipage}[t]{0.32\textwidth}
  \includegraphics[width=0.535\linewidth]{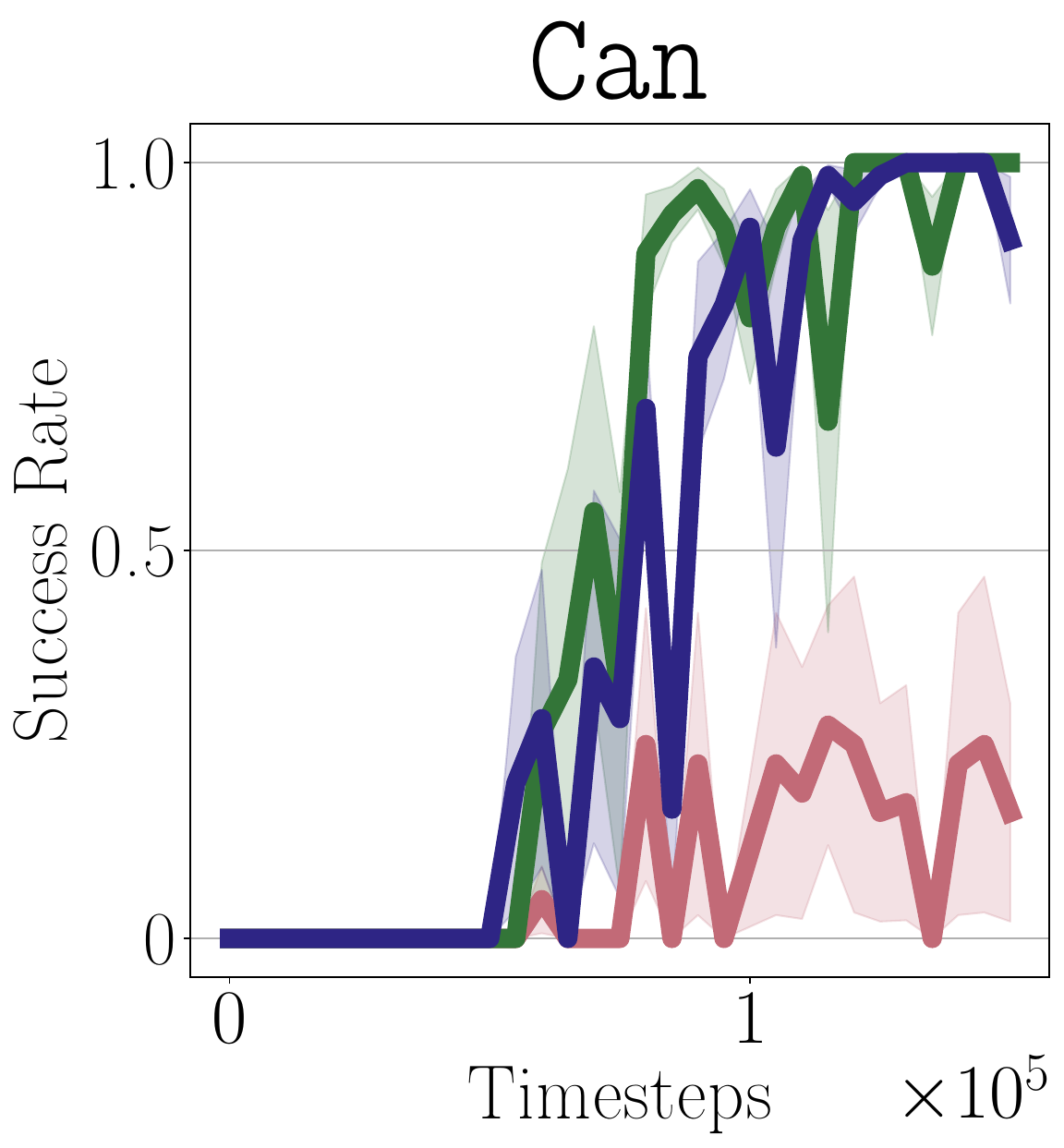}\hfill
  \includegraphics[width=0.45\linewidth]{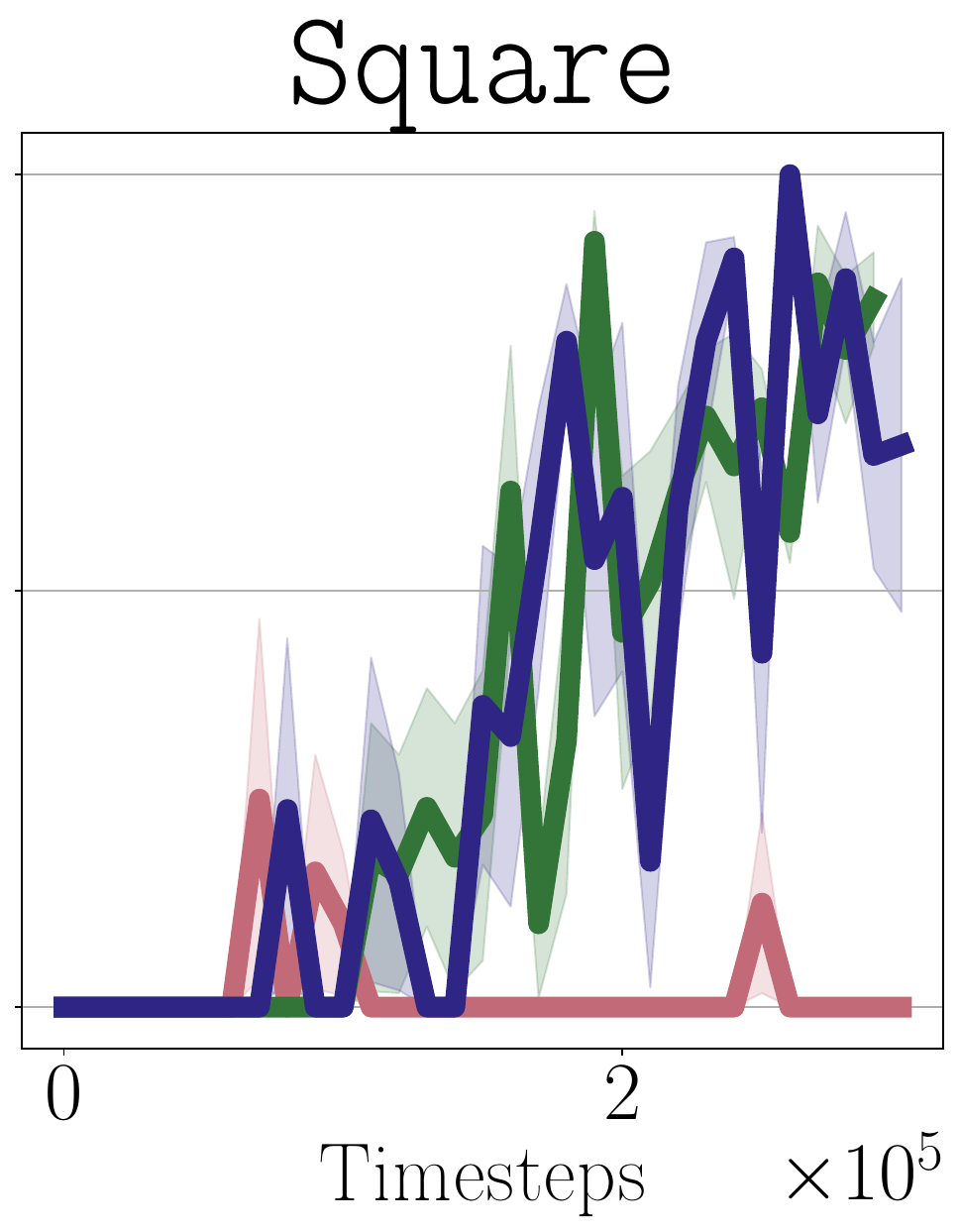}
  \caption{RL from demonstrations with \textsc{RLPD} and \prvm process reward on \texttt{Robomimic}.}
  \label{fig:robomimic_exp}
  \end{minipage}

  \vspace{-1.5em}
\end{figure*}

\iftoggle{arxiv}{}{\vspace{-0.25em}}
\subsection{\prvm Process Rewards Enable Efficient Real-World RL for Robotic Control Policies}\label{sec:exp_real}
\iftoggle{arxiv}{}{\vspace{-0.25em}}

We next demonstrate that \prvm scales to real-world robotic RL finetuning settings. 
For the pretrained policy, we train a multi-task diffusion policy on the BridgeData V2 dataset ~\cite{walke2024bridgedatav2datasetrobot}, which contains over 60{,}000 trajectories on a diverse set of tasks and scenes. 
\begin{wrapfigure}{r}{0.48\textwidth}
  \centering
  \vspace{-1em} 

  \raisebox{0.75em}{%
    \includegraphics[width=0.32\linewidth]{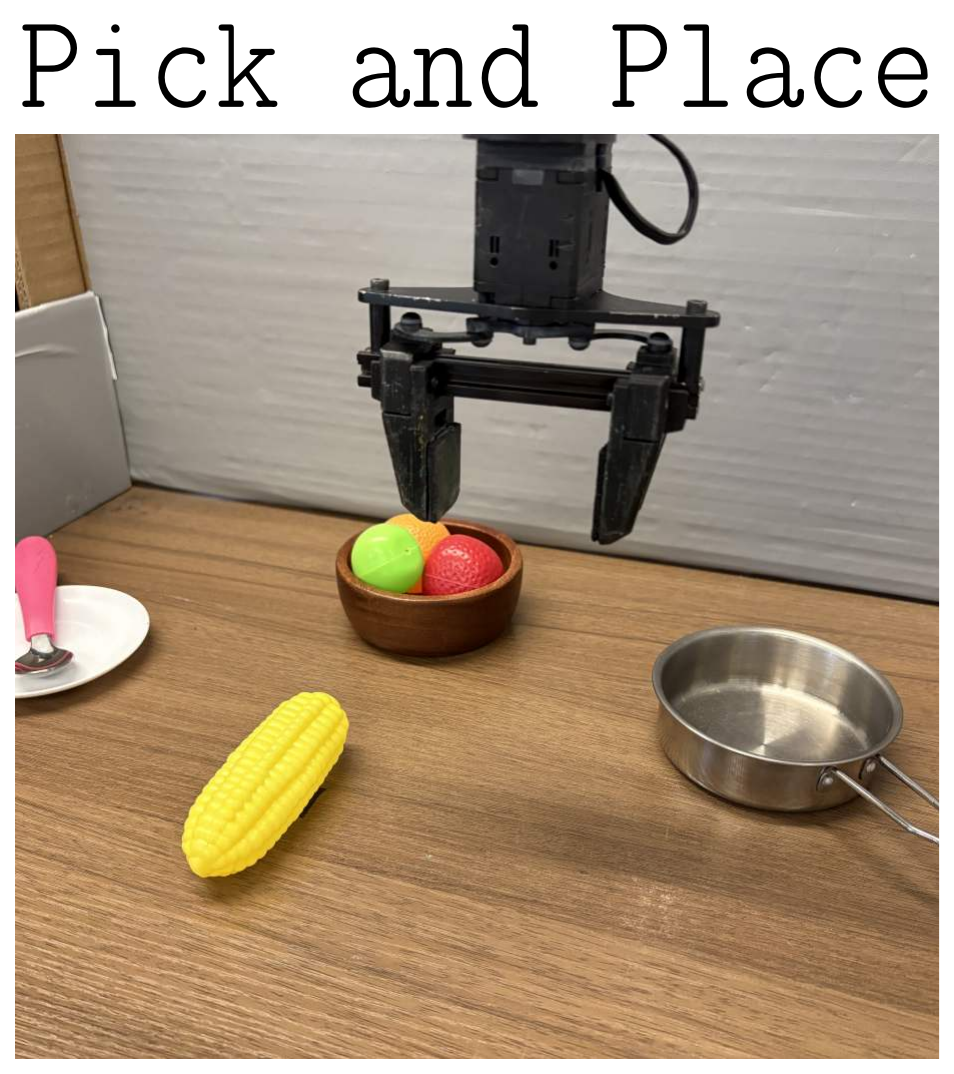}
  }\hfill
  \raisebox{0.75em}{%
    \includegraphics[width=0.32\linewidth]{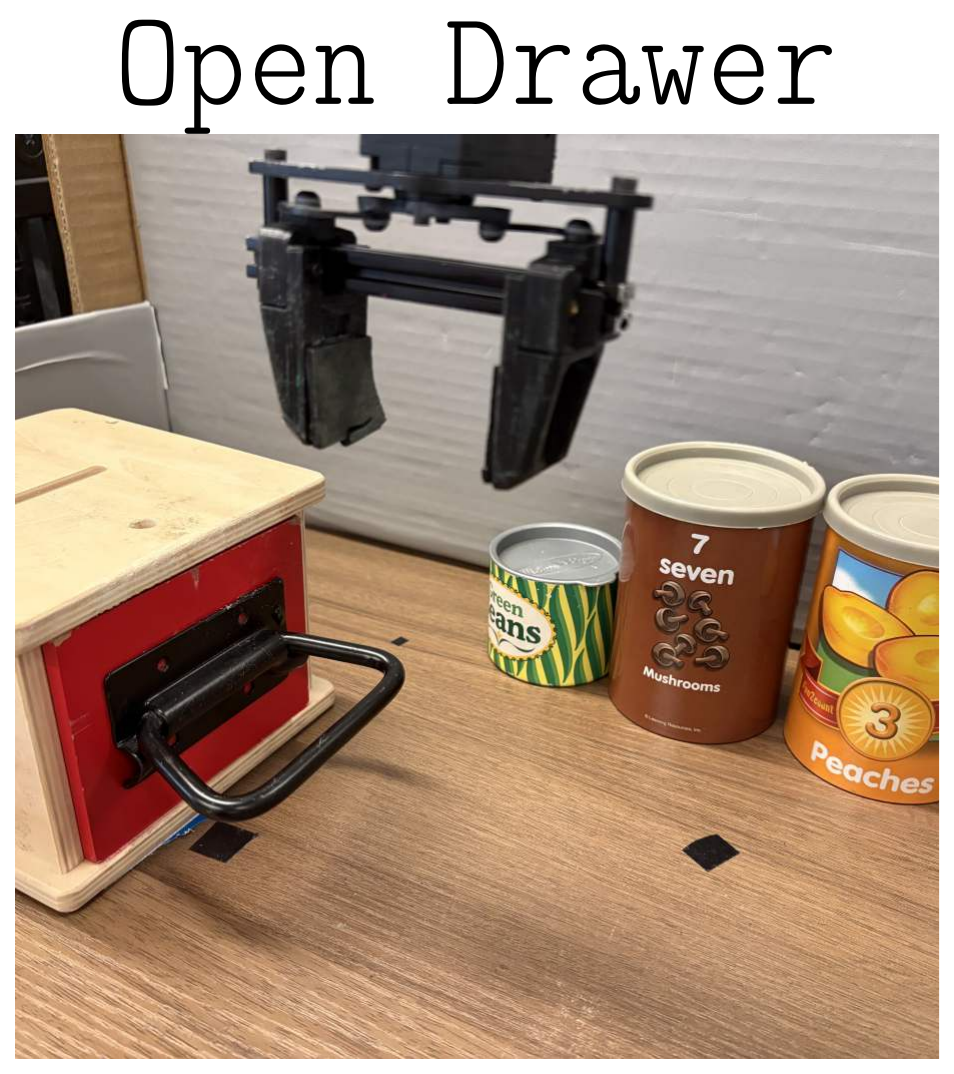}
  }\hfill
  \raisebox{0.75em}{%
    \includegraphics[width=0.32\linewidth]{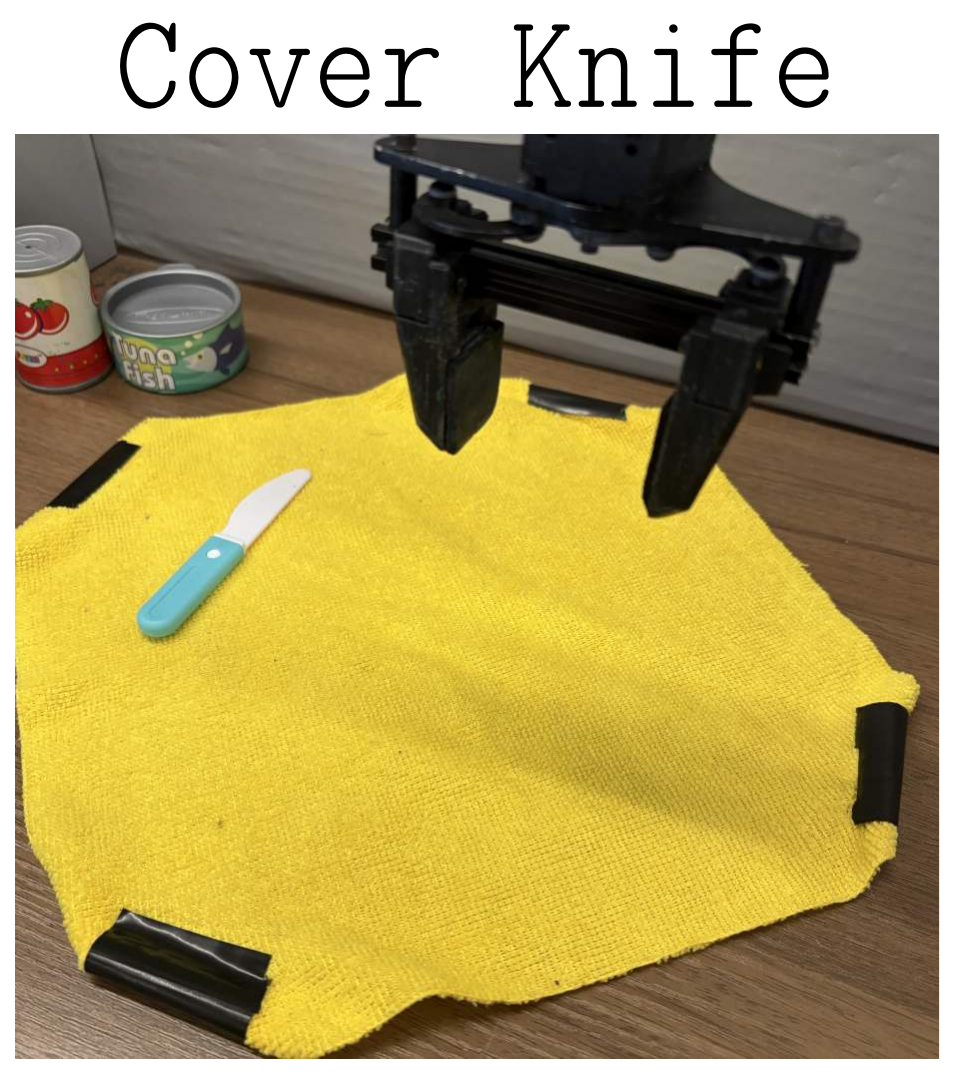}
  }
  \vspace{-2em}
  \caption{
    Visualization of real-world tasks on WidowX robotic arm.
    \texttt{Pick and Place}: pick up the corn and place it in the silver pot.
    \texttt{Open Drawer}: open the red drawer.
    \texttt{Cover Knife with Cloth}: lift the cloth and cover the knife.
  }
  \label{fig:widowx_exp_scene}

  \vspace{-1em} 
\end{wrapfigure}
We run \dsrl on three real-world tasks on the WidowX robotic arm: \texttt{Pick and Place}, \texttt{Open Drawer}, and \texttt{Cover Knife with Cloth}. Figure~\ref{fig:widowx_exp_scene} shows the scene setup and Figure \ref{fig:widowx_exp} the RL finetuning results, where each point is averaged across 20 evaluation rollouts. Across all tasks, the \prvm process reward substantially improves finetuning performance, enabling the policy to reach over $80\%$ success within a much smaller number of environment steps than running RL with only the outcome reward and, on \texttt{Pick and Place} and \texttt{Open Drawer}, enabling online improvement when training only on the outcome reward fails to learn. These results demonstrate the scalability of \prvm rewards, and its utility in enabling efficient real-world robotic RL.

\iftoggle{arxiv}{}{\vspace{-0.25em}}
\subsection{Finetuning VLAs with \prvm Process Rewards}
\iftoggle{arxiv}{}{\vspace{-0.25em}}

Next, we demonstrate that \prvm can improve the efficiency of running RL finetuning on pretrained generalist policies. Here, we consider $\pi_0$~\cite{black2024pi_0}, a 3.3B-parameter VLA with a VLM-based backbone and a flow-based action head. In particular, we use the publicly available \texttt{LIBERO} finetune of $\pi_0$. 
Following \cite{wagenmaker2025steering}, we consider RL finetuning on four \texttt{LIBERO-90} tasks for which $\pi_0$ achieves success in the 10-30\% range---tasks that leave room for RL improvement---and compare running RL with \prvm rewards to running RL with only outcome rewards.
Figure~\ref{fig:dsrl_pi0_exp} shows the finetuning performance of $\pi_0$ with \dsrl, aggregated across these tasks. We see that the \prvm process reward again yields substantial gains in sample efficiency, requiring roughly $2\times$ fewer environment steps to reach $90\%$ success, as compared to RL without reward shaping.

\iftoggle{arxiv}{}{\vspace{-0.25em}}
\subsection{RL from Demonstrations with \prvm Process Rewards}\label{sec:exp_rlpd}
\iftoggle{arxiv}{}{\vspace{-0.25em}}
While we have focused primarily on the finetuning regime, here we seek to understand whether \prvm rewards can also lead to improvement when running RL from scratch given a set of successful demonstrations. We focus on the \texttt{Robomimic} benchmark ~\cite{mandlekar2021matterslearningofflinehuman}, in particular the \texttt{Can} and \texttt{Square} tasks, and utilize \textsc{Rlpd} ~\cite{ball2023efficientonlinereinforcementlearning} as the base algorithm. 
We provide $200$ successful demonstrations, which are placed in the initial replay buffer.
We compare RL performance with the \prvm process reward to \gail and RL on only the outcome reward. As noted previously, here we run \gail exactly as described in \citet{ho2016generative}, utilizing the given demonstrations as positive examples and ignoring the outcome reward. For \prvm, we utilize the demonstrations to initialize $\Dpos$.

Our results are given in Figure~\ref{fig:robomimic_exp}. We see that \prvm leads to substantial gains over running RL on only the outcome reward, enabling the agent to successfully learn to solve the task. However, the performance of \prvm is essentially identical to that of \gail. We hypothesize that in this case, the rewards obtained by \gail and \prvm are approximately equivalent---since the RL policy is trained from scratch and starts with success rate of 0\%, the discriminator learned by \prvm and \gail is trained on similar sets of positive and negative examples. Furthermore, due to the simplicity of the \texttt{Robomimic} tasks, simply observing a handful of successful examples is likely sufficient to correctly classify a successful example without an outcome reward. Nonetheless, these results demonstrate that \prvm is an effective approach for reward shaping in RL regimes where we have access to demonstrations.\loose

\subsection{Understanding \prvm Process Rewards}\label{sec:ablations}

\begin{figure*}[t]
\centering
\begin{minipage}[t]{0.32\textwidth}
    \centering
    \begin{minipage}[t]{0.6\linewidth}
        \centering
        \includegraphics[width=\linewidth]{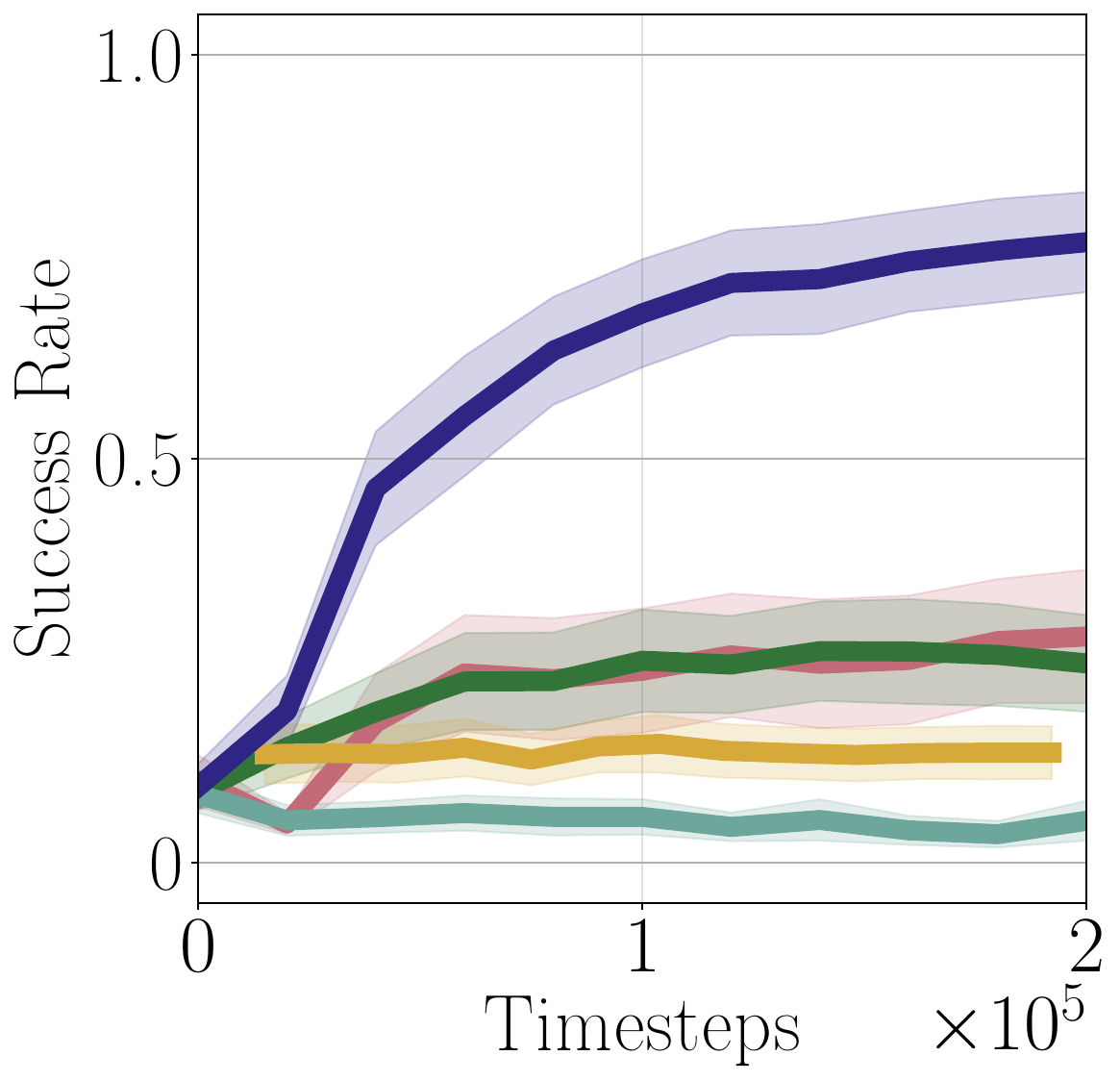}
    \end{minipage}
    \hfill
    \begin{minipage}[t]{0.22\linewidth}
        \centering
        \raisebox{0.35cm}{%
          \makebox[\linewidth][c]{\hspace*{-6mm}\includegraphics[height=3.4cm]{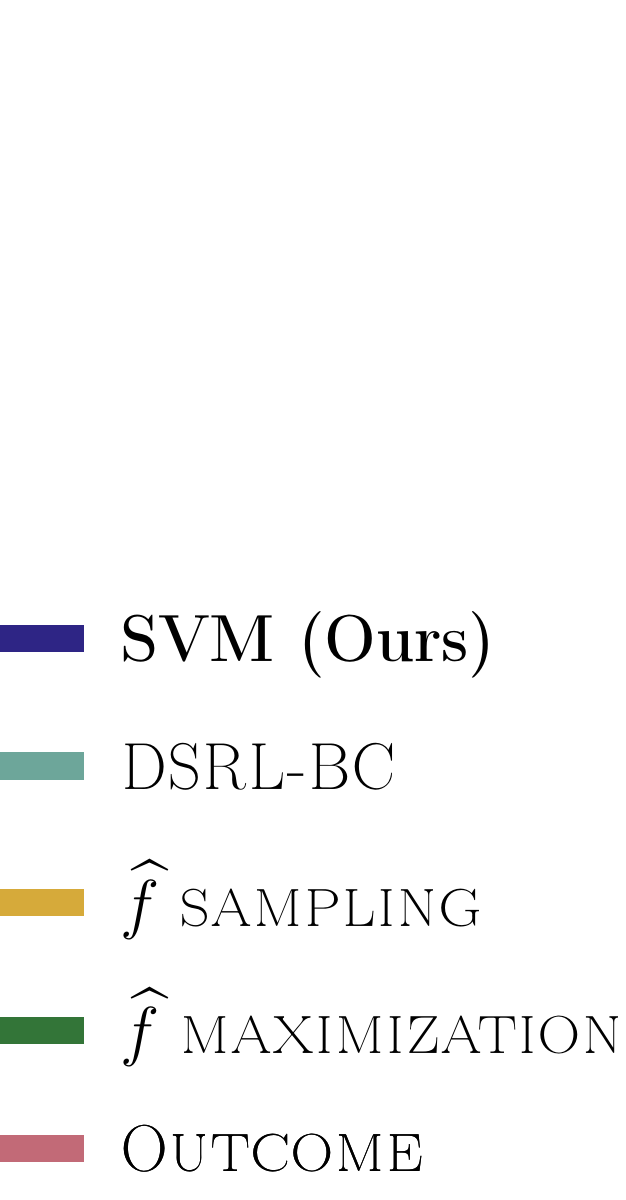}}%
        }
    \end{minipage}
    \caption{Ablation of policy extraction approach.}
    \label{fig:policy_extraction_ablation}
\end{minipage}
\hfill
\begin{minipage}[t]{0.32\textwidth}
    \centering
    \begin{minipage}[t]{0.6\linewidth}
        \centering
        \includegraphics[width=\linewidth]{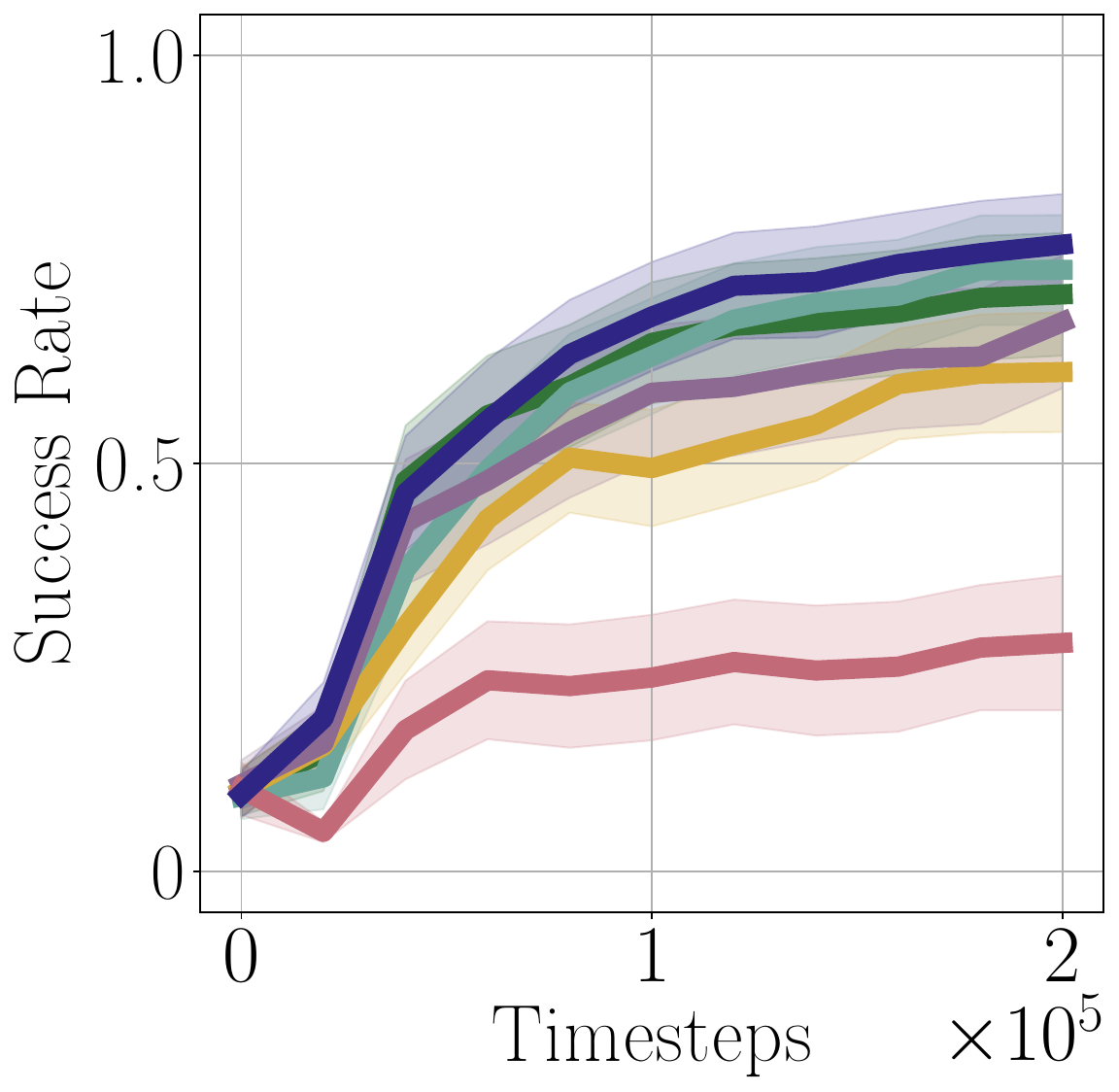}
    \end{minipage}
    \hfill
    \begin{minipage}[t]{0.22\linewidth}
        \centering
        \raisebox{0.35cm}{%
          \makebox[\linewidth][c]{\hspace*{-7mm}\includegraphics[height=3.4cm]{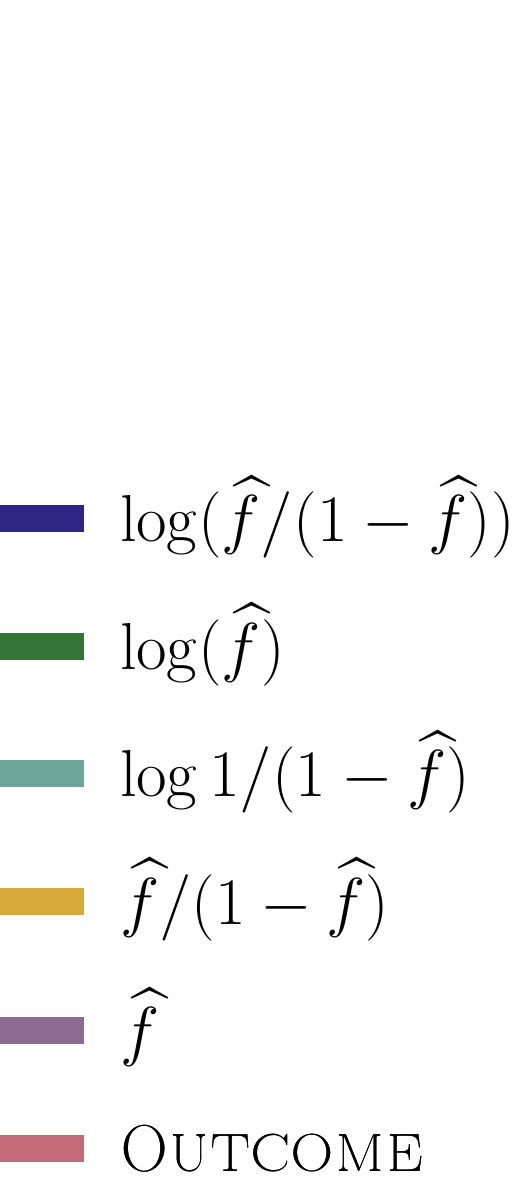}}%
        }
    \end{minipage}
    \caption{Ablation of functional form of \prvm.}
    \label{fig:reward_form_exp}
\end{minipage}
\hfill
\begin{minipage}[t]{0.32\textwidth}
    \centering
    \begin{minipage}[t]{0.6\linewidth}
        \centering
        {\hspace*{-1mm}\includegraphics[width=\linewidth]{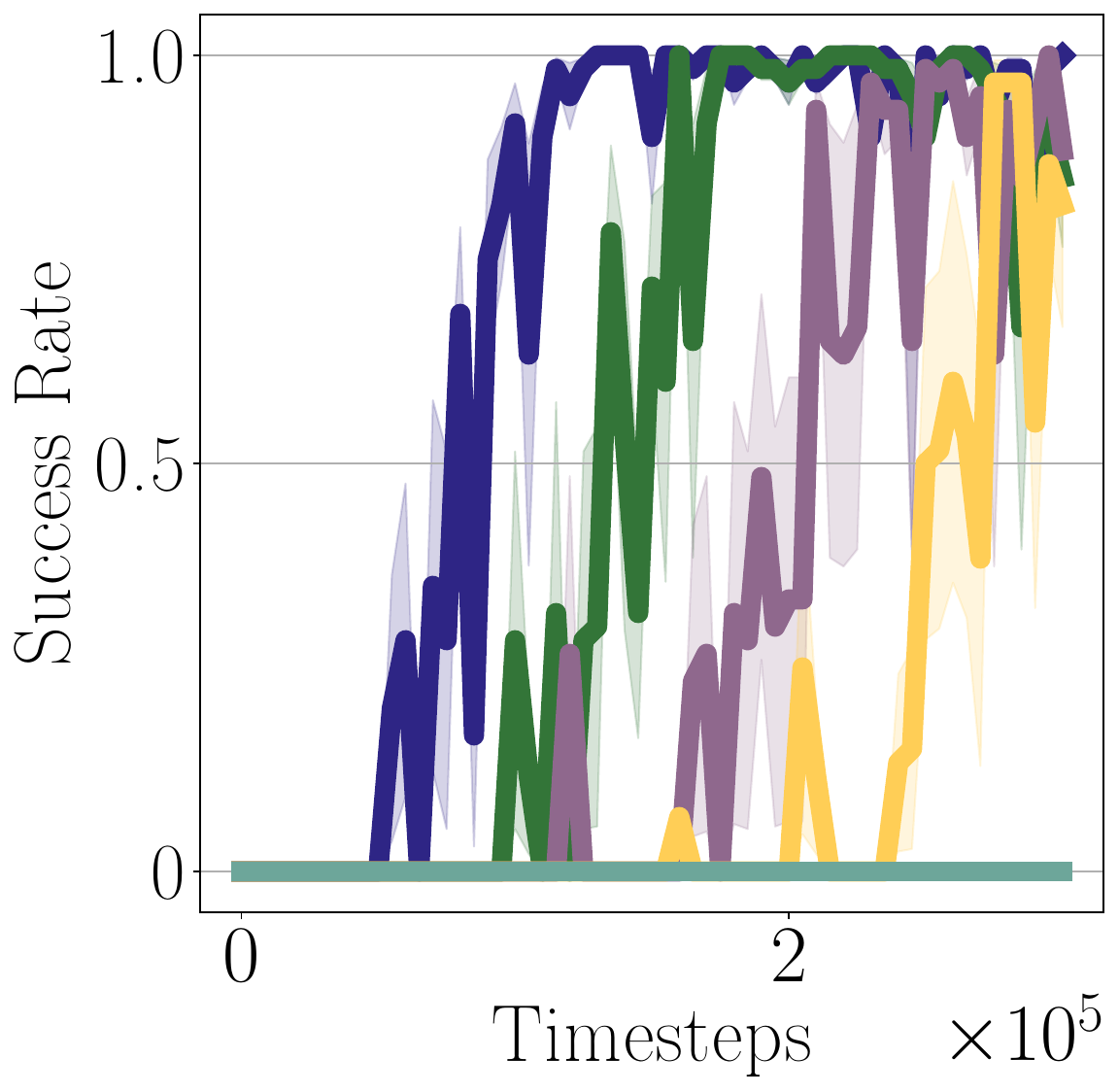}}
    \end{minipage}
    \hfill
    \begin{minipage}[t]{0.22\linewidth}
        \centering
        \raisebox{0.35cm}{%
          \makebox[\linewidth][c]{\hspace*{-10mm}\includegraphics[height=3.4
          cm]{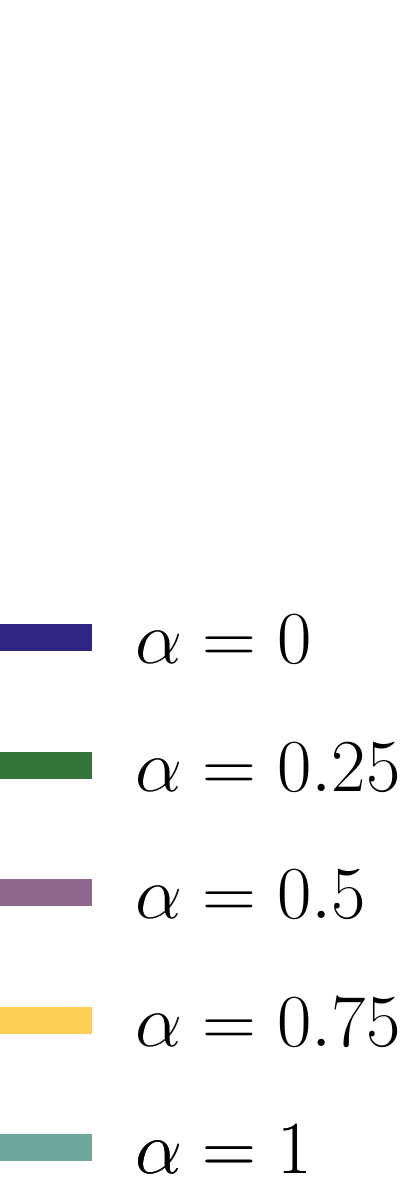}}%
        }
    \end{minipage}
    \caption{Ablation of $\%$ successful trajectories in $\Dneg$.}
    \label{fig:ablation_mixture}
\end{minipage}
\vspace{-1em}
\end{figure*}

Finally, we ablate several of the key design choices in \prvm and seek to understand how it leads to improved RL performance. Please see \Cref{sec:additional_ablations} for additional ablations.

\textbf{How can we most effectively extract a policy from $\fhat$?}
Our approach utilizes the learned discriminator $\fhat$ to define a process reward that we then learn to maximize with RL. However, other approaches for extracting a policy from $\fhat$ could be considered. Here we ablate our choice of policy extraction approach, and consider three alternate policy extraction approaches: 
\begin{itemize}[leftmargin=*]
    \item First, instead of running standard RL, we consider training a policy to maximize $\fhat$ directly. In particular, rather than training the \dsrl actor $\pi_{\textsc{Dsrl}}$ to maximize a learned $Q$-function, we train it to maximize $\fhat$, thereby encouraging the actor to directly maximize the estimated probability of an action leading to a success. We refer to this approach as ``$\fhat$ maximization''.
    \item Second, we utilize $\fhat$ to filter actions produced by the base policy $\pipre$---at each step in deployment sampling $N$ actions from $\pipre$, and then executing the action with the maximum value of $\fhat$. We note that this strategy is quite similar to the approach proposed in \citet{attarian2026update}, and we refer to it as ``$\fhat$ sampling''.\loose
    \item Finally, rather than training a discriminator $\fhat$ at all, we instead simply add a behavioral cloning term over $\frakD^+$ to the \dsrl actor loss. This encourages the \dsrl actor to increase the probability of re-producing the actions that were taken in $\frakD^+$ and led to success. We refer to this as ``\dsrl-BC''.
\end{itemize}
For the first two approaches, we train $\fhat$ as in \Cref{alg:rdrs_rl}, but use the alternate approach to policy extraction described above in place of the RL training proposed in \Cref{alg:rdrs_rl}.
We compare these approaches to our proposed approach of running RL with the \prvm process reward, and also compare to running RL with only the outcome reward, using \dsrl as our RL algorithm. We illustrate our results on \texttt{LIBERO Kitchen Scene 2} (all 7 tasks) in \Cref{fig:policy_extraction_ablation}. As these results illustrate, running RL with \prvm process rewards is significantly more effective than other approaches to policy extraction given $\fhat$ or successful episodes.

\textbf{How does the functional form of the process reward impact performance?} 
\prvm trains a discriminator $\fhat$ to define the process reward $\log \fhat/(1-\fhat)$. While we show in \Cref{sec:method} that this form of the reward corresponds to (a clipped version of) the KL-divergence between policy visitations, simply increasing $\fhat$ itself should encourage the policy to visit states likely to be in successful trajectories. 
Given this, we ablate several monotone transformations of the discriminator probability $\fhat$ to determine whether the exact form is critical. Here we consider \texttt{LIBERO Kitchen Scene 2} (all 7 tasks) and utilize \dsrl as the RL algorithm.
As shown in Figure~\ref{fig:reward_form_exp}, $\log \fhat/(1-\fhat)$ in general performs the best, yet all forms of process reward lead to significant gains over only maximizing outcome reward.\loose

\begin{figure}[t]
  \centering

  \begin{minipage}[t]{0.80\textwidth}
    \vspace{0pt}
    \centering
    \newcommand{\imgw}{0.155\linewidth}
    \newcommand{\heatmapimggap}{0.006\linewidth}

    \begin{minipage}[t]{\imgw}
      \centering
      \includegraphics[width=\linewidth]{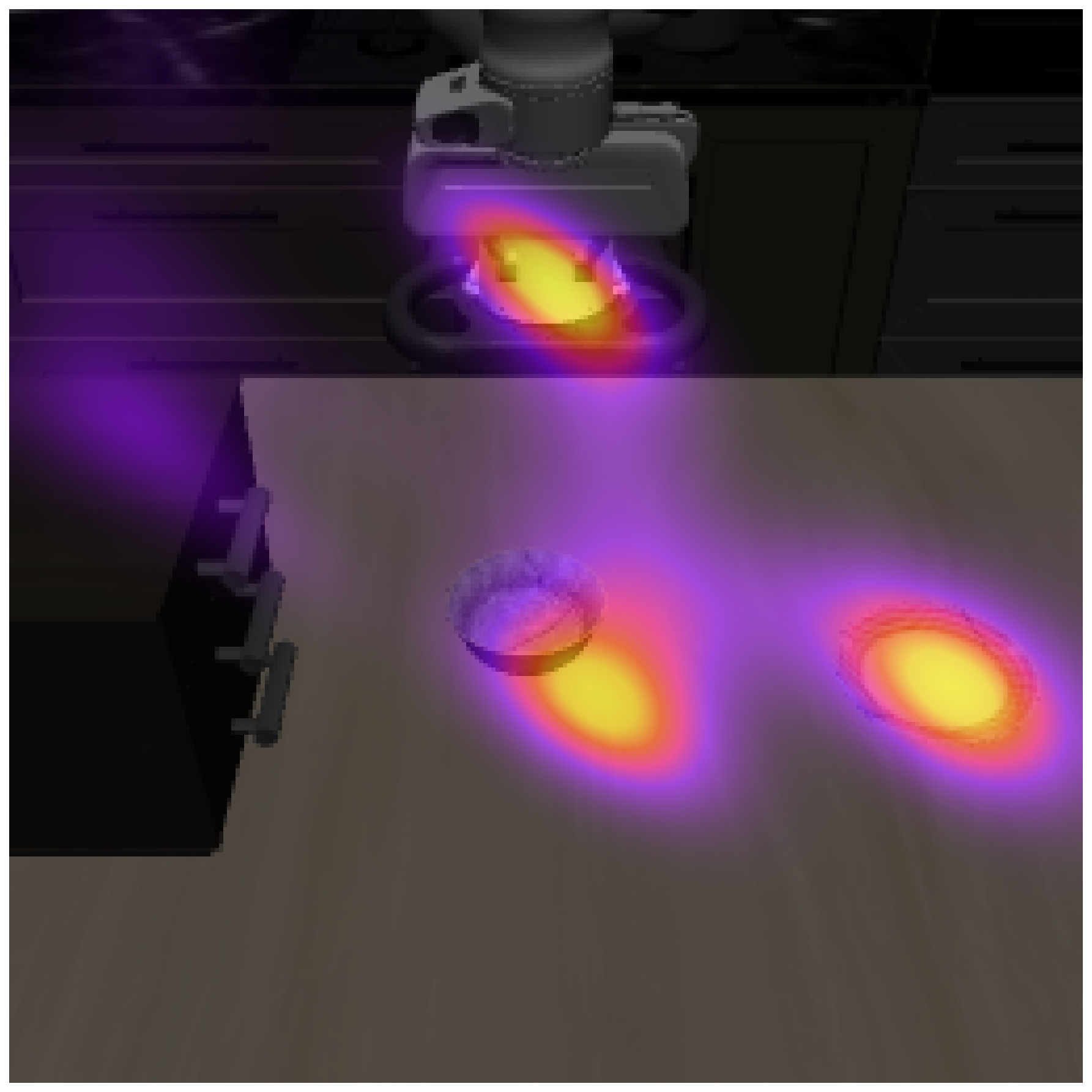}
      \subcaption*{{\scriptsize\textbf{\prvm}}}
    \end{minipage}%
    \hspace{\heatmapimggap}%
    \begin{minipage}[t]{\imgw}
      \centering
      \includegraphics[width=\linewidth]{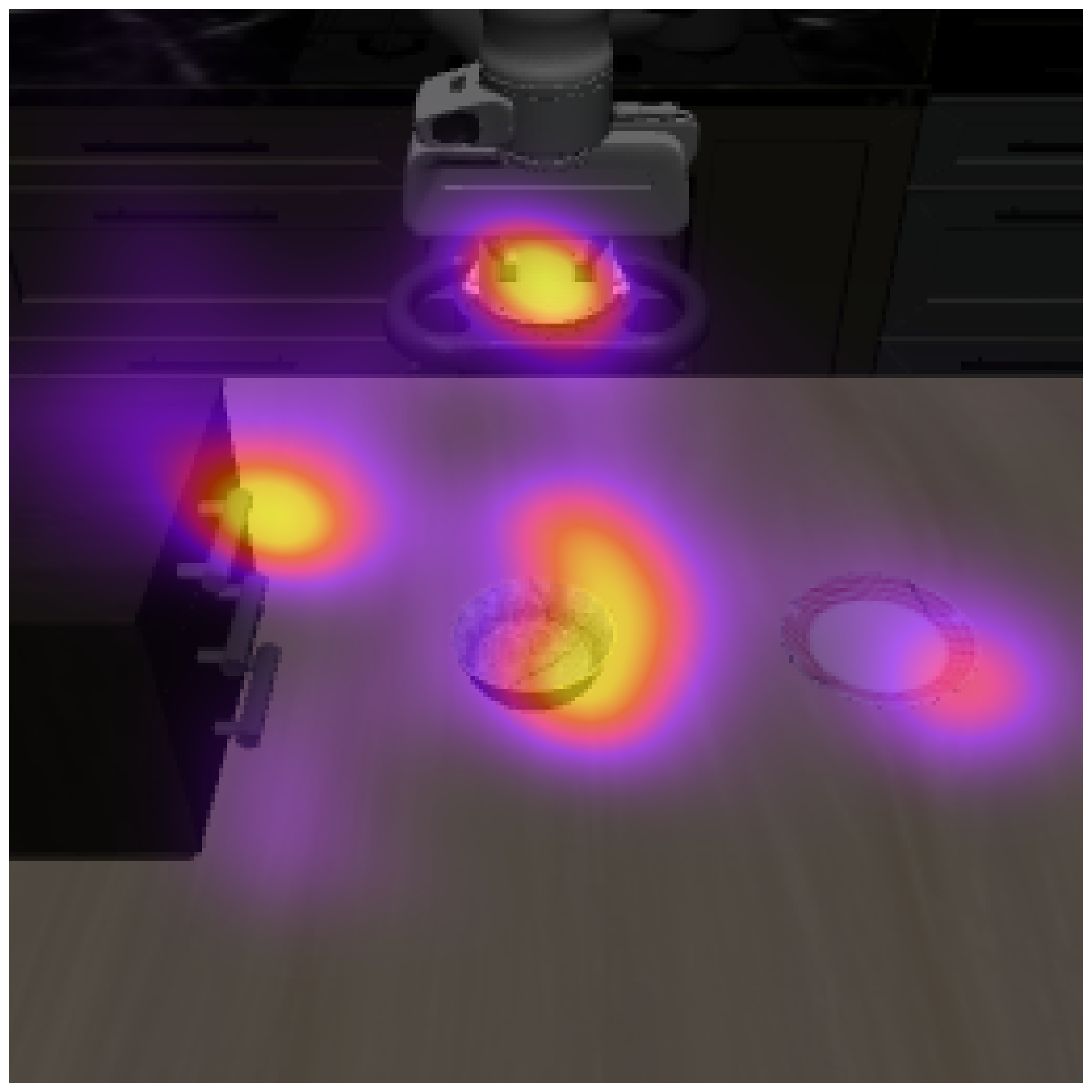}
      \subcaption*{{\scriptsize\textsc{Gail-reward}}}
    \end{minipage}%
    \hspace{\heatmapimggap}%
    \begin{minipage}[t]{\imgw}
      \centering
      \includegraphics[width=\linewidth]{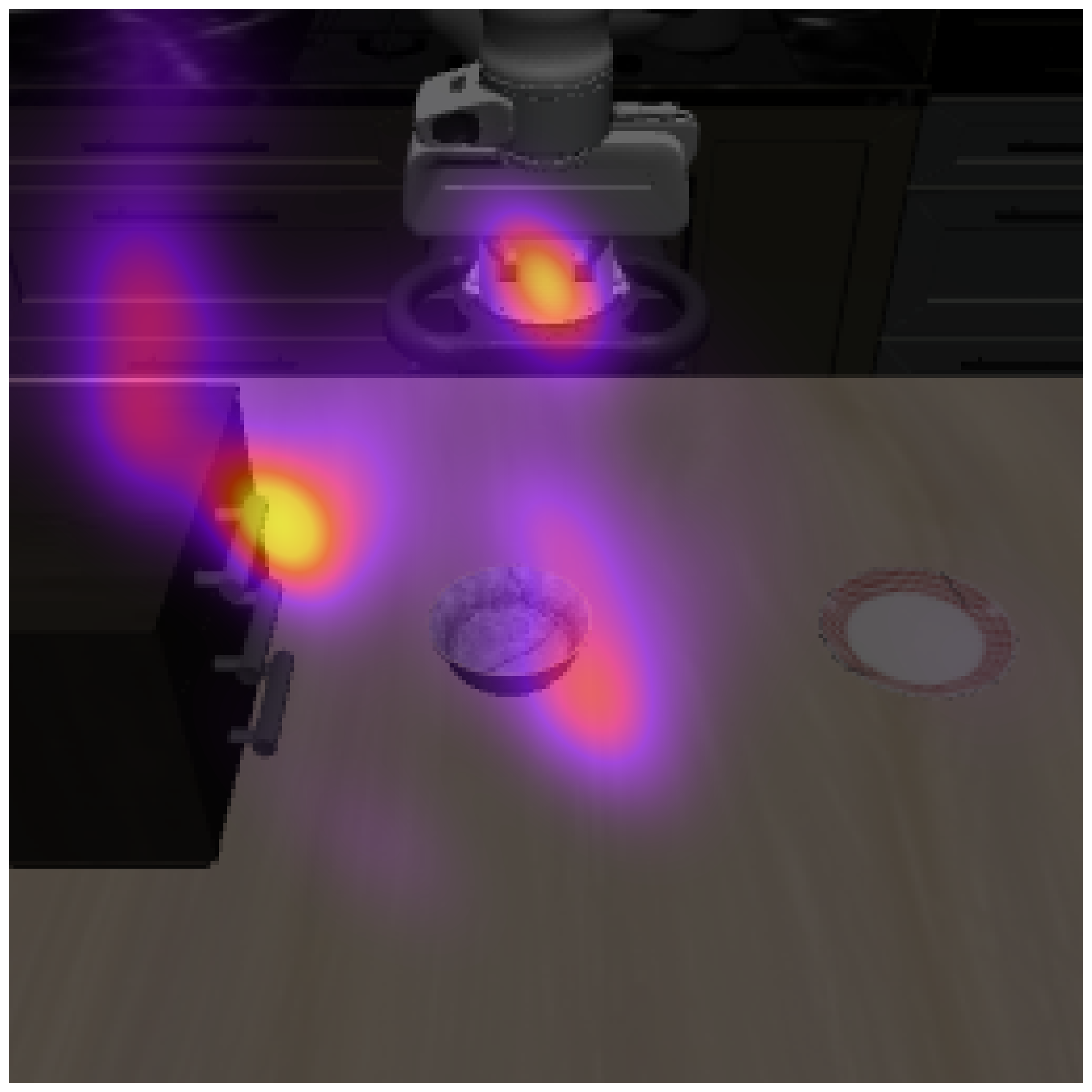}
      \subcaption*{{\scriptsize\textsc{Rnd}}}
    \end{minipage}%
    \hspace{\heatmapimggap}%
    \begin{minipage}[t]{\imgw}
      \centering
      \includegraphics[width=\linewidth]{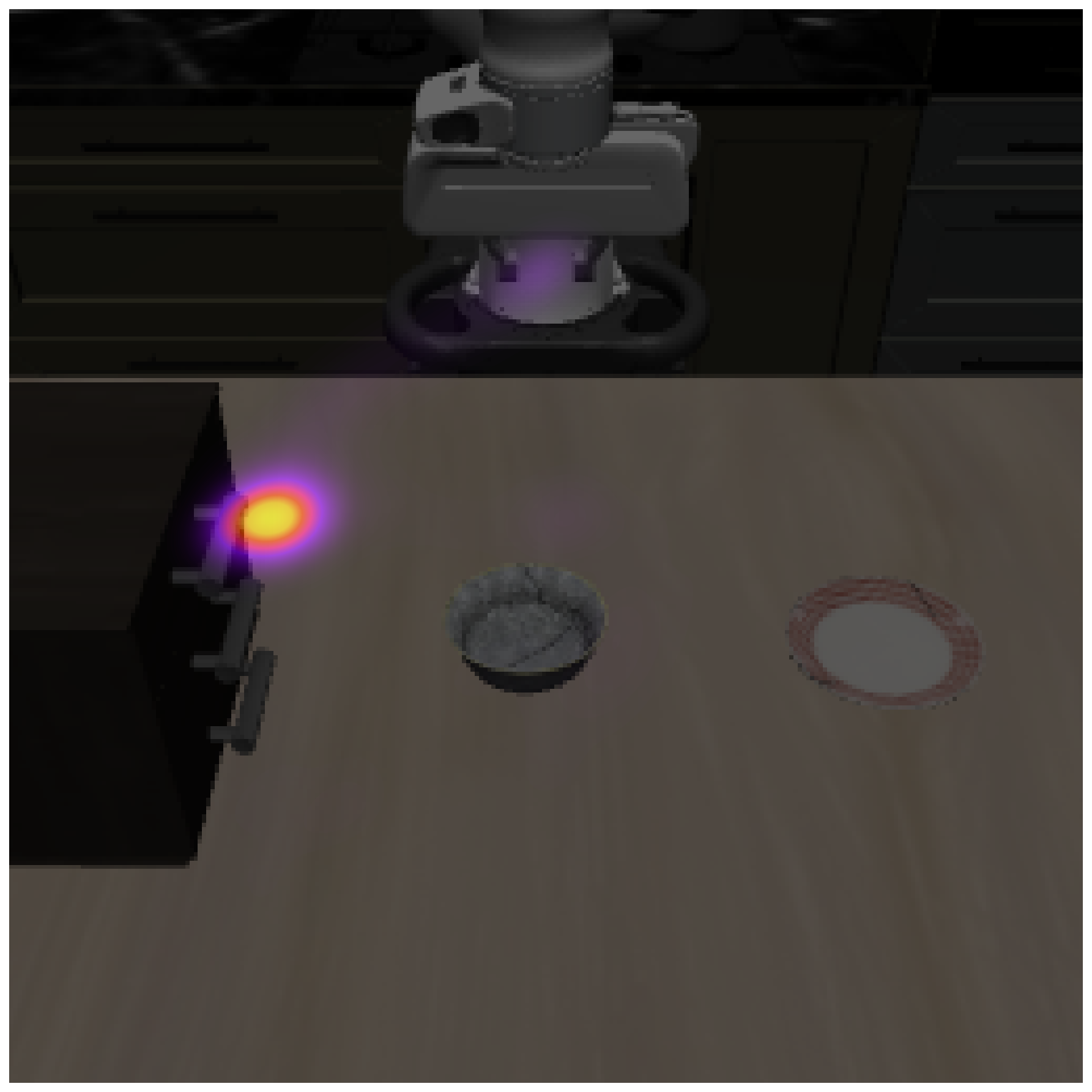}
      \subcaption*{{\scriptsize\textsc{SORS}}}
    \end{minipage}%
    \hspace{\heatmapimggap}%
    \begin{minipage}[t]{\imgw}
      \centering
      \includegraphics[width=\linewidth]{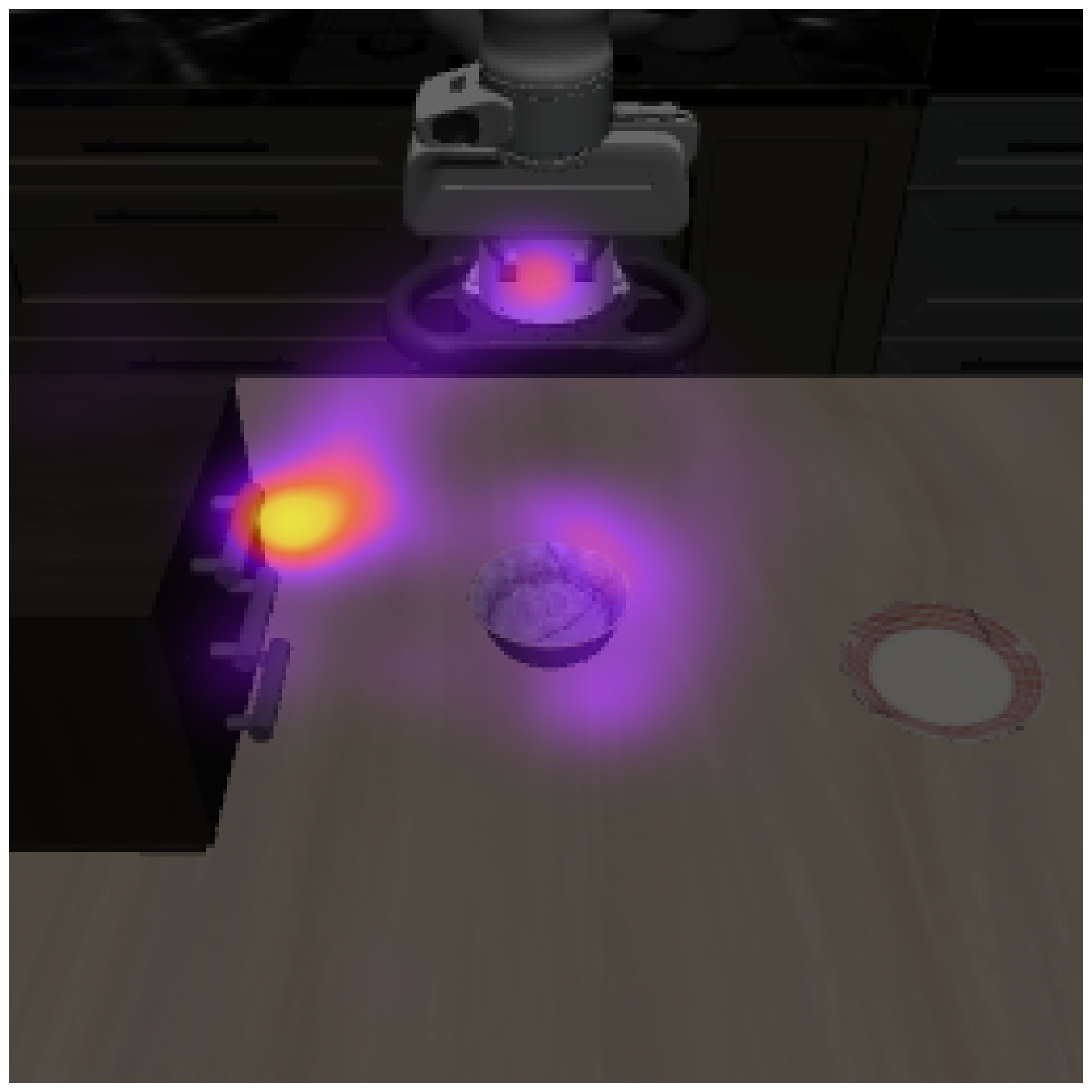}
      \subcaption*{{\scriptsize\textsc{SASR}}}
    \end{minipage}%
    \hspace{\heatmapimggap}%
    \begin{minipage}[t]{\imgw}
      \centering
      \includegraphics[width=\linewidth]{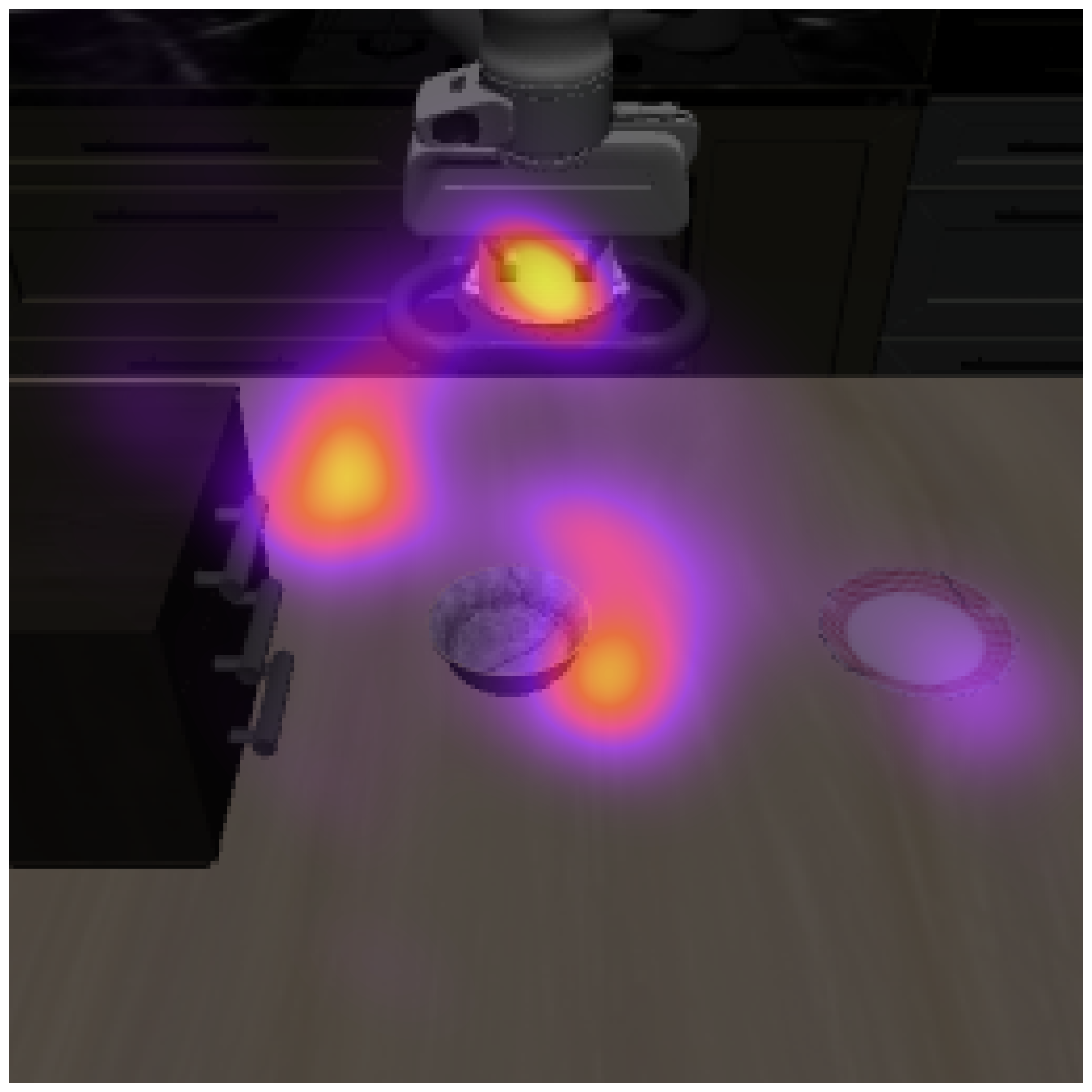}
      \subcaption*{{\scriptsize\textsc{Outcome}}}
    \end{minipage}
    \caption{State visitation heatmaps for different process reward methods at step 20000, estimated from $100$ training trajectories.}
    \label{fig:reward_heatmap}
  \end{minipage}
  \hfill
  \begin{minipage}[t]{0.17\textwidth}
    \vspace{0pt}
    \centering
    {\vspace*{-1.0mm}\includegraphics[width=\linewidth]{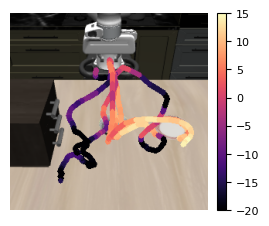}}
    \vspace{-0.5em}
    \caption{\prvm reward at step $20{,}000$.}
    \label{fig:reward_plot}
  \end{minipage}
  \vspace{-1em}
\end{figure}

\textbf{How does negative feedback impact \prvm performance?}
While it is obvious that rewarding the learner for visiting states that are likely to lead to success (the $\whatp$ term in \eqref{eq:rew_def}) should mitigate the credit assignment problem, it is not immediately clear that penalizing the learner for visiting states in unsuccessful trajectories (the $\whatm$ term in \eqref{eq:rew_def}) is necessary. Here we test whether this form of negative feedback is important to the performance of \prvm rewards. To test this, rather than setting $\Dneg$ in \eqref{eq:Dhat} to only negative trajectories, we set it to include proportion $\alpha$ successful trajectories, and $1-\alpha$ unsuccessful trajectories. Thus, $\alpha$ interpolates between explicitly penalizing the learner for visiting states in unsuccessful trajectories, and simply encouraging the learner to visit positive states relative to a neutral baseline. We consider different choice of $\alpha$, running on \texttt{Robomimic Can} in the RL from demonstrations setting of \Cref{sec:exp_rlpd}, and provide our results in Figure~\ref{fig:ablation_mixture}. As can be seen, negative feedback does significantly impact performance---the more the learner is penalized for visiting states in unsuccessful trajectories, the faster \prvm rewards enable improvement.

\textbf{Visualizing \prvm process rewards.} 
Finally, we seek to visualize  the process reward obtained
from \prvm, and how the trajectories it induces behave.
We consider the task \texttt{Put the black bowl on the plate}, which requires picking up the bowl at the center of the scene and setting it on the plate to the right. Figure~\ref{fig:reward_plot} visualizes the estimated \prvm reward at step $20{,}000$ computed along several trajectories. We see that the learned reward assigns high values to state-action pairs that lead to correct behavior (moving the bowl in the center to the plate on the right): approaching the bowl is rewarded, whereas moving toward irrelevant objects (e.g., the drawer) is penalized. 
In Figure~\ref{fig:reward_heatmap} we illustrate the state visitation heatmap, comparing \prvm to the baseline reward shaping approaches and only training on outcome reward. 
We see that the visitations induced by \prvm rapidly concentrate around regions corresponding to successful trajectories, while other approaches either lead to less focused trajectories (that visit successful regions but also visit unsuccessful regions), or trajectories that focus on unsuccessful regions.\loose

\iftoggle{arxiv}{}{\vspace{-0.5em}}
\section{Conclusion}
\iftoggle{arxiv}{}{\vspace{-0.5em}}
In this work, we have shown that by augmenting a sparse outcome reward with a reward that incentivizes staying close to previous successful trajectories and avoiding previous unsuccessful trajectories, we can obtain much faster RL convergence without changing the optimal policy. A primary limitation of our theoretical contribution is the assumption that the environment is deterministic. While many environments of interest are deterministic, and our experimental results show that even in non-deterministic real-world settings our approach leads to large improvements, extending the theoretical results to non-deterministic settings is of interest. In addition, while we have focused primarily on robotic control settings, our approach could apply equally well to settings such as RLVR for LLMs, which is an interesting direction for future work.

\newpage
\subsection*{Acknowledgments}
This research was partly supported by DARPA ANSR, AFOSR FA9550-22-1-0273, and ONR N00014-25-1-2060. This research used the Savio computational cluster resource provided by the Berkeley Research Computing program at UC Berkeley.

\bibliographystyle{plainnat}
\bibliography{bibliography}

\newpage
\appendix


\newcommand{\piproc}{\pi^{\mathrm{proc}}}

\section{Additional Related Work}\label{sec:additional_related}

\textbf{Process rewards for LLMs.}
With the recent success of the ``RL with verifiable rewards'' paradigm in language domains \cite{guo2025deepseek,team2025kimi}, significant attention has been given to designing process rewards in order to enable more efficient RL for LLM reasoning. 
Early works in this direction learn process reward from dense human or AI feedback \cite{lightman2023let,luo2023wizardmath}. More recently, works such as \cite{yuan2024free, cui2025process} learn ``implicit'' process rewards, by training an outcome reward model, and utilizing the log probabilities this induces as a dense reward signal. Most similar to our approach are works such as \cite{snell2024scaling, luo2024improve, wang2024math, setlur2024rewarding, choudhury2025process, qu2025optimizing}, which utilize a learned classifier or advantage function that models the future success probability from a given state, providing the learner with a signal of progress as a dense source of supervision. While conceptually similar to our approach, these approaches are not, in general, implementable in robotic settings, as they require the ability to roll out many episodes from each state to compute success probabilities; in robotic settings, where we cannot reset at will, such approaches cannot be instantiated. Furthermore, the functional form for the process reward these works propose differs significantly from our proposed reward and, as we show in the following, our form outperforms those proposed in the LLM literature.

\newcommand{\sterm}{s_{\mathrm{term}}}
\newcommand{\frakDplus}{\mathfrak{D}^+}
\newcommand{\frakDmin}{\mathfrak{D}^-}
\newcommand{\Np}{N^+}
\newcommand{\Nm}{N^-}
\newcommand{\aplus}{a^+}
\newcommand{\Vtil}{\widetilde{V}}

\section{Proof of \Cref{thm:proc_consistent}}\label{sec:proof}
For \Cref{thm:proc_consistent}, we make one additional modification to the environment definition. As stated in \Cref{sec:prelim}, if a state with $\rout(s) = 1$ is reached, the environment transitions to a terminal state and remains there for the rest of the episode, getting a reward of 0. We denote this state as $\sterm$ and assume that $\sterm$ can only be transitioned to after a reward of 1 has been obtained, so that $\sterm$ is only reached on successful episodes. Note that an environment can be trivially modified to satisfy this criteria.

We say an episode is ``successful'' if it has a reward of 1, and ``unsuccessful'' otherwise. 
Let $\Np := |\frakDplus|$ and $\Nm := |\frakDmin|$ the number of successful and unsuccessful episodes, respectively, and $\Np_h(s,a) := \sum_{(s',a') \in \frakDplus_h} \bbI \{ s' = s, a' = a \}$ (similarly $\Nm_h(s,a)$).
Given this notation, the reward stated in \Cref{thm:proc_consistent} can be rewritten as:
\begin{align*}
\rvm_h(s,a) := \rout(s) + \lambda \cdot \clip \left ( \log \frac{\Np_h(s,a) / \Np}{\Nm_h(s,a) / \Nm}, \beta \right )
\end{align*}
for some $\lambda, \beta > 0$.
We also adopt the convention that $\log 0/0 = -\infty$.
For notational convenience, we say an episode starts at $h = 1$ and finishes at $h = H$. We say that $(s,h)$ is ``success reachable'' if there exists some $s'$ with $\rout(s') = 1$ such that $s'$ can be reached from $s$ in at most $H - h$ steps.

\begin{proof}[Proof of \Cref{thm:proc_consistent}]
Let $\Vtil$ denote the value function under $\rvm$ (that is, $\Vtil_h(s) = \Exp^\pi[\sum_{h'=h}^H \rvm_h(s_h,a_h)]$).
To show this result, we aim to show that for any $(s,h)$, for $h \in \{ 1, \ldots, H \}$, any policy that maximizes $\rvm$ from $(s,h)$ will also maximize $\rout$. We consider two cases, and assume that $s \neq \sterm$ for simplicity, since if $s = \sterm$ by assumption the episode has already succeeded so the actions that the policy plays do not affect its optimality. 

\paragraph{Case 1: $(s,h)$ is not success reachable.} 
If $(s,h)$ is not success reachable and $s \neq \sterm$, then it is not possible to reach a successful state from $(s,h)$, by the definition of success reachability. Thus, any policy achieves the same outcome reward from $(s,h)$, so its trivially true that a policy maximizing $\rvm$ from $(s,h)$ also maximizes $\rout$.

Furthermore, no successful state would have been reached before $(s,h)$ (since all successful states immediately transition to $\sterm$ and remain there). 
Thus, $\Np_h(s,a) = 0$
The definition of success reachability also implies $\rout(s) = 0$. Thus, $\rvm_h(s) = -\lambda \beta$ (using the convention that $\log 0/0 = -\infty$ to handle the case when $\Nm_h(s) = 0$). 
By the definition of success reachability, once we are in $(s,h)$ that is not success reachable and $s \neq \sterm$, then we will be in states that are not success reachable for all subsequent steps in the episode. 
Altogether this implies that, if we are in $(s,h)$ that is not success reachable and $s \neq \sterm$, then we will have $\Vtil_h^\pi(s) = -\lambda \beta (H - h + 1)$ for all $\pi$. 

\paragraph{Case 2: $(s,h)$ is success reachable.}
Consider some $(s,h)$ that is success reachable. We now prove by backward induction that any trajectory starting from $(s,h)$ that maximizes $\rvm$ must be successful (reach a state with $\rout(s) = 1$).
For the base case, with $h = H$, by definition of success reachability $\rout(s) = 1$, so the trajectory is successful.

Assume now that for all $(s,h+1)$ that are success reachable, any trajectory starting from $(s,h+1)$ that maximizes $\rvm$ is successful. Consider some $(s,h)$ that is success reachable. By definition of success reachability, either $\rout(s) = 1$, or there exists some $\aplus$ such that for $s' = P(s,\aplus)$, $(s',h+1)$ is success reachable. In the former case we are already done since the trajectory is already successful, so assume we are in the latter case. By the inductive assumption, any trajectory starting from $(s',h+1)$ which maximizes $\rvm$ succeeds. Thus, since $\rvm_h(s) \ge -\lambda \beta$, and we will visit a successful state in the episode, any trajectory maximizing $\rvm$ that starts from $(s',h+1)$ must have reward at least $1 - \lambda \beta (H - h)$. 
Consider some $a$ such that $s'' = P(s,a)$ is not success reachable. As shown above, $\Vtil_{h+1}^\pi(s'') = -\lambda\beta (H - h)$ for any $\pi$. Thus, taking action $\aplus$ and then playing a policy that maximizes $\rvm$ will lead to a reward of at least $1 - \lambda \beta (H-h)$, while taking an action $a$ that leads to a state which is not success reachable will lead to a reward of $-\lambda \beta (H - h)$. 

It follows that either $\rout(s) = 1$, or any policy that maximizes $\rvm$ from $(s,h)$ must take an action that ensures the next state is also success reachable. By the inductive assumption, for all $(s,h+1)$ that are success reachable, any trajectory starting from $(s,h+1)$ that maximizes $\rvm$ is successful. It follows that any trajectory that maximizes $\rvm$ from $(s,h)$ is successful, which proves the inductive hypothesis.

\paragraph{Concluding the proof.} Thus, if $(s,1)$ is success reachable, any trajectory maximizing $\rvm$ from $(s,1)$ succeeds. Under the original outcome reward, this trajectory achieves a reward of 1. As this is the maximum possible reward, it follows that any trajectory maximizing $\rvm$ from $(s,1)$ maximizes the original outcome reward. This proves the result.
\end{proof}

\section{Additional Experimental Results}\label{sec:exp_individual_results}

In this section we provide additional experimental results, including base policy success rates, individual per-task learning curves, and additional ablations. See \Cref{fig:all_scenes_for_exp} for a visualization of all simulated experimental environments we consider.

\begin{figure*}[h]
  \centering
  \newcommand{\sceneimgw}{0.1355\textwidth}

  \includegraphics[width=\sceneimgw]{images/scene_images/robomimic_can_scene.png}%
  \hfill
  \includegraphics[width=\sceneimgw]{images/scene_images/libero_kitchen_scene_1.png}%
  \hfill
  \includegraphics[width=\sceneimgw]{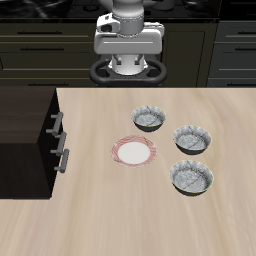}%
  \hfill
  \includegraphics[width=\sceneimgw]{images/scene_images/libero_kitchen_scene_3.png}%
  \hfill
  \includegraphics[width=\sceneimgw]{images/scene_images/robocasa_banana_scene.png}%
  \hfill
  \includegraphics[width=\sceneimgw]{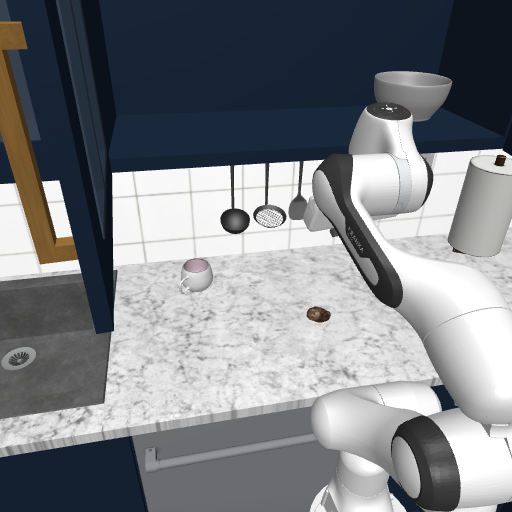}%
  \hfill
  \includegraphics[width=\sceneimgw]{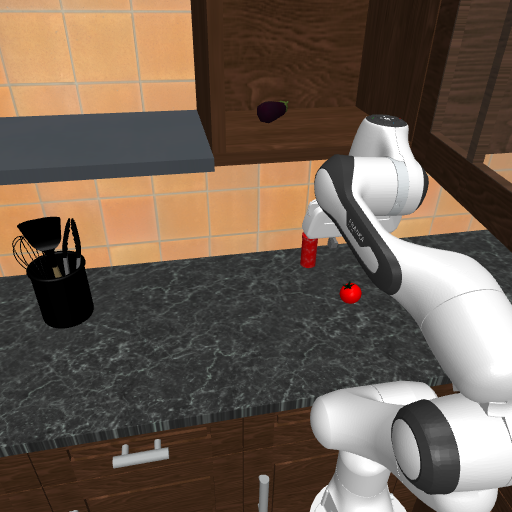}

  \caption{\texttt{Robomimic}, \texttt{LIBERO}, and \texttt{Robocasa} scenes.}
  \label{fig:all_scenes_for_exp}
\end{figure*}

\subsection{Additional information on LIBERO tasks}

For our \dsrl and \resrl experiments, we evaluate on \texttt{LIBERO Kitchen Scene 1--3}, covering tasks $6$--$21$. Specifically, Scene 1 contains tasks $6$--$10$, Scene 2 contains tasks $11$--$17$, and Scene 3 contains tasks $18$--$21$. We also provide the initial success rate of the diffusion transformer base policy and the $\pi_0$ base policy used in our experiments in Table~\ref{tab:libero_success_rate} and~\ref{tab:pi0_libero_success_rate}.

\begin{table}[h]
\centering
\begin{tabular}{c c c c c c c c c}
\hline
Task ID        & 6 & 7 & 8 & 9 & 10 & 11 & 12 & 13  \\ 
\hline
Success Rate   & 0.05 & 0.52 & 0.40 & 0.10 & 0.20 & 0.32 & 0.02 & 0.08\\ 
\hline
\end{tabular}\\
\vspace{3mm}
\begin{tabular}{c c c c c c c c c}
\hline
Task ID        & 14 & 15 & 16 & 17 & 18 & 19 & 20 & 21  \\ 
\hline
Success Rate   & 0.01 & 0.03 & 0.22 & 0.02 & 0.44 & 0.15 & 0.55 & 0.26\\ 
\hline
\end{tabular}
\vspace{3mm}
\caption{Average success rate of the base policy before RL finetuning on each \texttt{LIBERO} task.}
\label{tab:libero_success_rate}
\end{table}

\begin{table}[H]
\centering
\begin{tabular}{c c c c c}
\hline
Task ID        & 20 & 22 & 38 & 79  \\ 
\hline
Success Rate   & 0.30 & 0.30 & 0.10 & 0.10 \\ 
\hline
\end{tabular}
\vspace{3mm}
\caption{Average success rate of the $\pi_0$ base policy before RL finetuning on each \texttt{LIBERO} task.}
\label{tab:pi0_libero_success_rate}
\end{table}

\subsection{Individual results for DSRL and Residual RL on Libero}

We provide the individual task results for \dsrl and \resrl on Libero $\texttt{Kitchen Scene 1-3}$ in Figure~\ref{fig:full_dsrl_libero} and ~\ref{fig:full_resrl_libero}. Colors correspond to the same methods as in the main text.

\begin{figure}[H]
  \centering
  \begin{subfigure}[t]{0.98\textwidth}
    \centering
    \includegraphics[width=0.98\linewidth]{images/legend.pdf}
  \end{subfigure}
  \vspace{0.15em}

  \begin{subfigure}[t]{0.13\textwidth}
    \centering
    \includegraphics[width=\linewidth,
      height=2.75cm,
      keepaspectratio]{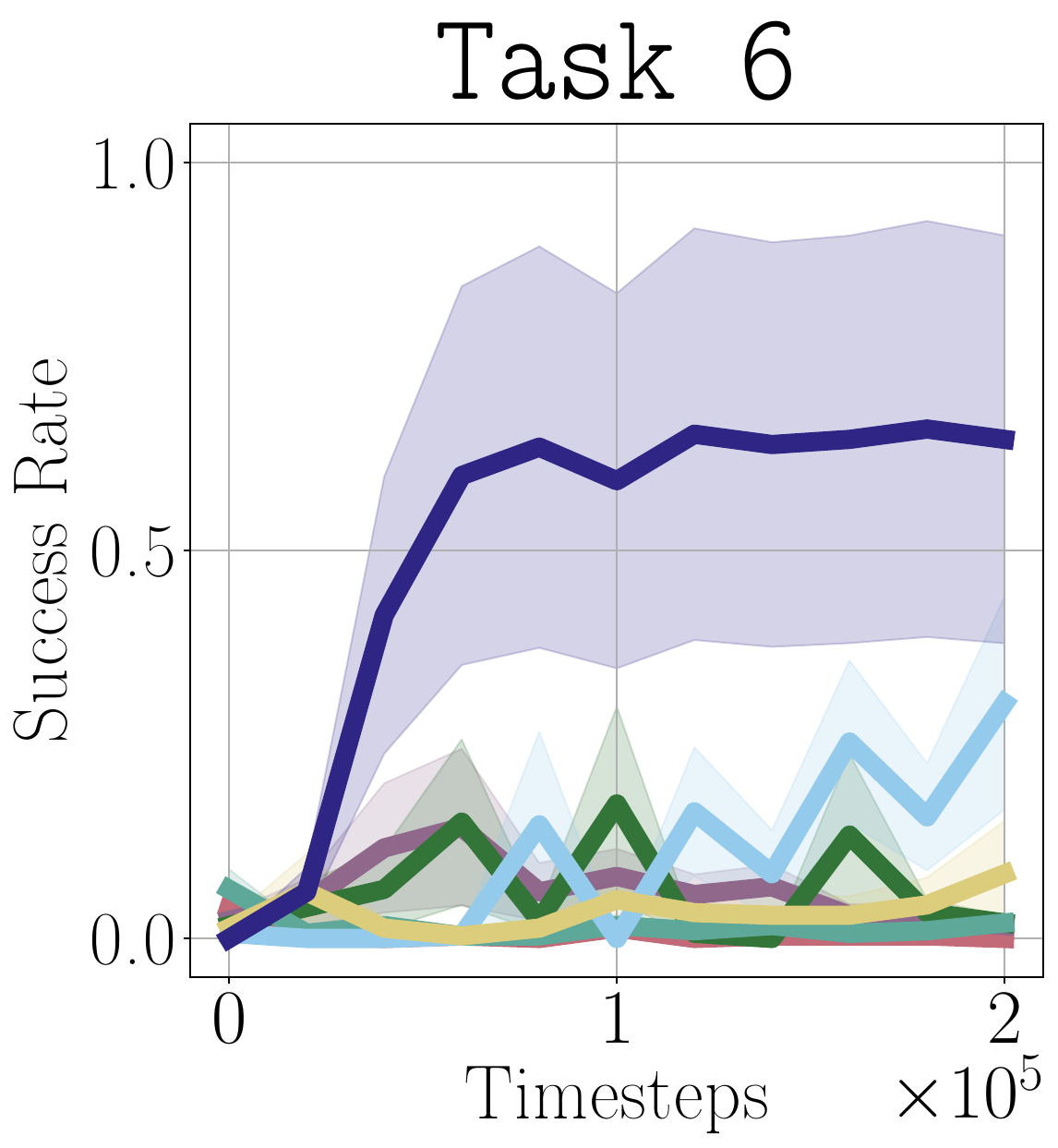}
  \end{subfigure}\hspace{0.0025\textwidth}
  \begin{subfigure}[t]{0.111\textwidth}
    \centering
    \includegraphics[width=\linewidth,
      height=2.75cm,
      keepaspectratio]{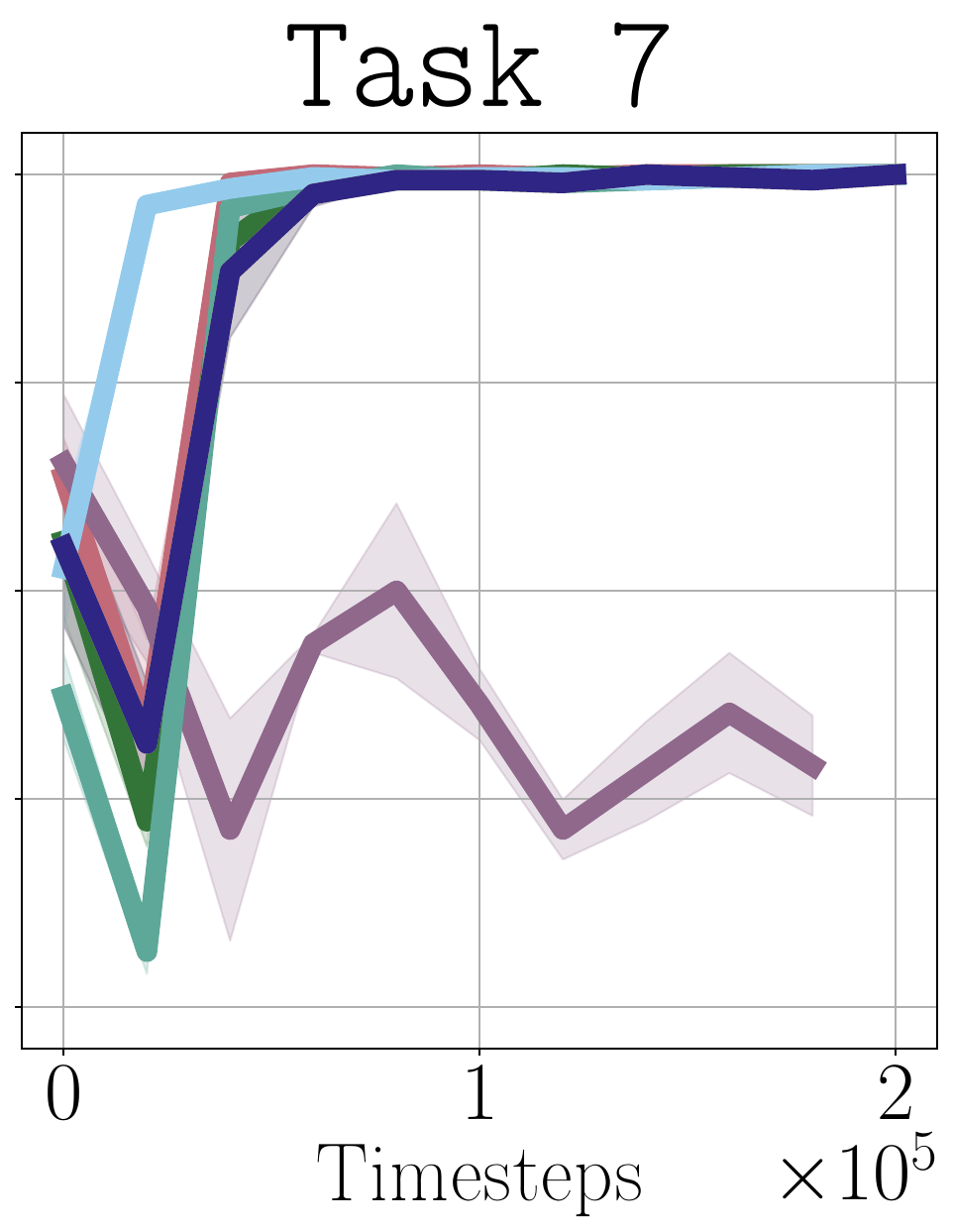}
  \end{subfigure}\hspace{0.0025\textwidth}
  \begin{subfigure}[t]{0.111\textwidth}
    \centering
    \includegraphics[width=\linewidth,
      height=2.75cm,
      keepaspectratio]{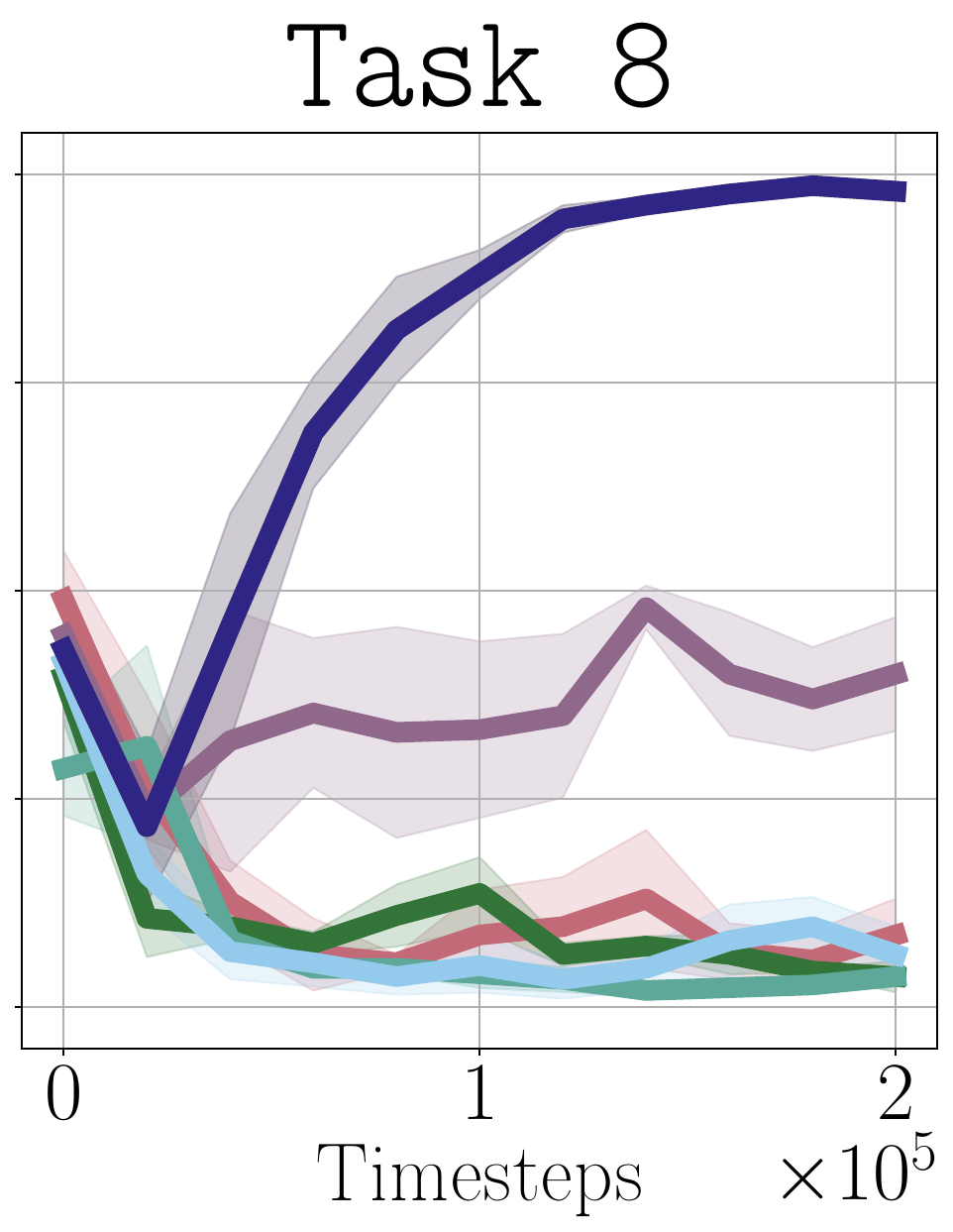}
  \end{subfigure}\hspace{0.0025\textwidth}
  \begin{subfigure}[t]{0.111\textwidth}
    \centering
    \includegraphics[width=\linewidth,
      height=2.75cm,
      keepaspectratio]{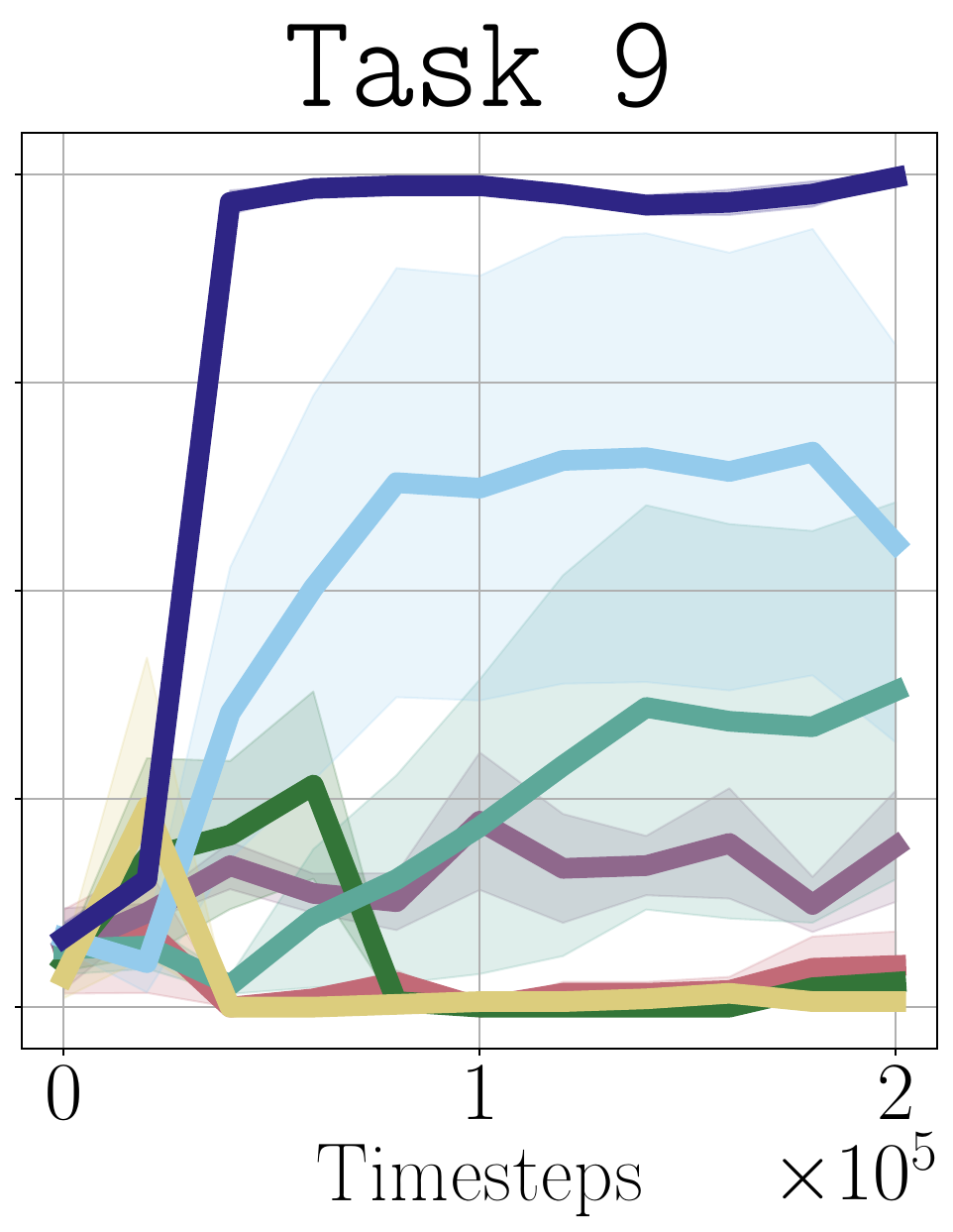}
  \end{subfigure}\hspace{0.0025\textwidth}
  \begin{subfigure}[t]{0.111\textwidth}
    \centering
    \includegraphics[width=\linewidth,
      height=2.75cm,
      keepaspectratio]{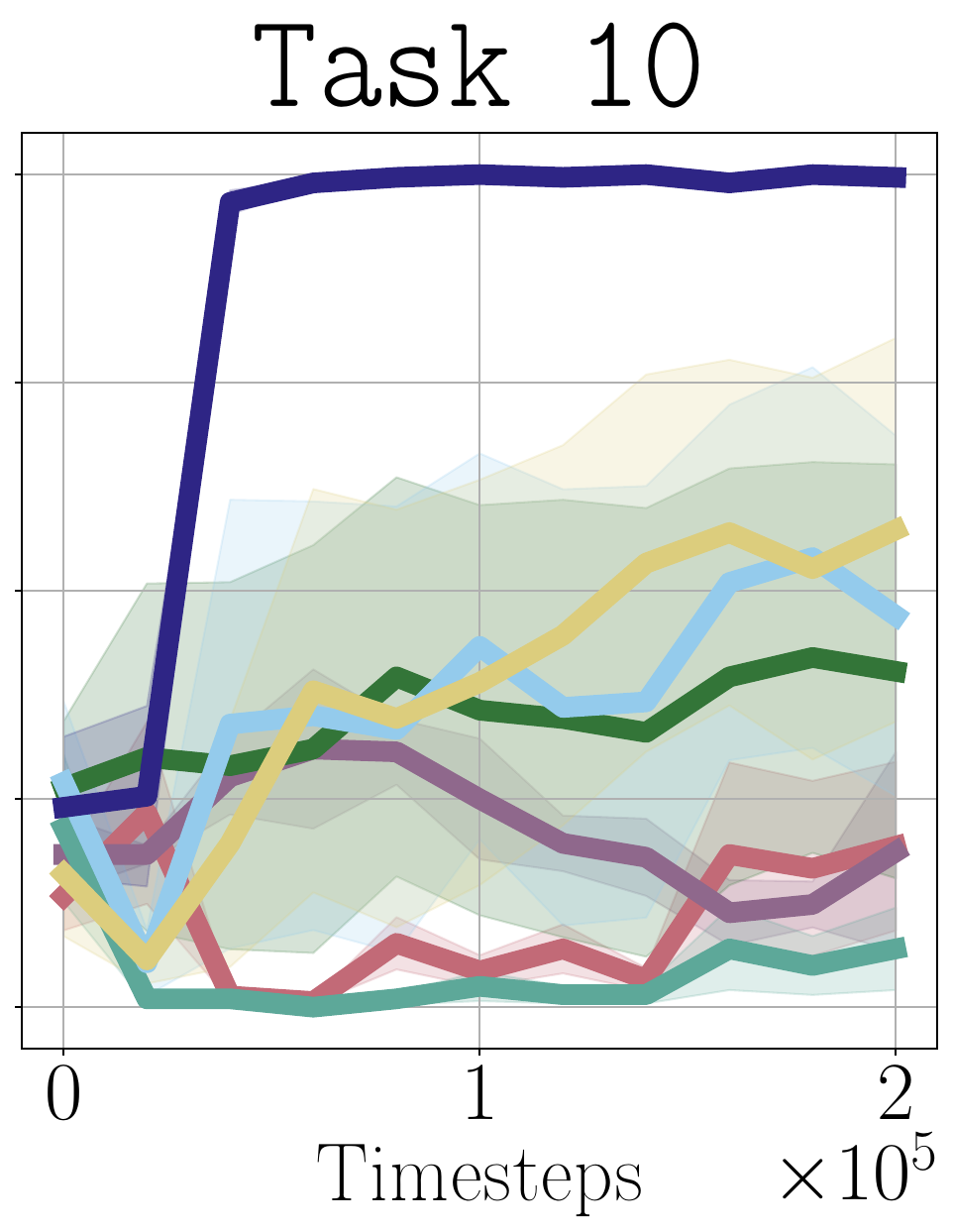}
  \end{subfigure}\hspace{0.0025\textwidth}
  \begin{subfigure}[t]{0.111\textwidth}
    \centering
    \includegraphics[width=\linewidth,
      height=2.75cm,
      keepaspectratio]{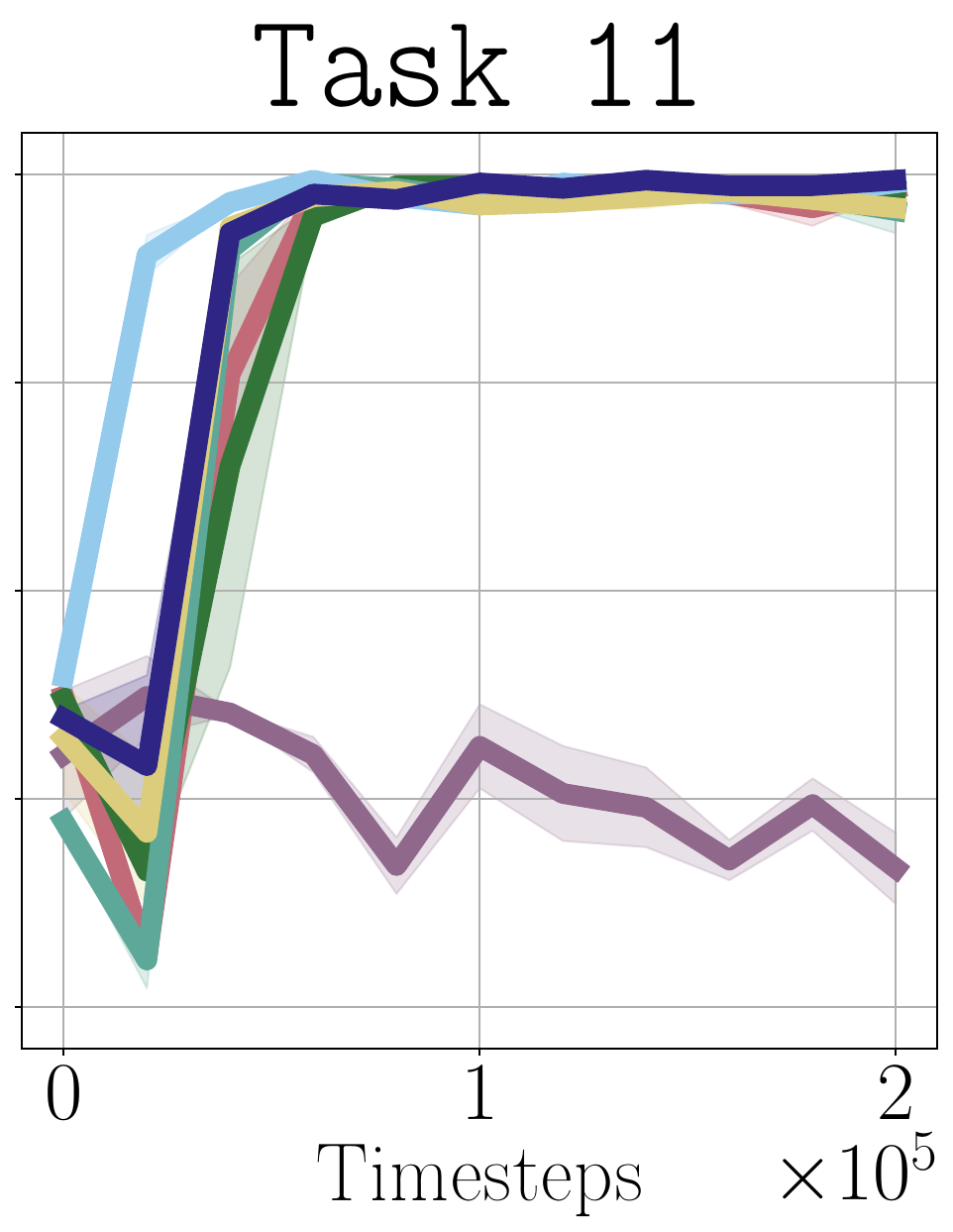}
  \end{subfigure}\hspace{0.0025\textwidth}
  \begin{subfigure}[t]{0.111\textwidth}
    \centering
    \includegraphics[width=\linewidth,
      height=2.75cm,
      keepaspectratio]{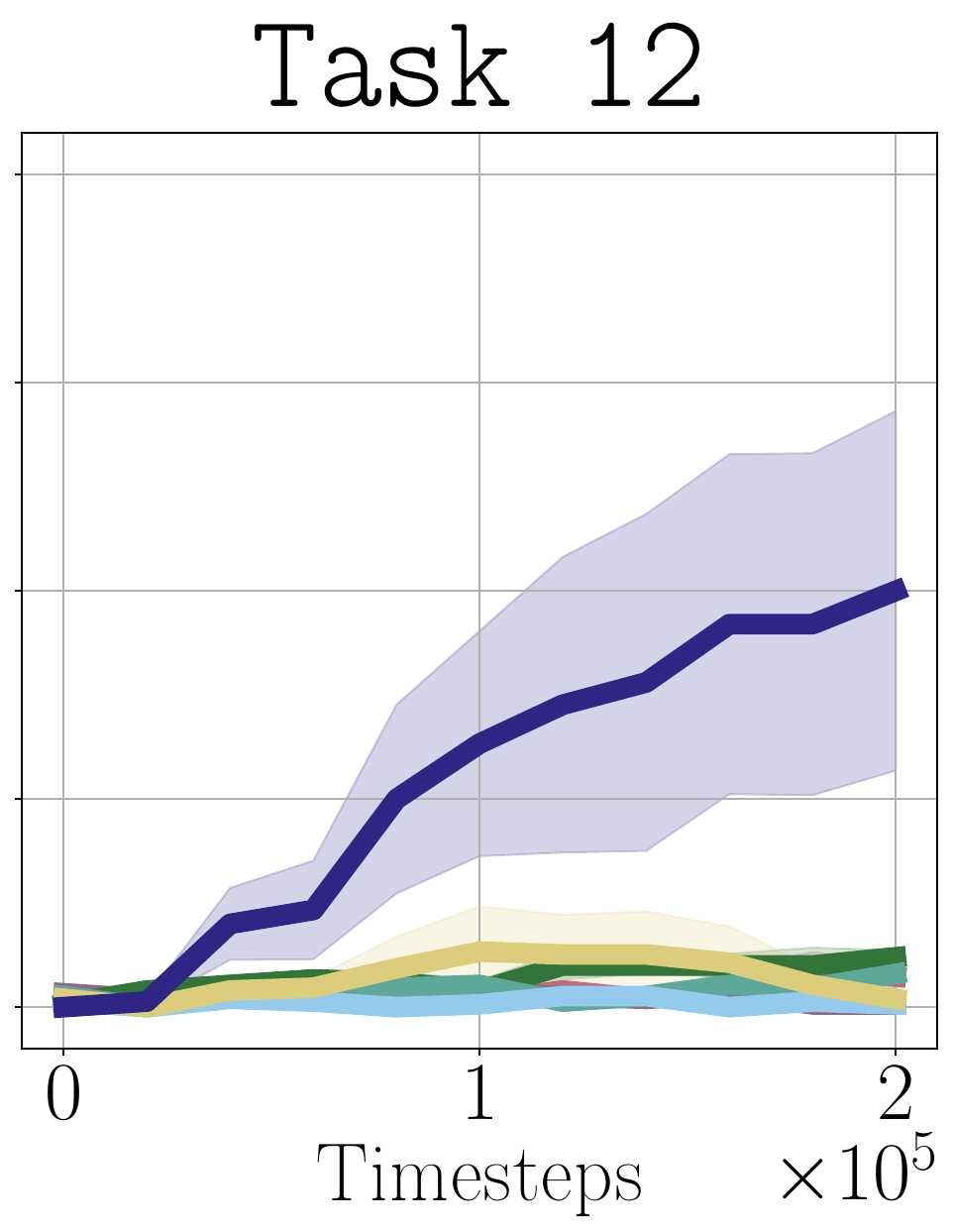}
  \end{subfigure}\hspace{0.0025\textwidth}
  \begin{subfigure}[t]{0.111\textwidth}
    \centering
    \includegraphics[width=\linewidth,
      height=2.75cm,
      keepaspectratio]{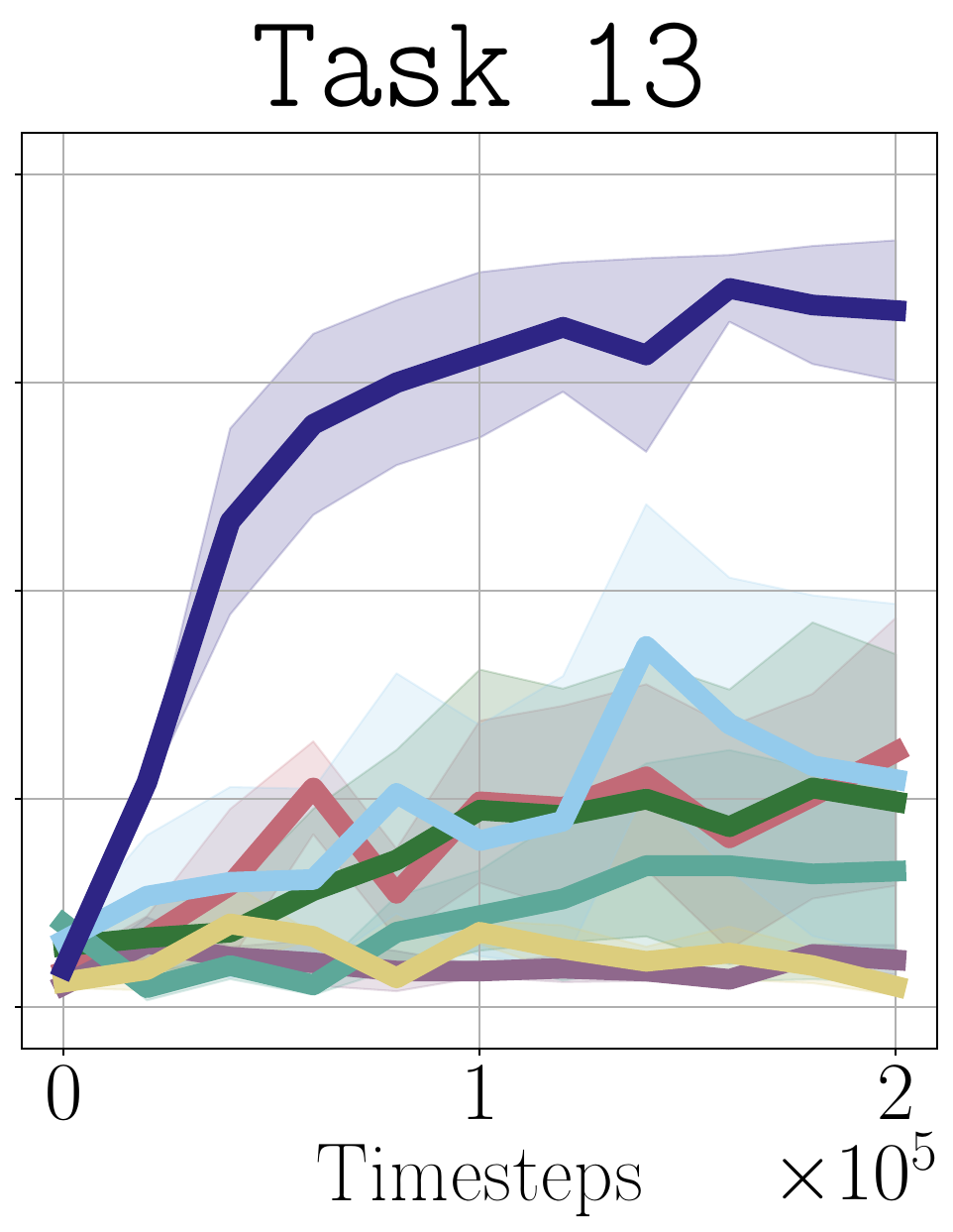}
  \end{subfigure}

  \vspace{0.15em}

  \begin{subfigure}[t]{0.13\textwidth}
    \centering
    \includegraphics[width=\linewidth,
      height=2.75cm,
      keepaspectratio]{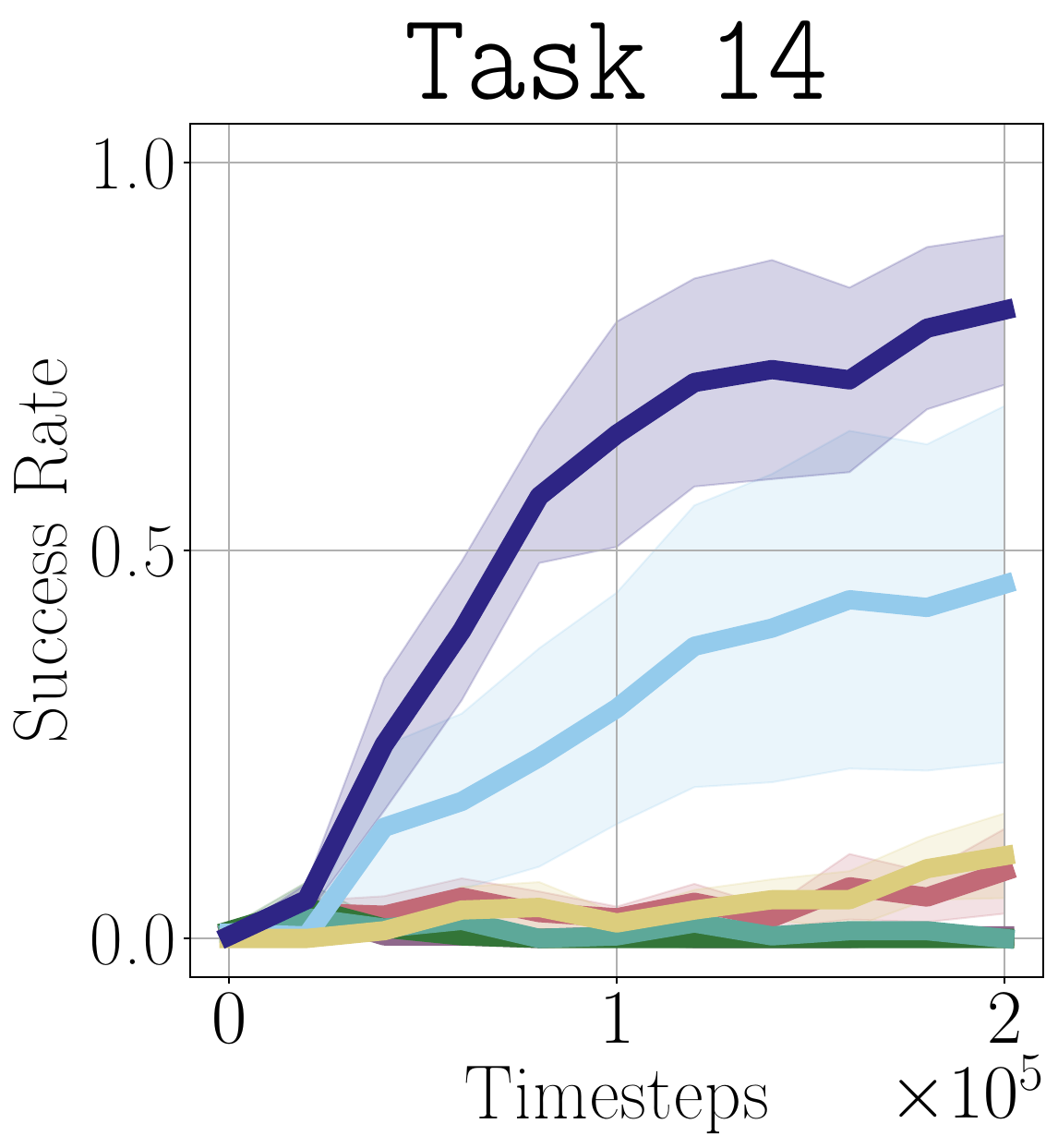}
  \end{subfigure}\hspace{0.0025\textwidth}
  \begin{subfigure}[t]{0.111\textwidth}
    \centering
    \includegraphics[width=\linewidth,
      height=2.75cm,
      keepaspectratio]{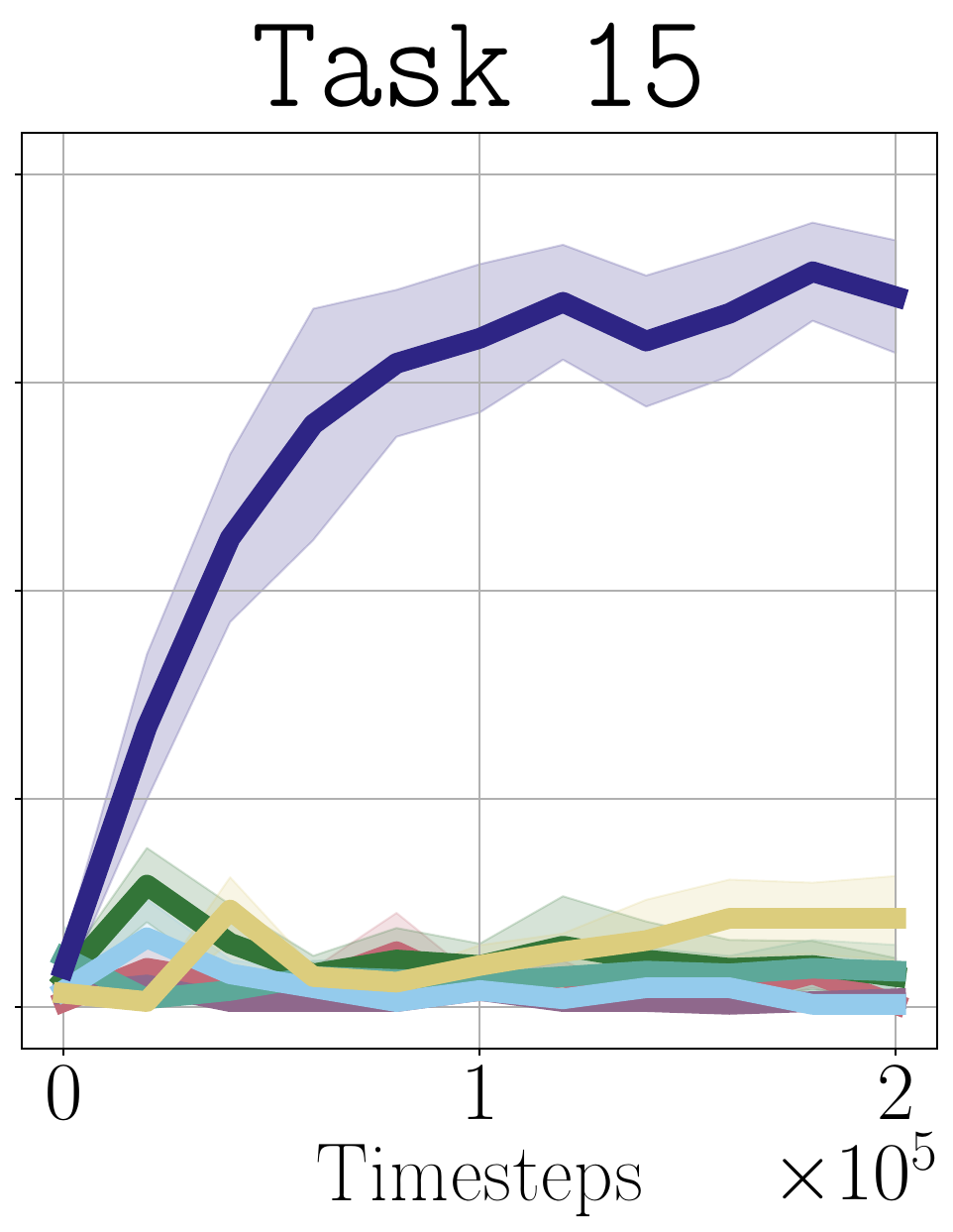}
  \end{subfigure}\hspace{0.0025\textwidth}
  \begin{subfigure}[t]{0.111\textwidth}
    \centering
    \includegraphics[width=\linewidth,
      height=2.75cm,
      keepaspectratio]{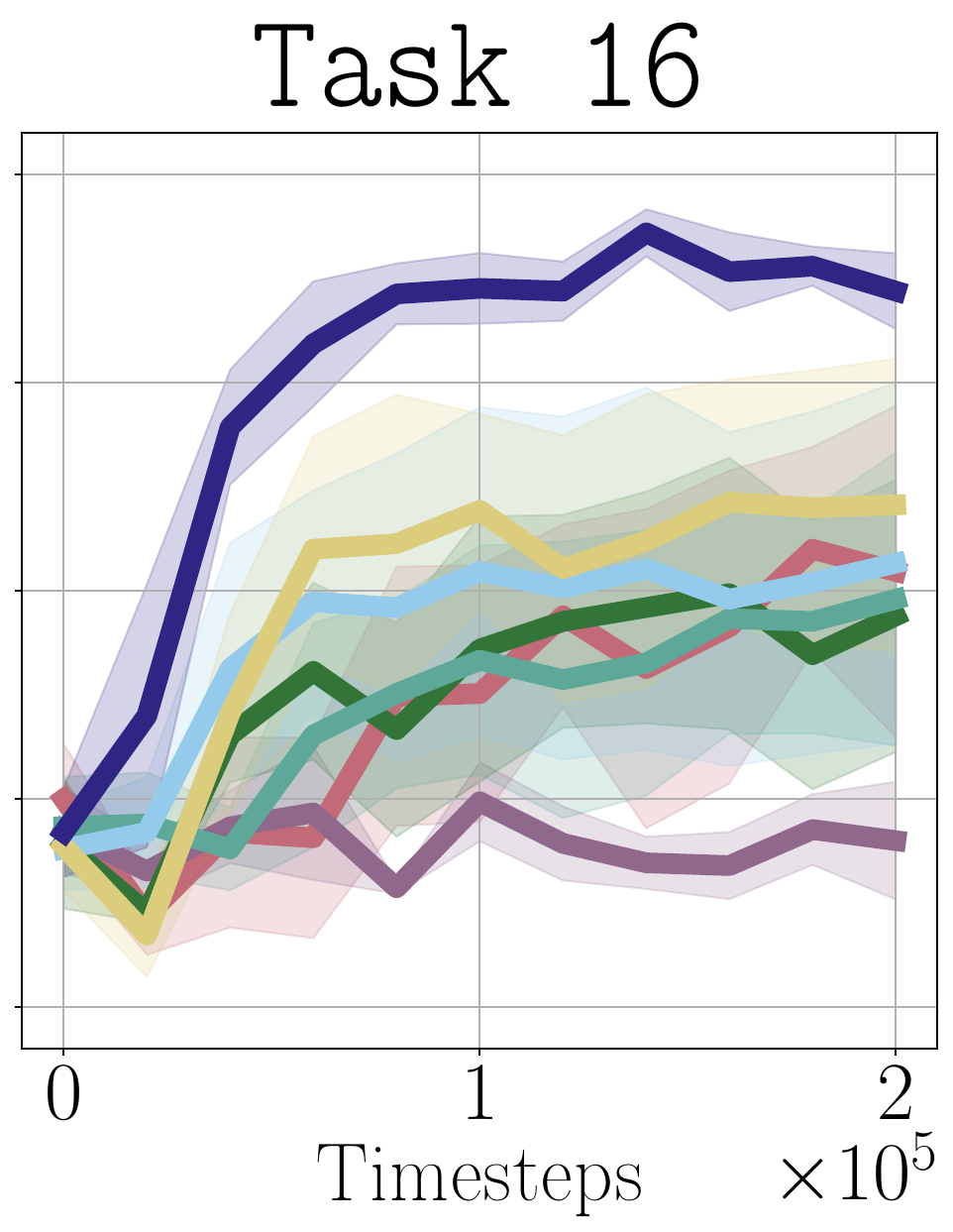}
  \end{subfigure}\hspace{0.0025\textwidth}
  \begin{subfigure}[t]{0.111\textwidth}
    \centering
    \includegraphics[width=\linewidth,
      height=2.75cm,
      keepaspectratio]{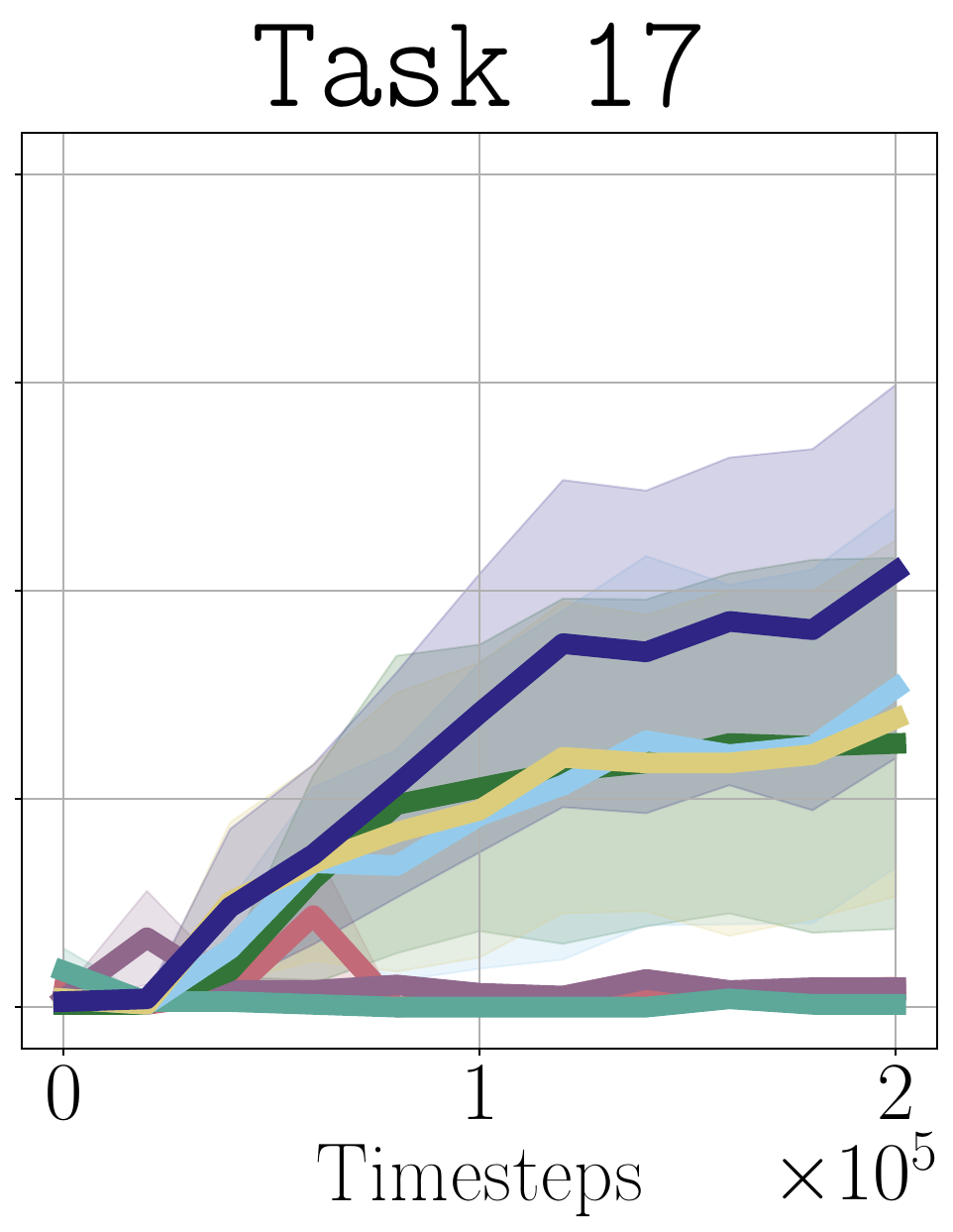}
  \end{subfigure}\hspace{0.0025\textwidth}
  \begin{subfigure}[t]{0.111\textwidth}
    \centering
    \includegraphics[width=\linewidth,
      height=2.75cm,
      keepaspectratio]{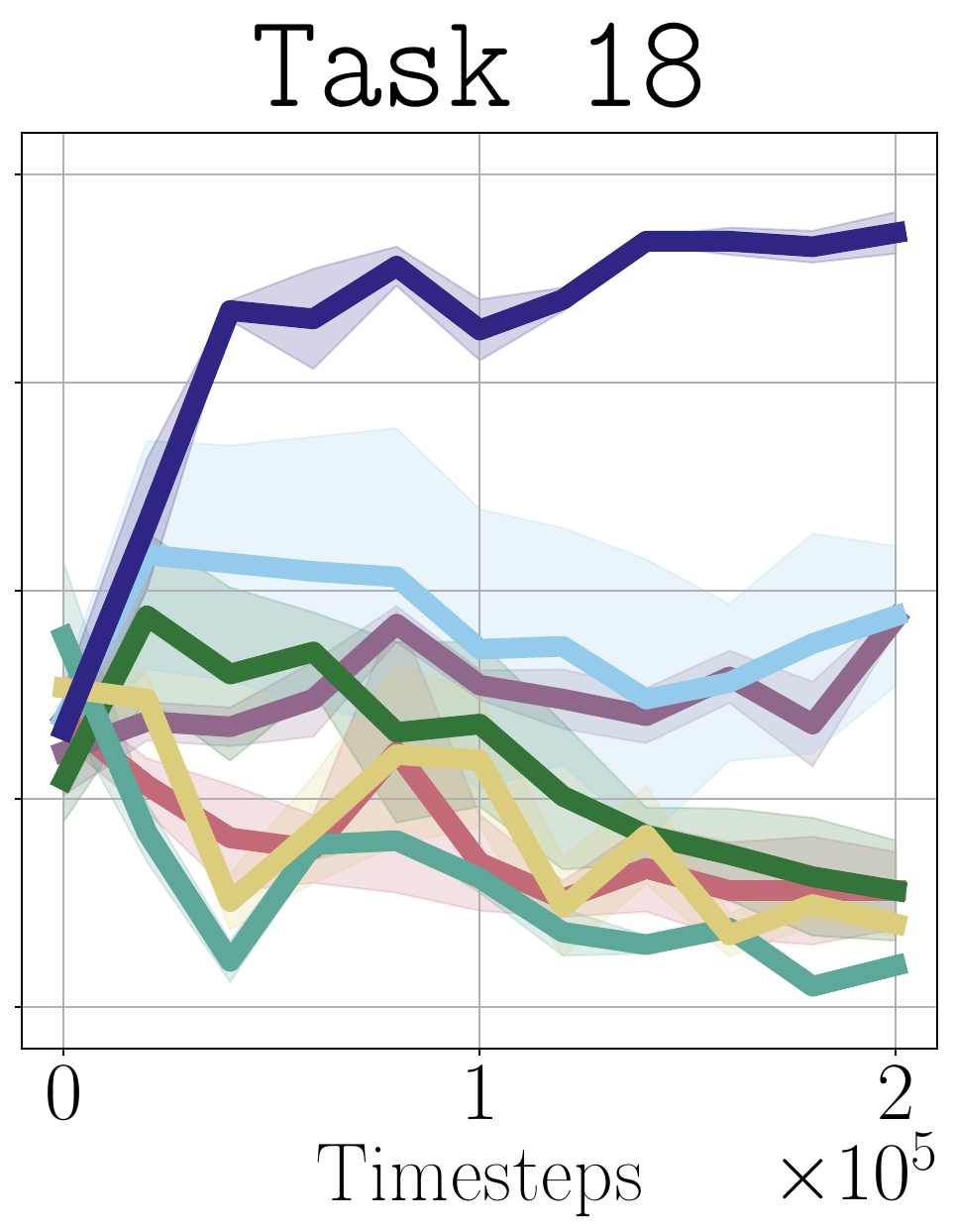}
  \end{subfigure}\hspace{0.0025\textwidth}
  \begin{subfigure}[t]{0.111\textwidth}
    \centering
    \includegraphics[width=\linewidth,
      height=2.75cm,
      keepaspectratio]{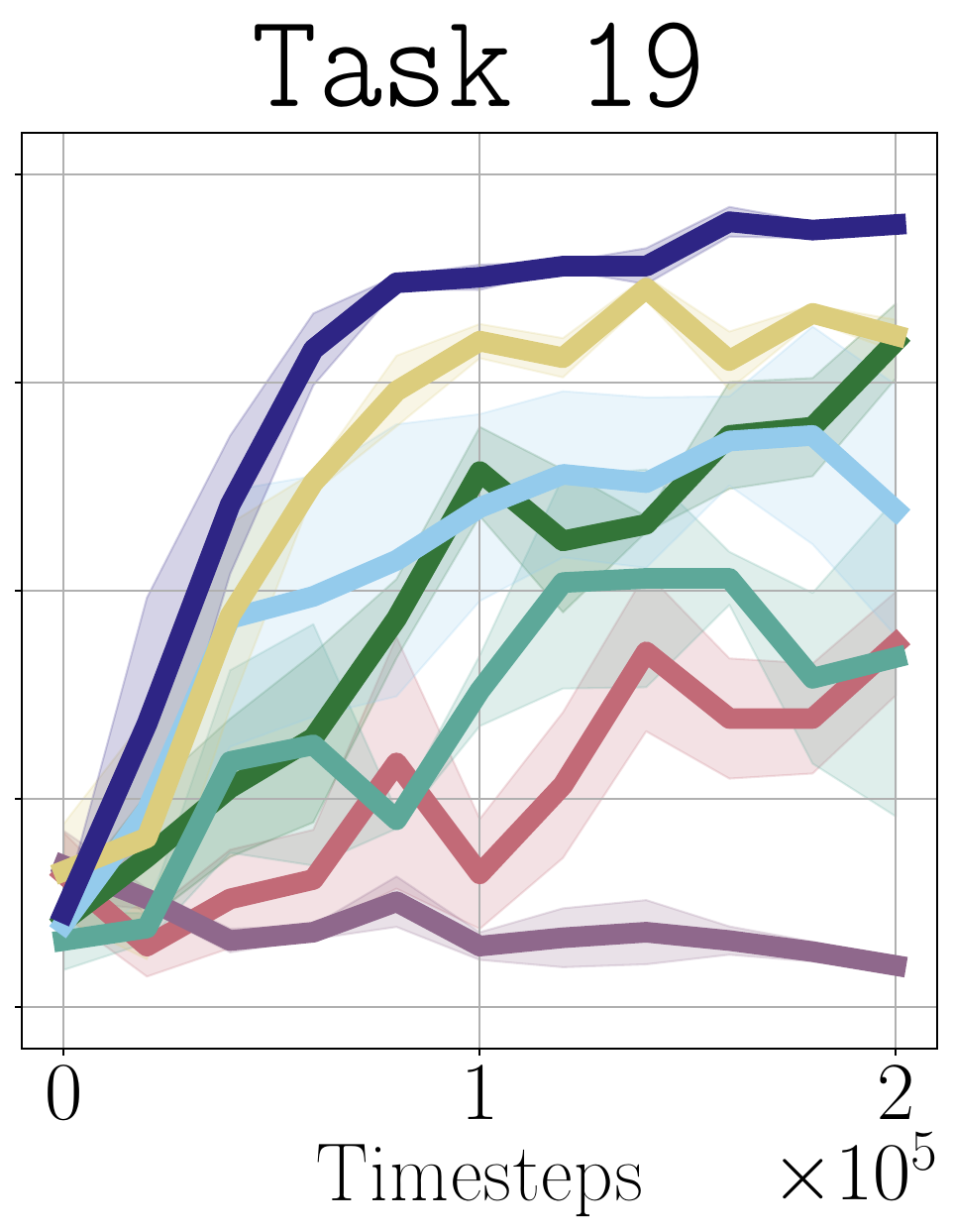}
  \end{subfigure}\hspace{0.0025\textwidth}
  \begin{subfigure}[t]{0.111\textwidth}
    \centering
    \includegraphics[width=\linewidth,
      height=2.75cm,
      keepaspectratio]{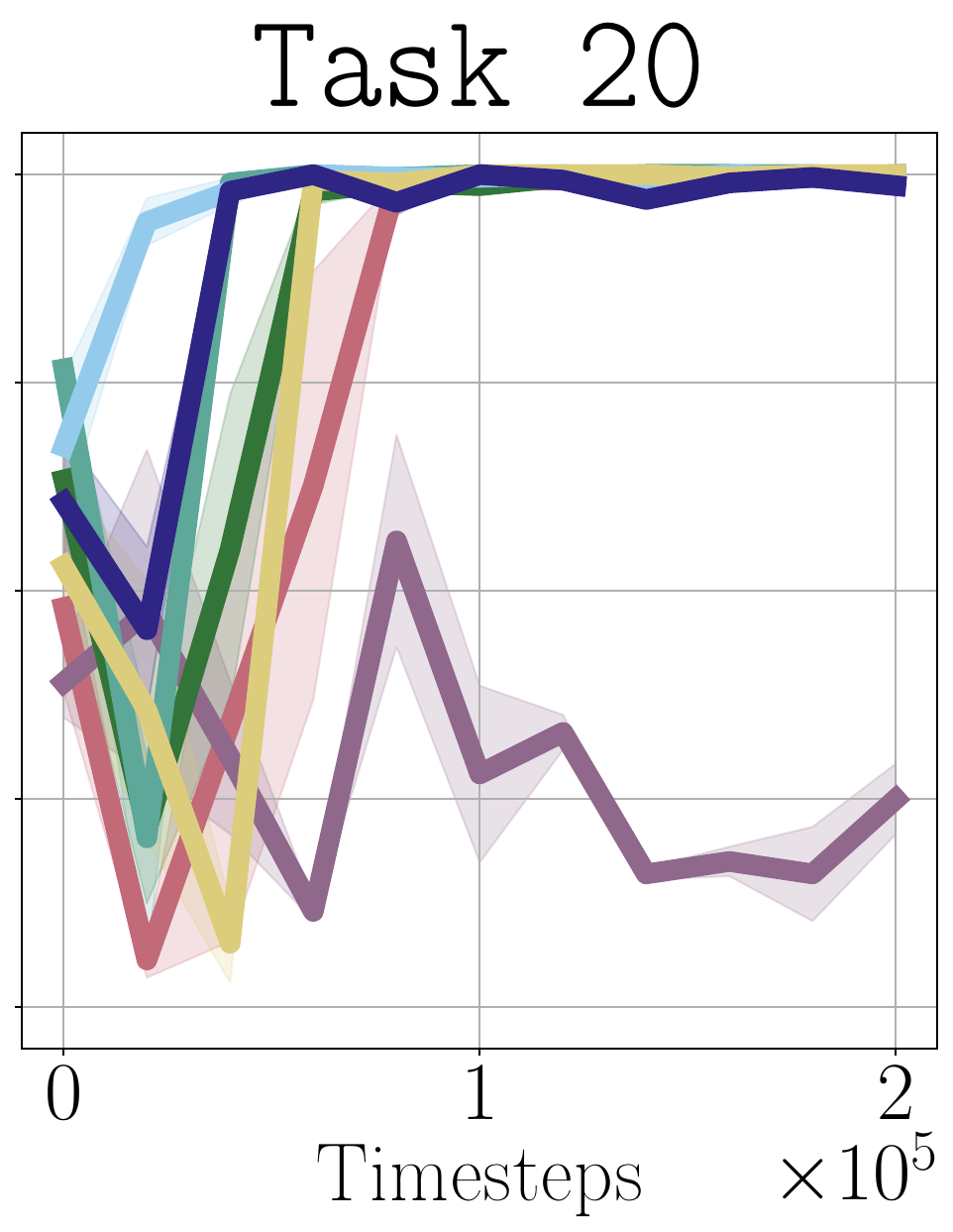}
  \end{subfigure}\hspace{0.0025\textwidth}
  \begin{subfigure}[t]{0.111\textwidth}
    \centering
    \includegraphics[width=\linewidth,
      height=2.75cm,
      keepaspectratio]{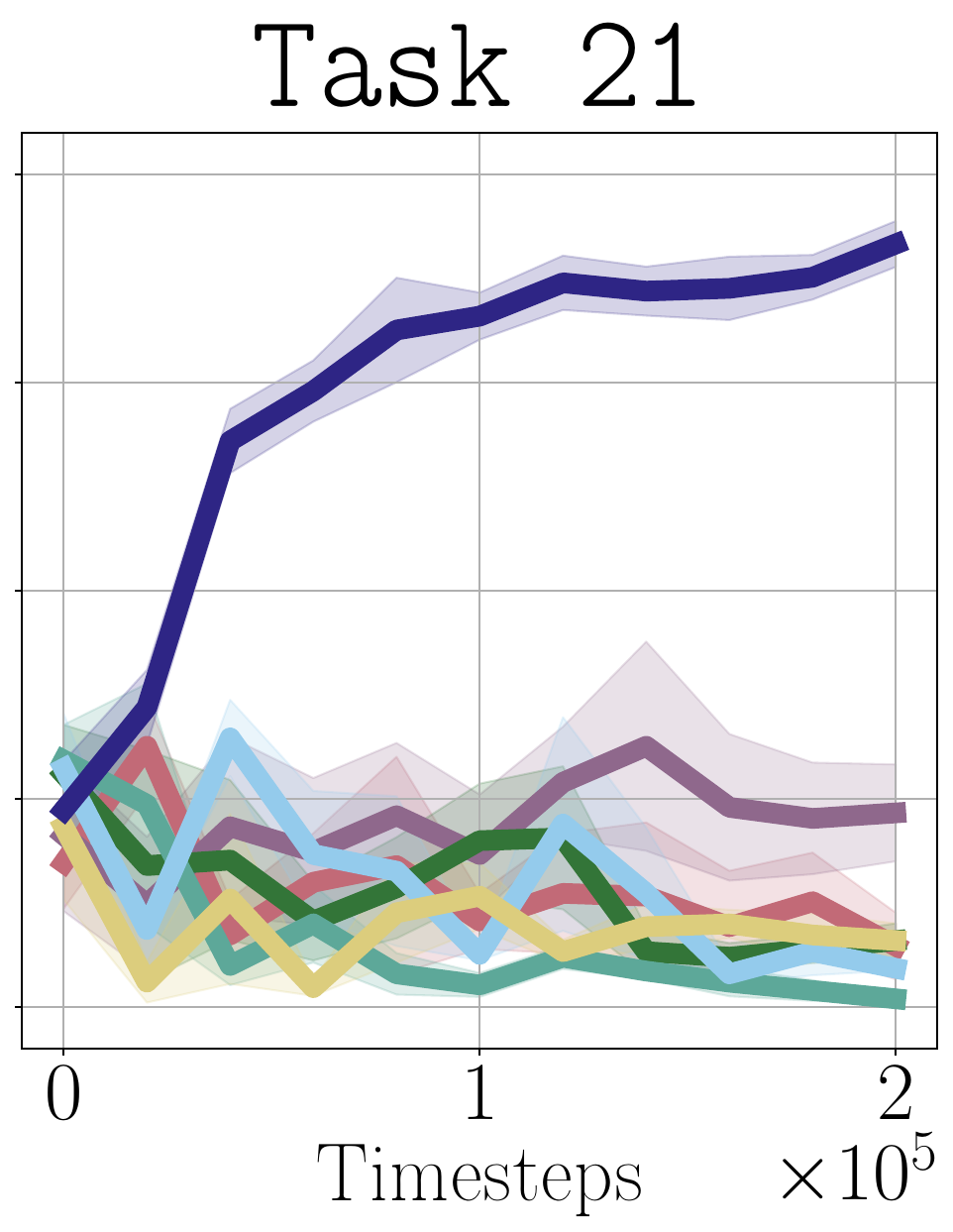}
  \end{subfigure}

  \caption{Per-task results of RL finetuning with \dsrl and \prvm process reward on \texttt{LIBERO Kitchen Scene 1--3}.}
  \label{fig:full_dsrl_libero}
\end{figure}

\begin{figure}[H]
  \centering
  \begin{subfigure}[t]{0.98\textwidth}
    \centering
    \includegraphics[width=0.98\linewidth]{images/legend.pdf}
  \end{subfigure}
  \vspace{0.15em}

  \begin{subfigure}[t]{0.13\textwidth}
    \centering
    \includegraphics[width=\linewidth,
      height=2.75cm,
      keepaspectratio]{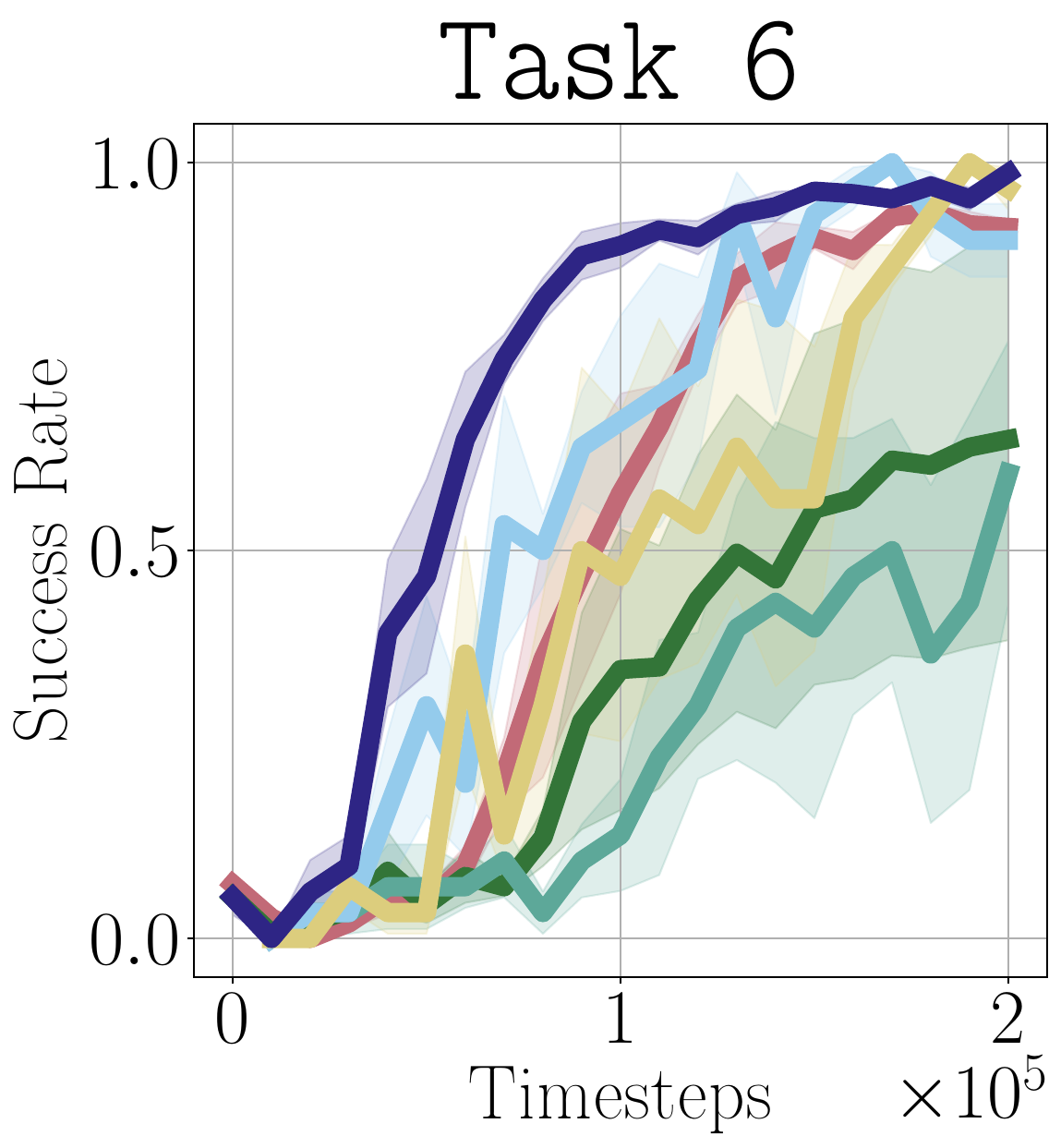}
  \end{subfigure}\hspace{0.0025\textwidth}
  \begin{subfigure}[t]{0.111\textwidth}
    \centering
    \includegraphics[width=\linewidth,
      height=2.75cm,
      keepaspectratio]{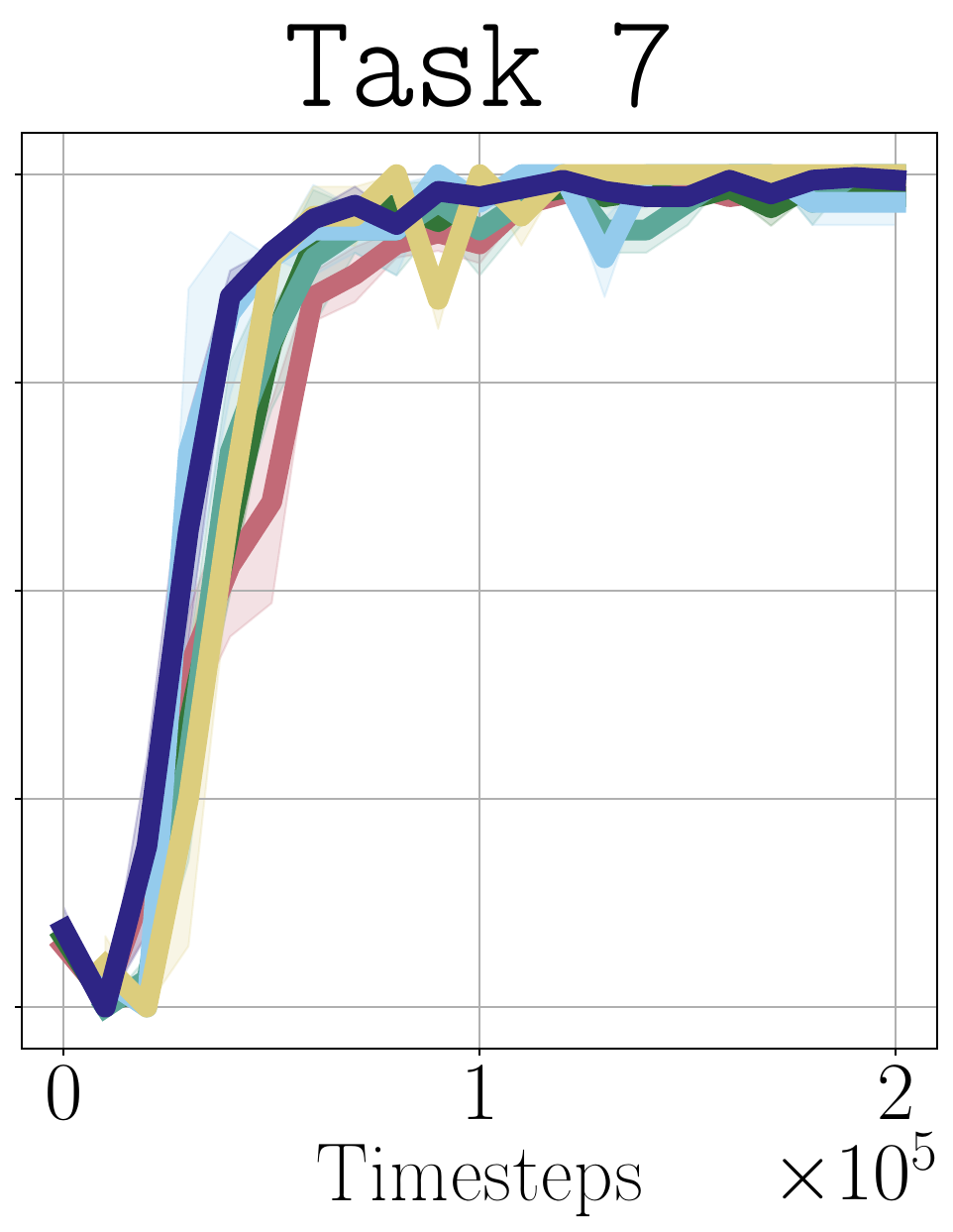}
  \end{subfigure}\hspace{0.0025\textwidth}
  \begin{subfigure}[t]{0.111\textwidth}
    \centering
    \includegraphics[width=\linewidth,
      height=2.75cm,
      keepaspectratio]{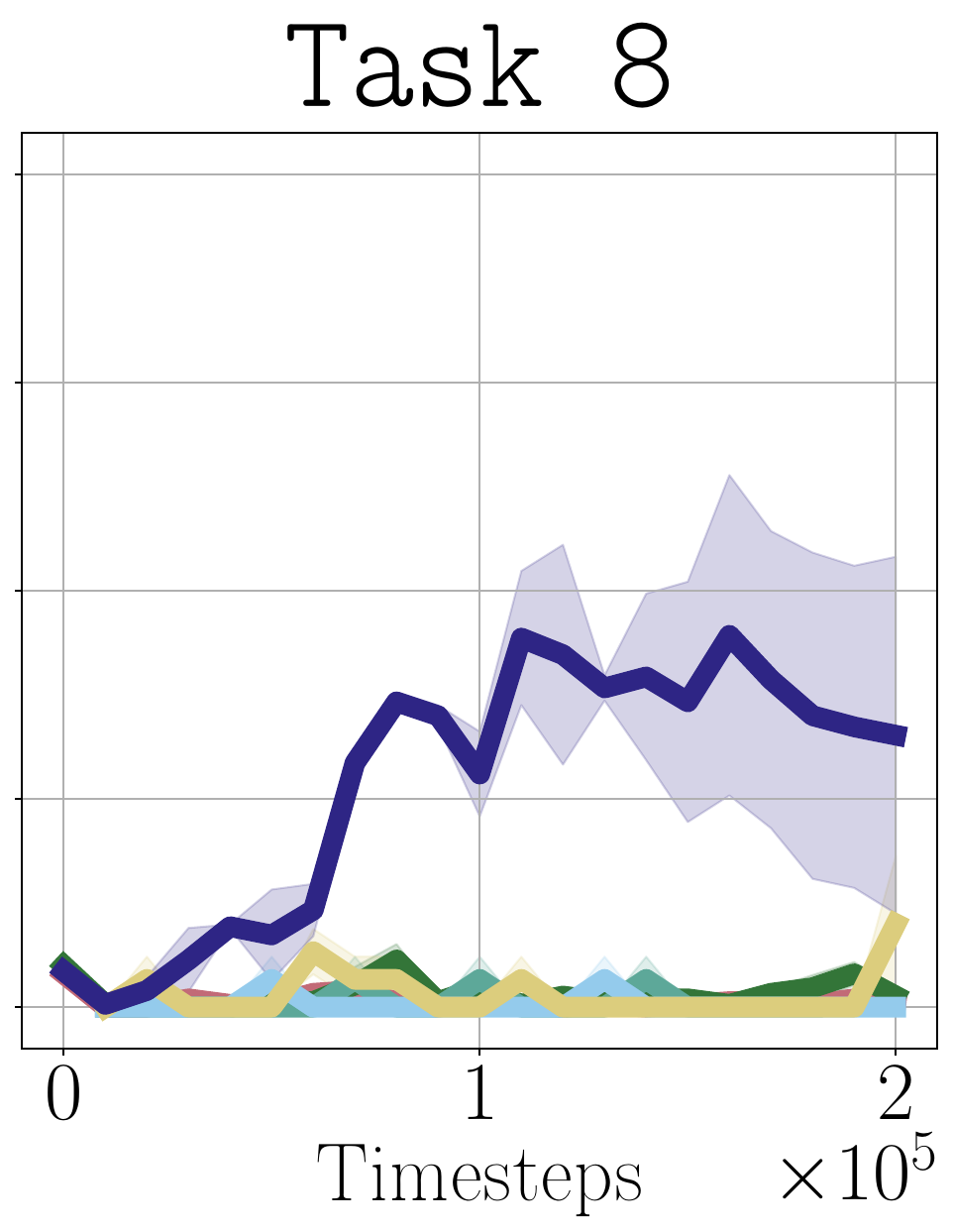}
  \end{subfigure}\hspace{0.0025\textwidth}
  \begin{subfigure}[t]{0.111\textwidth}
    \centering
    \includegraphics[width=\linewidth,
      height=2.75cm,
      keepaspectratio]{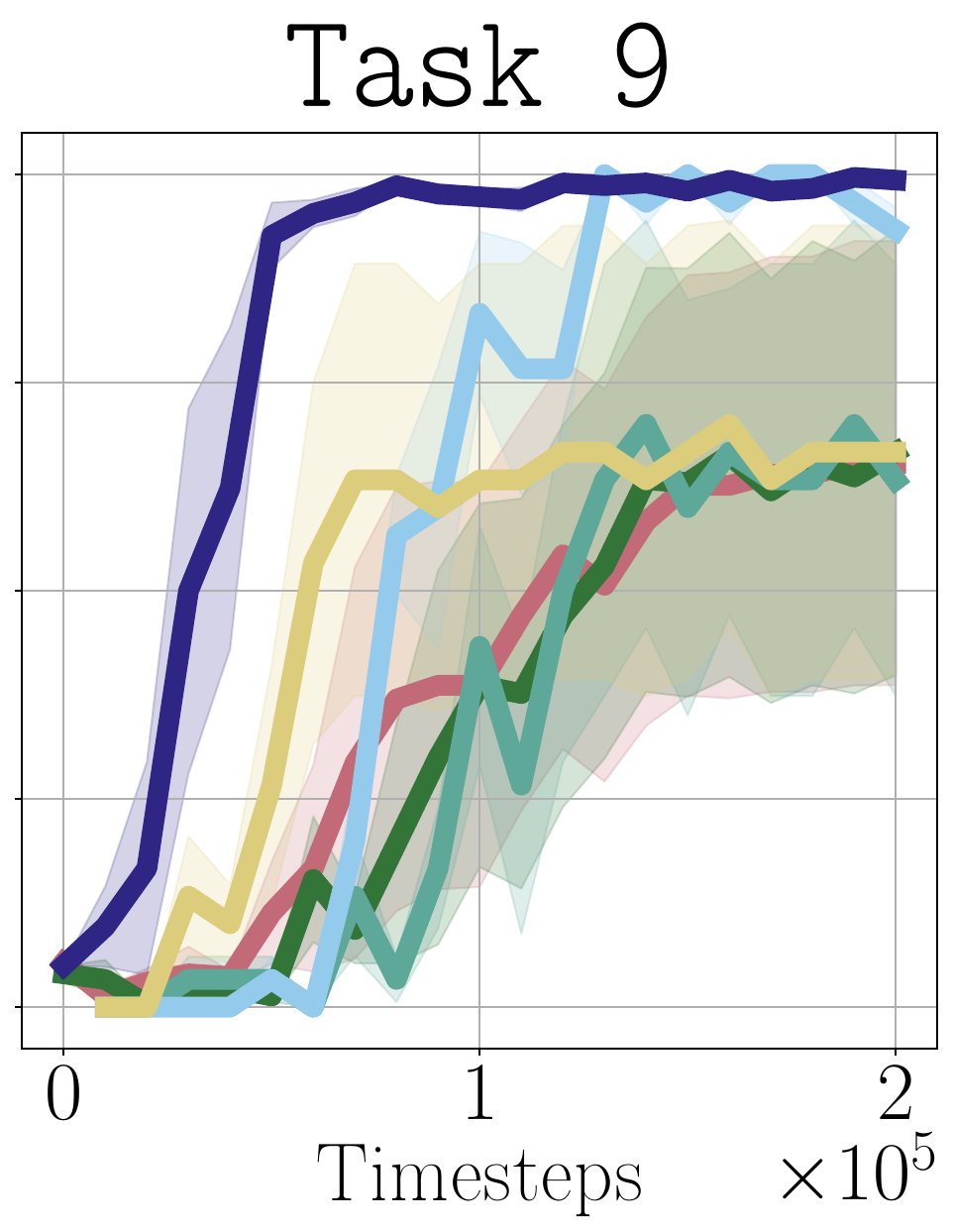}
  \end{subfigure}\hspace{0.0025\textwidth}
  \begin{subfigure}[t]{0.111\textwidth}
    \centering
    \includegraphics[width=\linewidth,
      height=2.75cm,
      keepaspectratio]{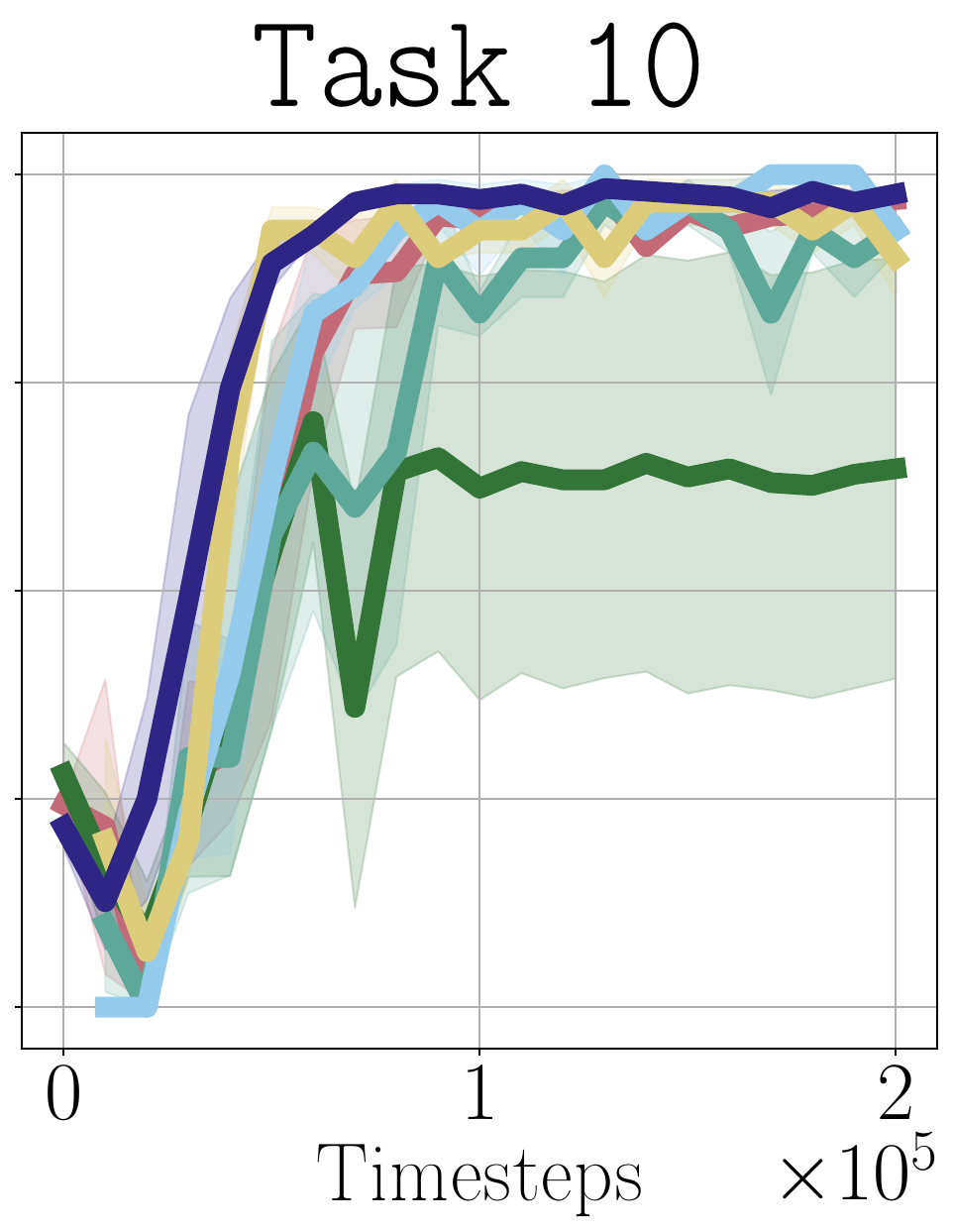}
  \end{subfigure}\hspace{0.0025\textwidth}
  \begin{subfigure}[t]{0.111\textwidth}
    \centering
    \includegraphics[width=\linewidth,
      height=2.75cm,
      keepaspectratio]{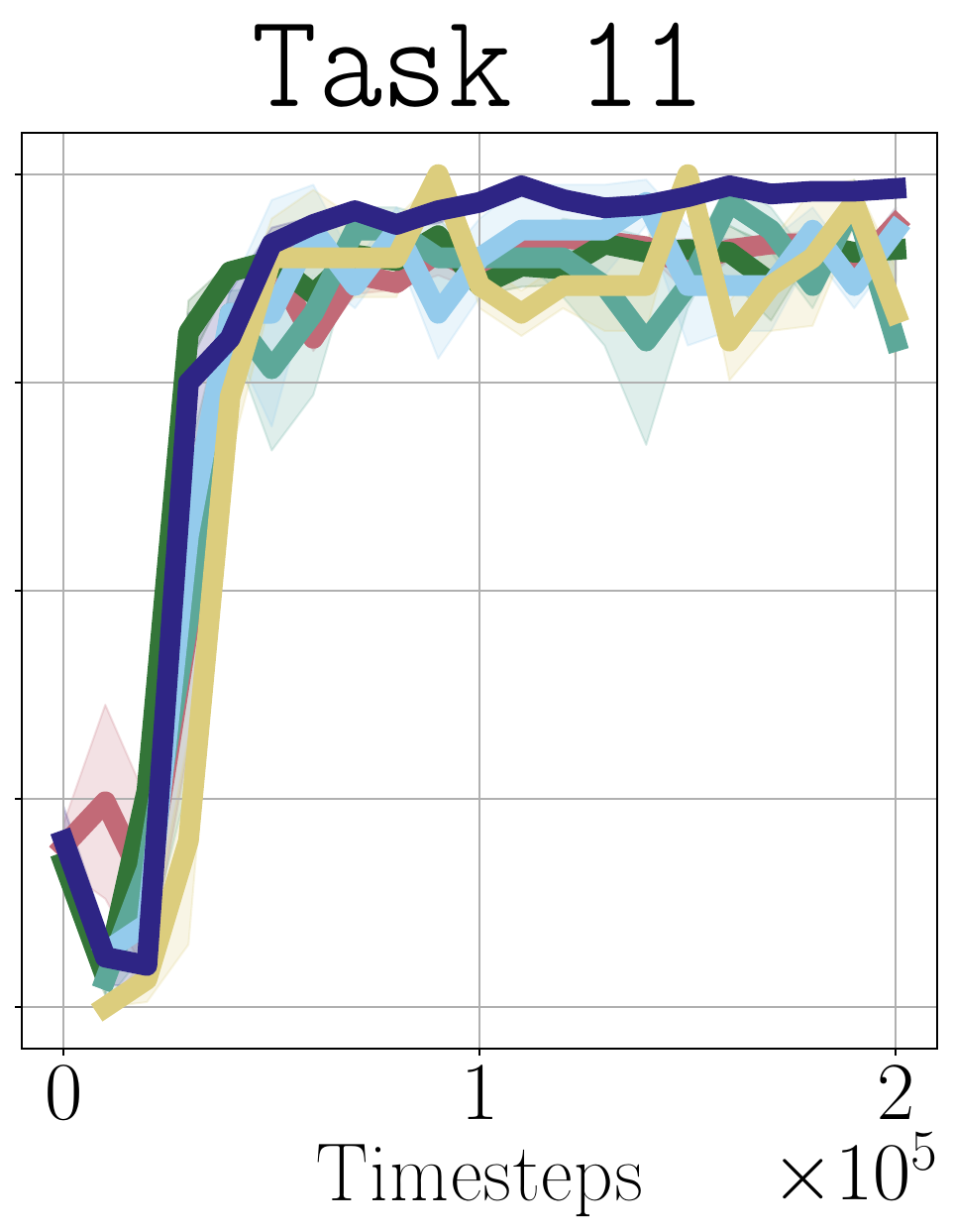}
  \end{subfigure}\hspace{0.0025\textwidth}
  \begin{subfigure}[t]{0.111\textwidth}
    \centering
    \includegraphics[width=\linewidth,
      height=2.75cm,
      keepaspectratio]{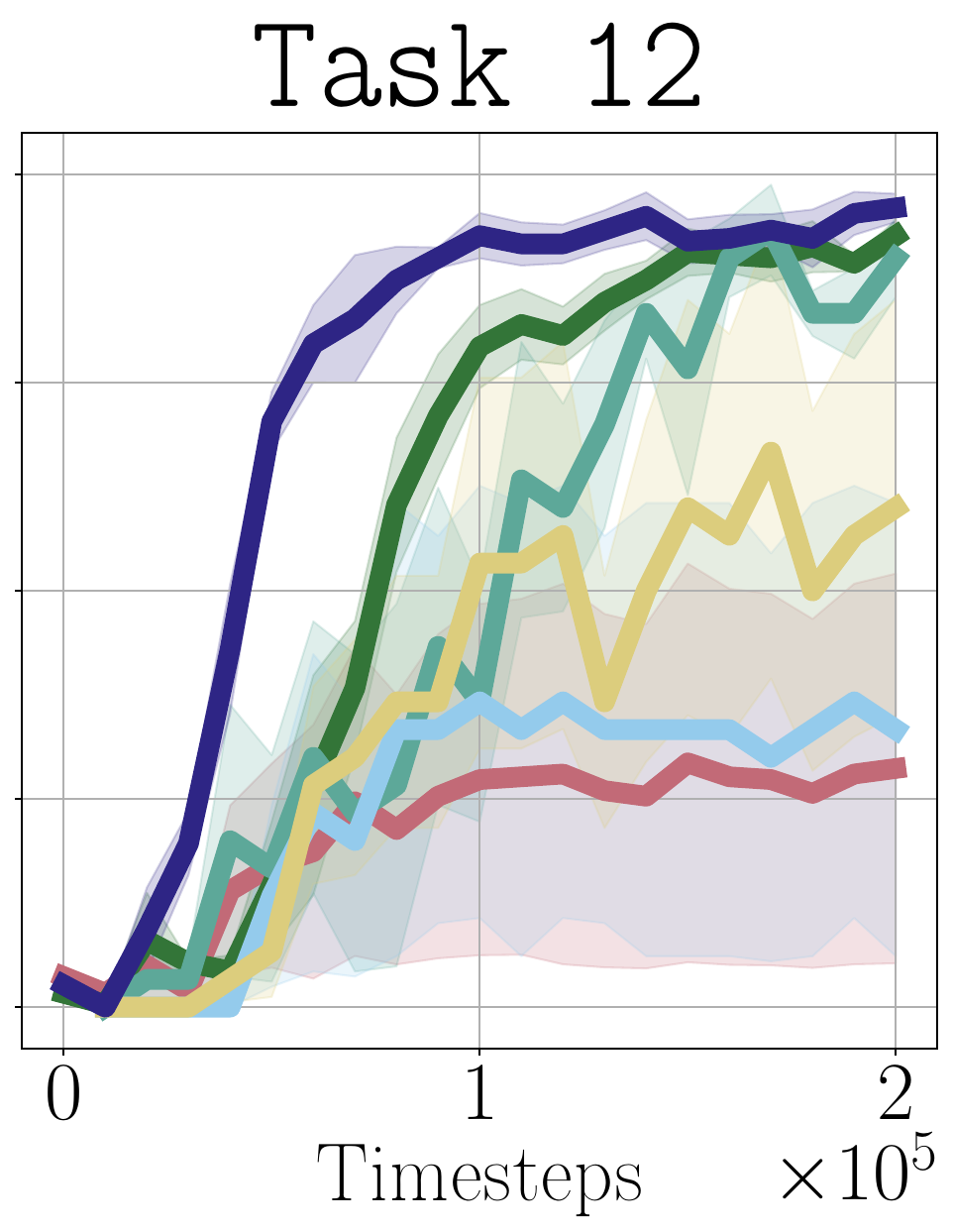}
  \end{subfigure}\hspace{0.0025\textwidth}
  \begin{subfigure}[t]{0.111\textwidth}
    \centering
    \includegraphics[width=\linewidth,
      height=2.75cm,
      keepaspectratio]{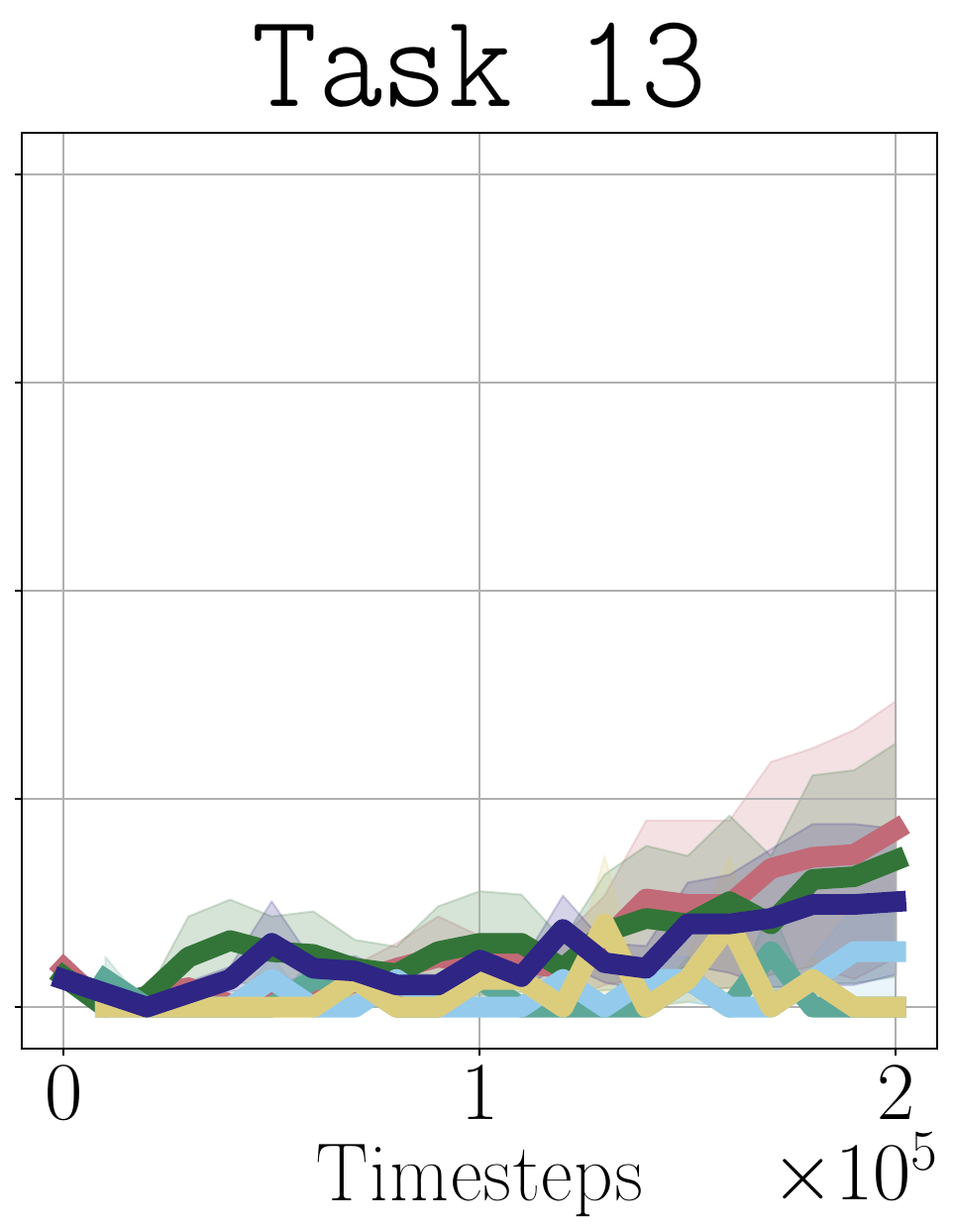}
  \end{subfigure}

  \vspace{0.15em}

  \begin{subfigure}[t]{0.13\textwidth}
    \centering
    \includegraphics[width=\linewidth,
      height=2.75cm,
      keepaspectratio]{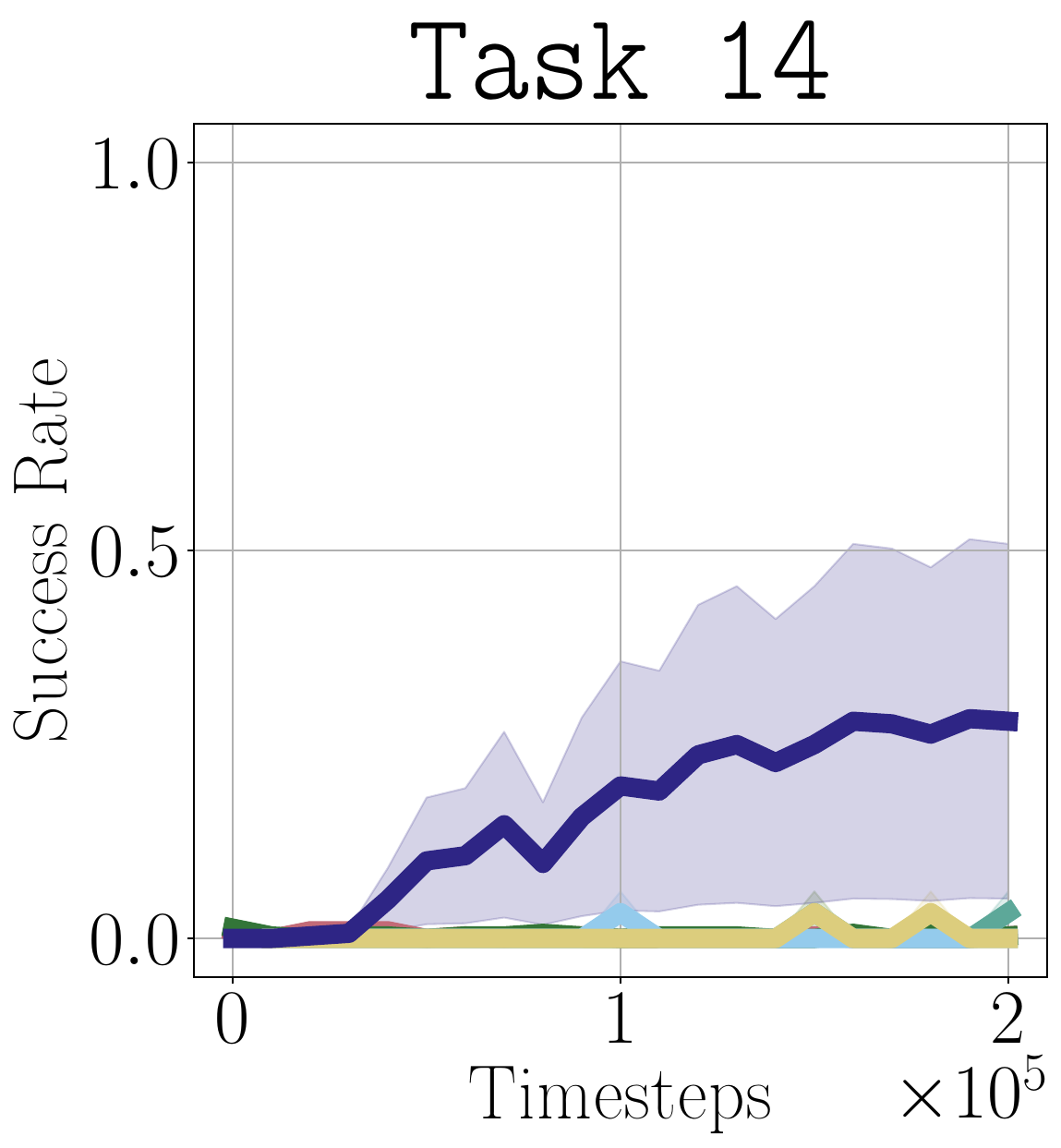}
  \end{subfigure}\hspace{0.0025\textwidth}
  \begin{subfigure}[t]{0.111\textwidth}
    \centering
    \includegraphics[width=\linewidth,
      height=2.75cm,
      keepaspectratio]{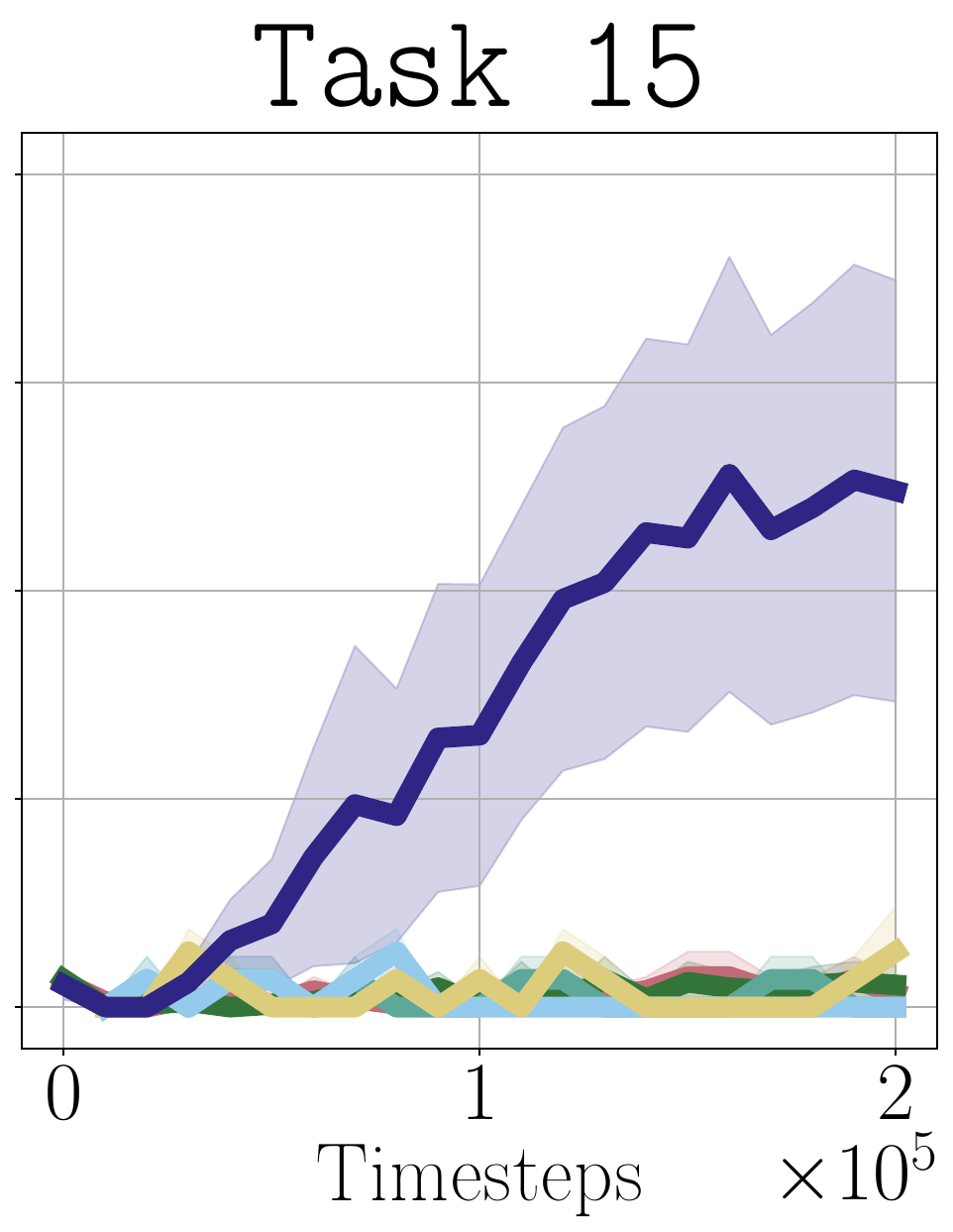}
  \end{subfigure}\hspace{0.0025\textwidth}
  \begin{subfigure}[t]{0.111\textwidth}
    \centering
    \includegraphics[width=\linewidth,
      height=2.75cm,
      keepaspectratio]{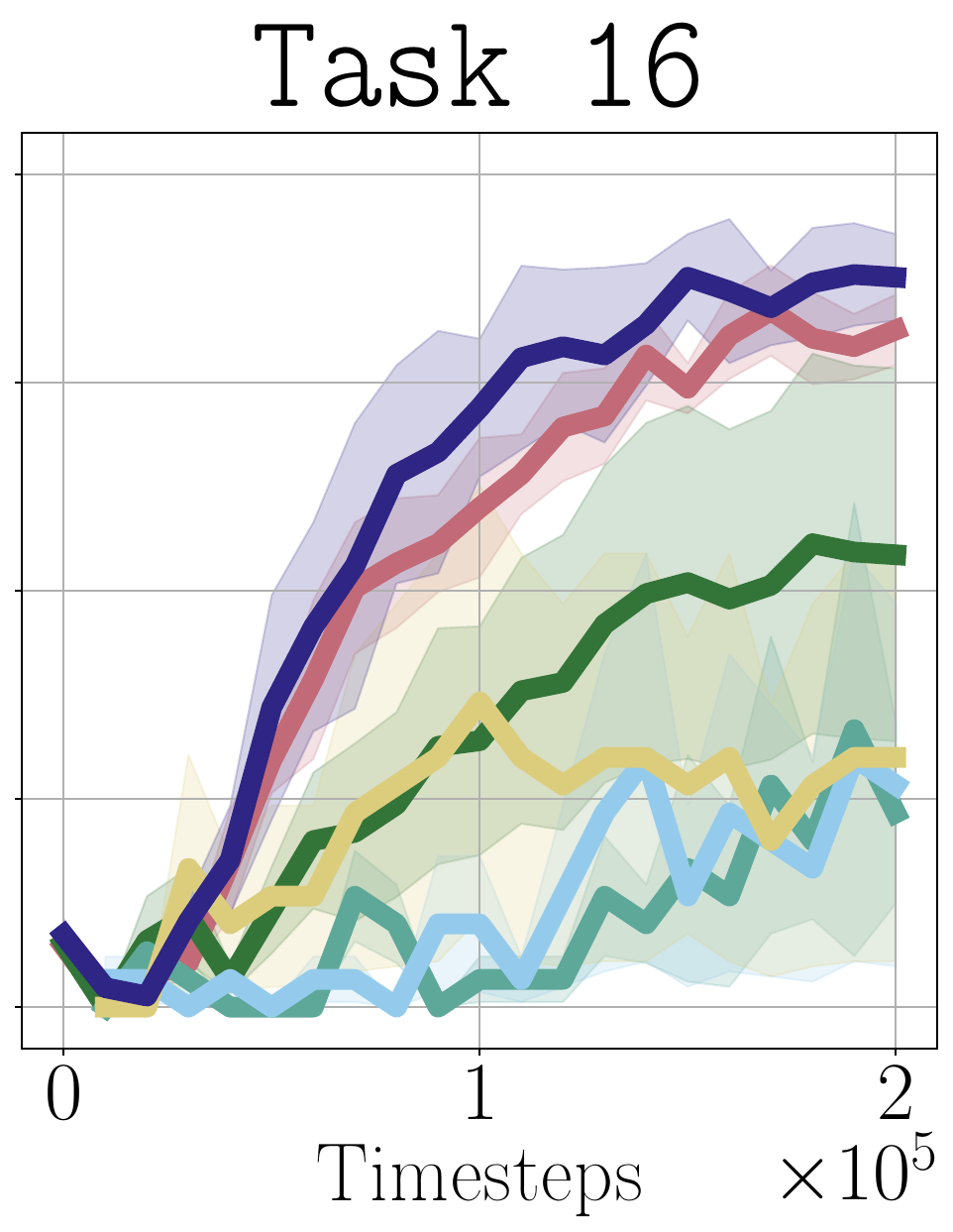}
  \end{subfigure}\hspace{0.0025\textwidth}
  \begin{subfigure}[t]{0.111\textwidth}
    \centering
    \includegraphics[width=\linewidth,
      height=2.75cm,
      keepaspectratio]{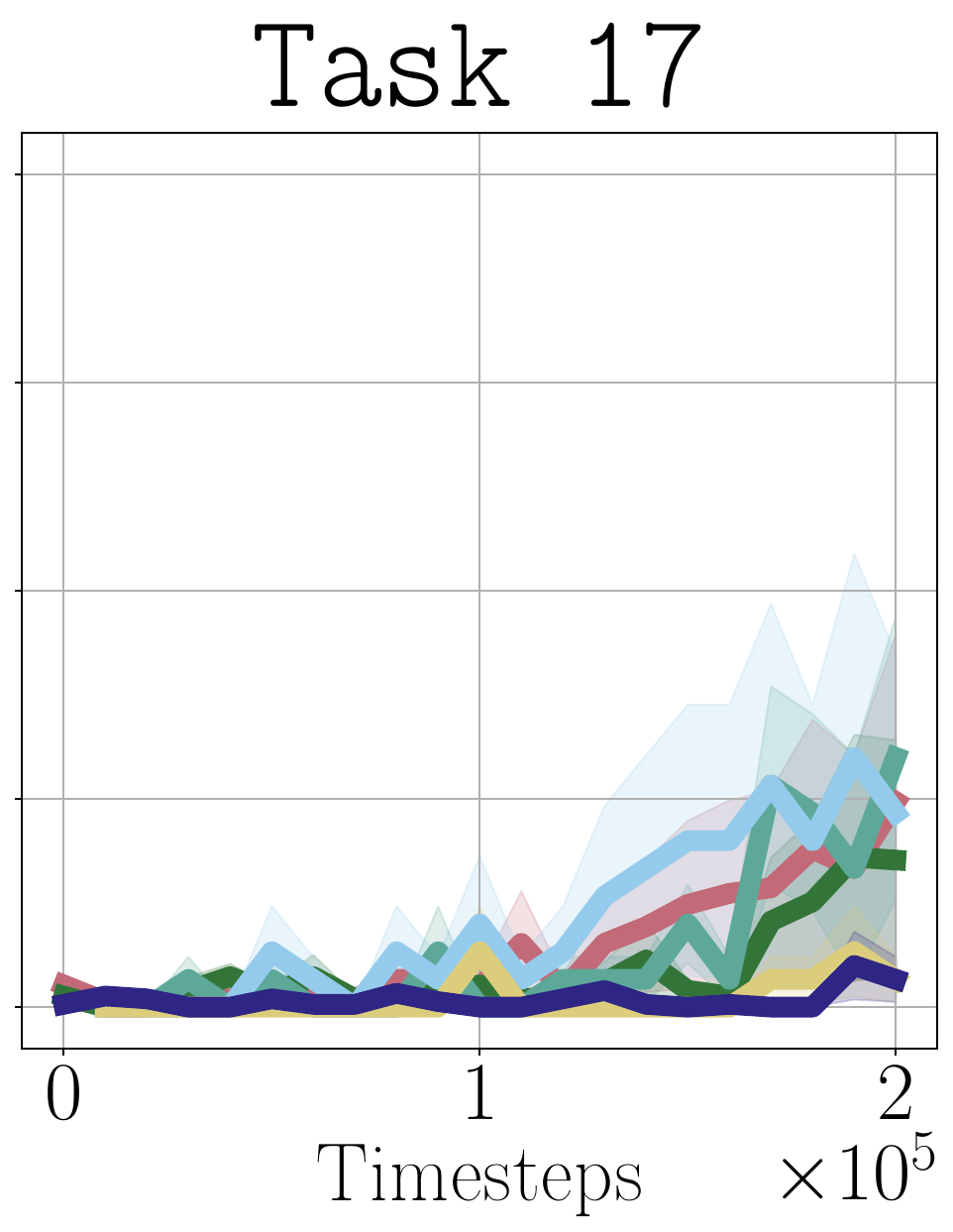}
  \end{subfigure}\hspace{0.0025\textwidth}
  \begin{subfigure}[t]{0.111\textwidth}
    \centering
    \includegraphics[width=\linewidth,
      height=2.75cm,
      keepaspectratio]{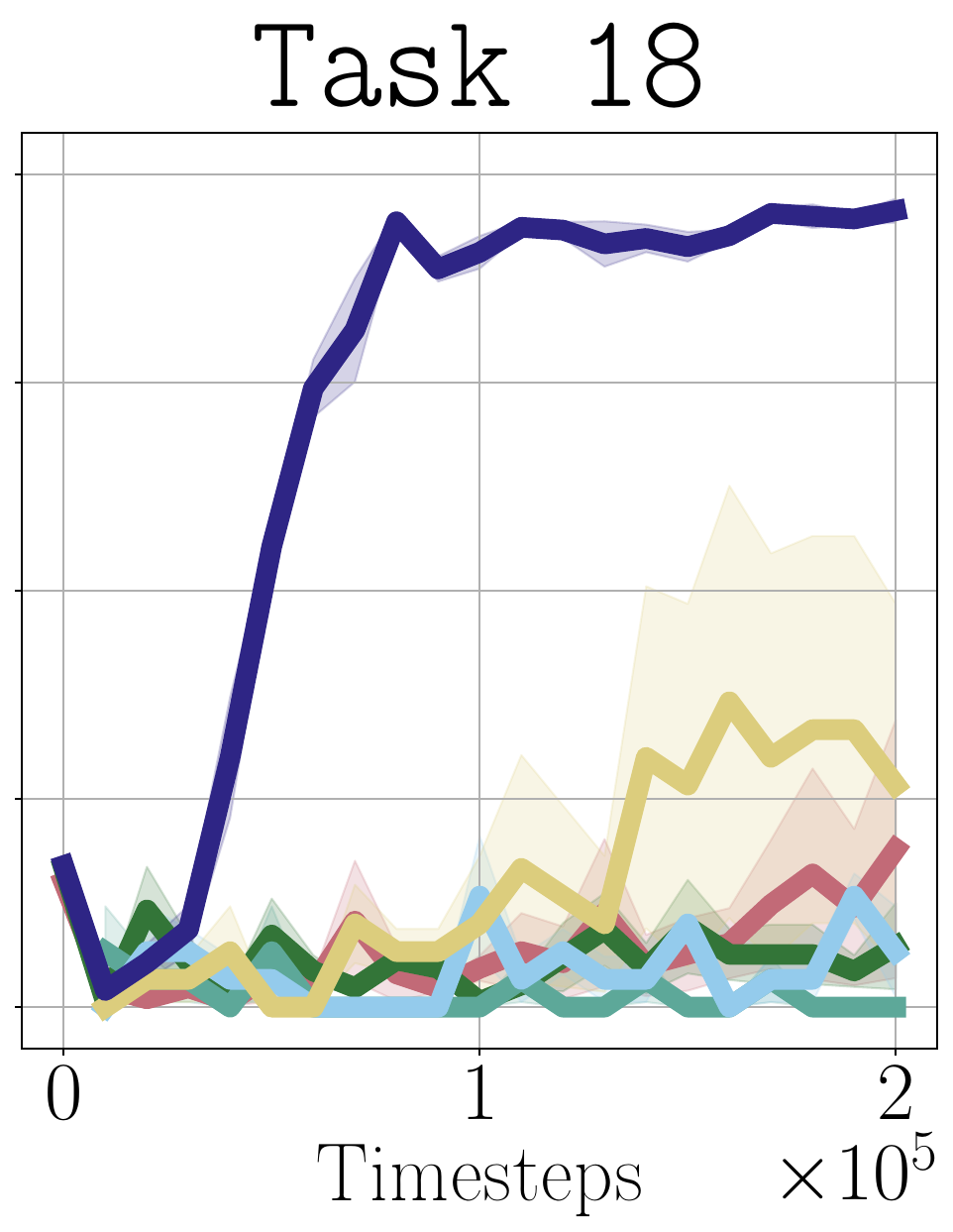}
  \end{subfigure}\hspace{0.0025\textwidth}
  \begin{subfigure}[t]{0.111\textwidth}
    \centering
    \includegraphics[width=\linewidth,
      height=2.75cm,
      keepaspectratio]{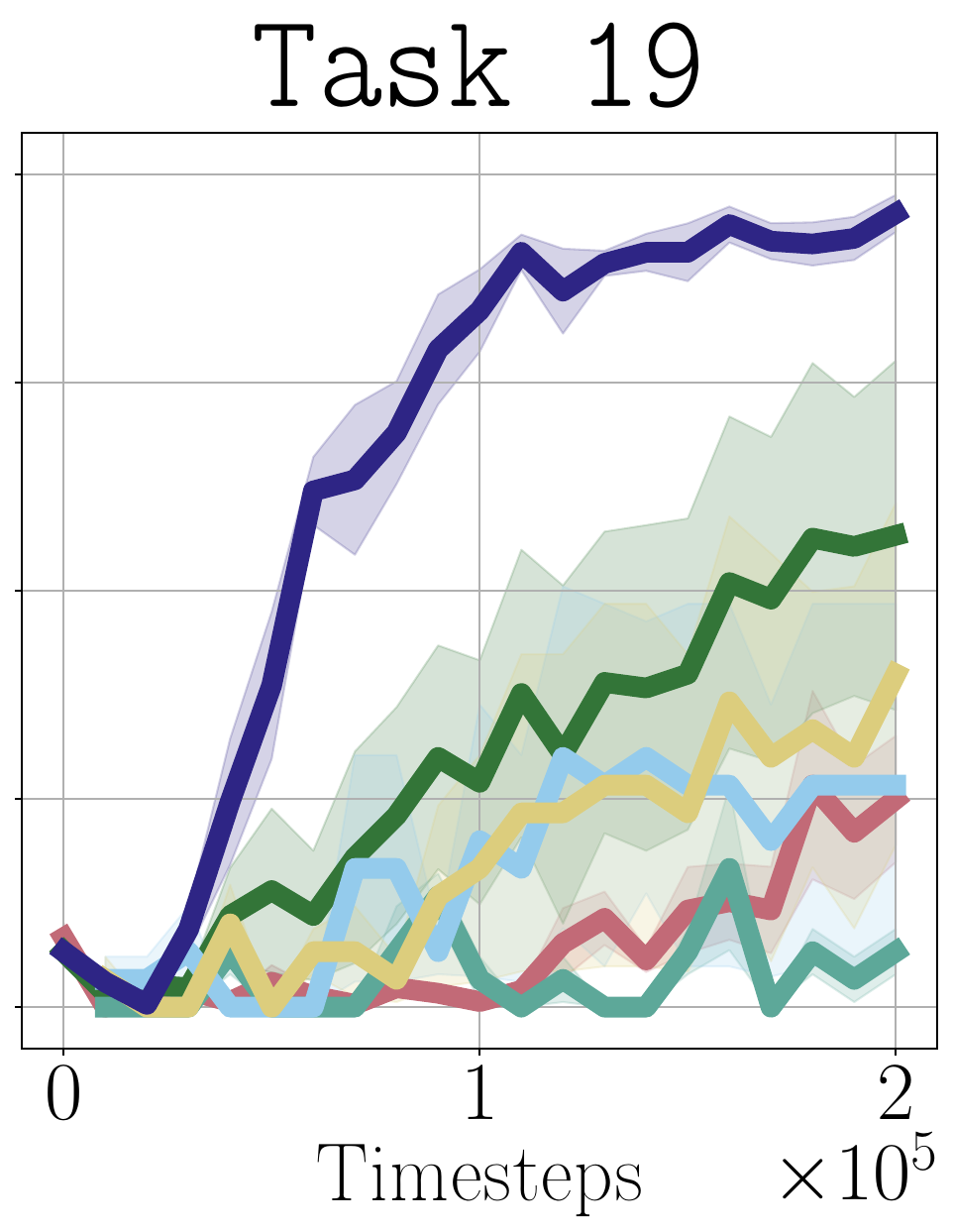}
  \end{subfigure}\hspace{0.0025\textwidth}
  \begin{subfigure}[t]{0.111\textwidth}
    \centering
    \includegraphics[width=\linewidth,
      height=2.75cm,
      keepaspectratio]{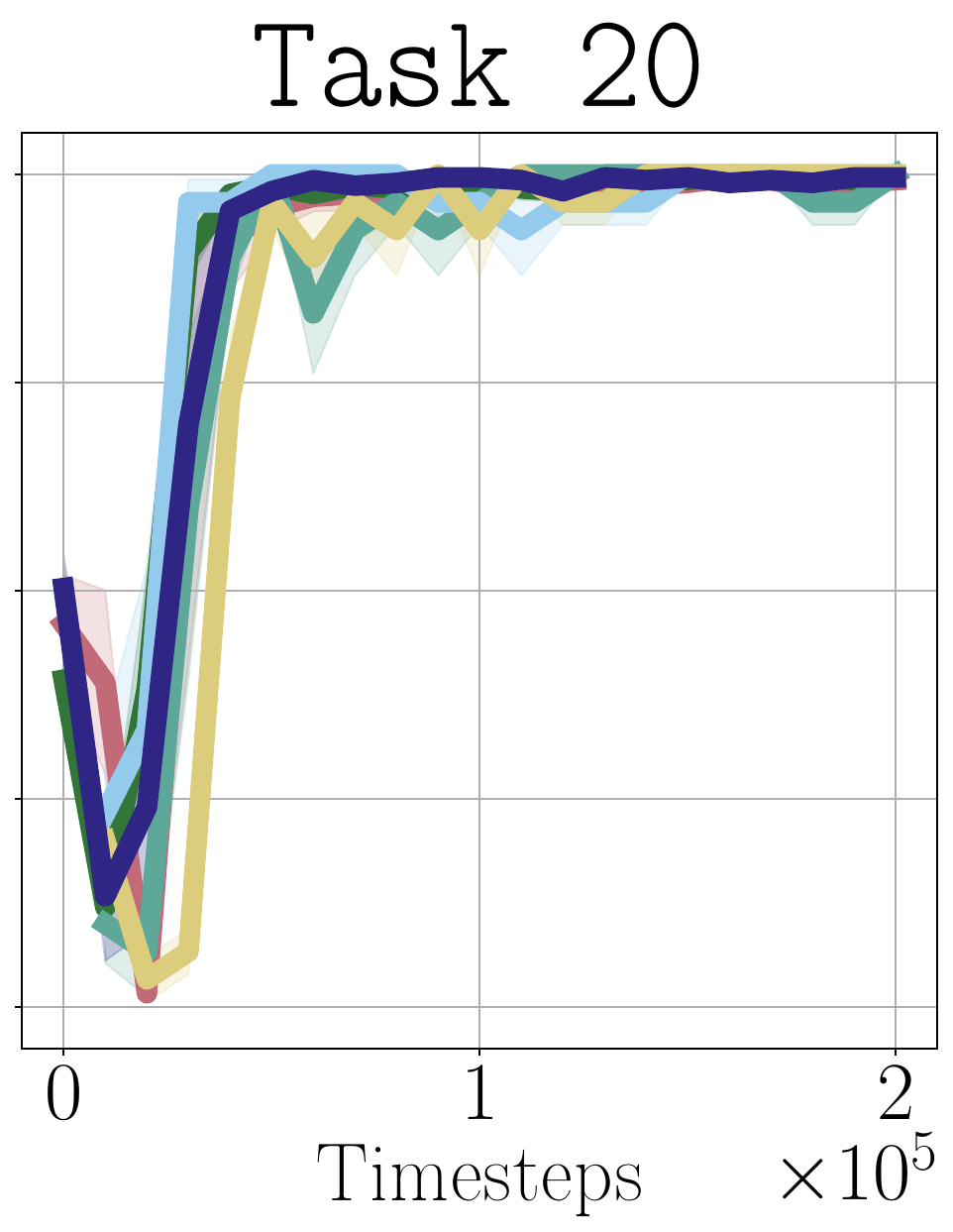}
  \end{subfigure}\hspace{0.0025\textwidth}
  \begin{subfigure}[t]{0.111\textwidth}
    \centering
    \includegraphics[width=\linewidth,
      height=2.75cm,
      keepaspectratio]{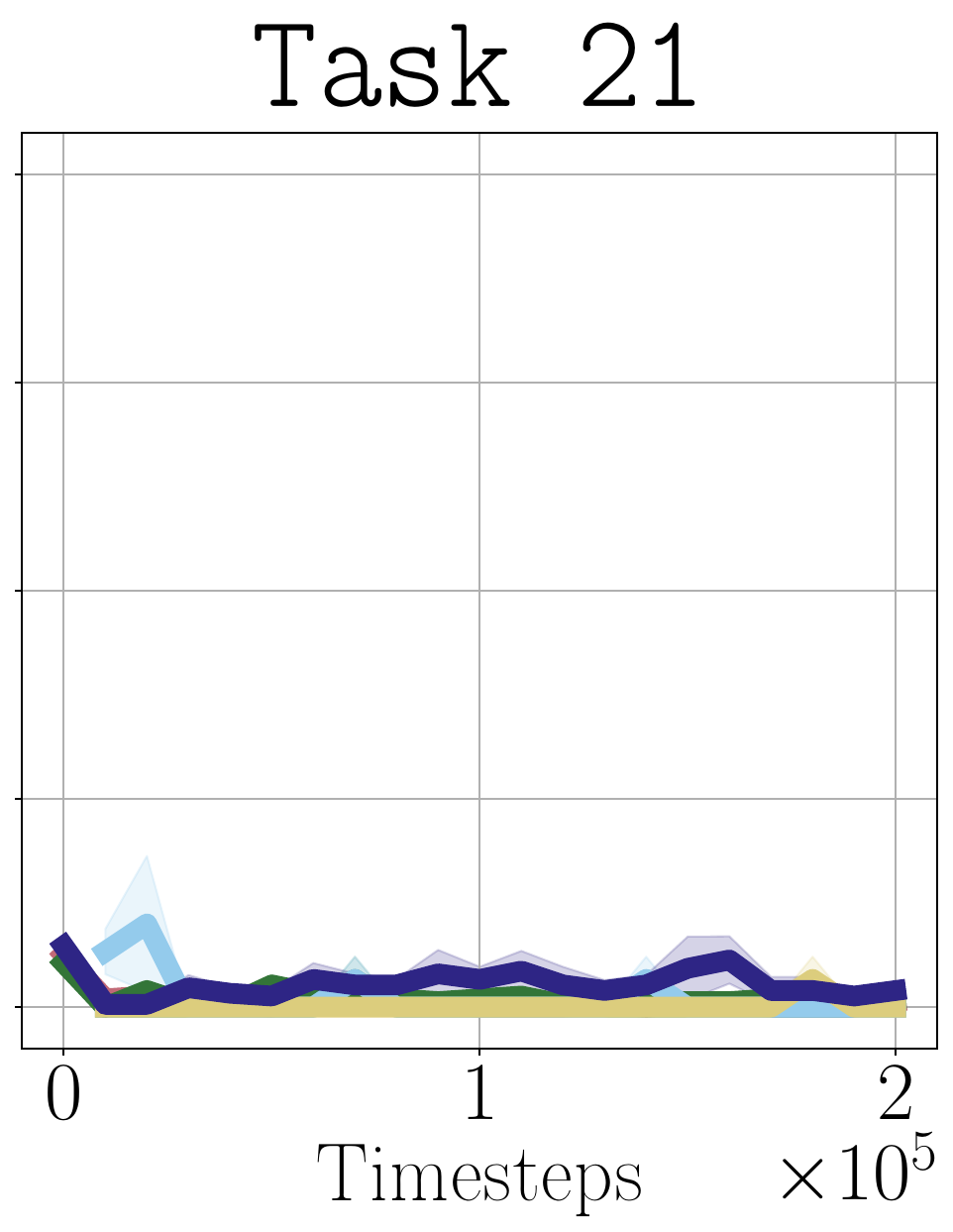}
  \end{subfigure}

  \caption{Per-task results of RL finetuning with \resrl and \prvm process reward on \texttt{LIBERO Kitchen Scene 1--3}.}
  \label{fig:full_resrl_libero}
\end{figure}

\subsection{Individual results for Residual RL on Robocasa}

\begin{figure}[H]
  \centering
  \includegraphics[width=0.98\linewidth]{images/legend.pdf}
  \begin{subfigure}[t]{0.145\textwidth}
    \centering
    \includegraphics[width=\linewidth]{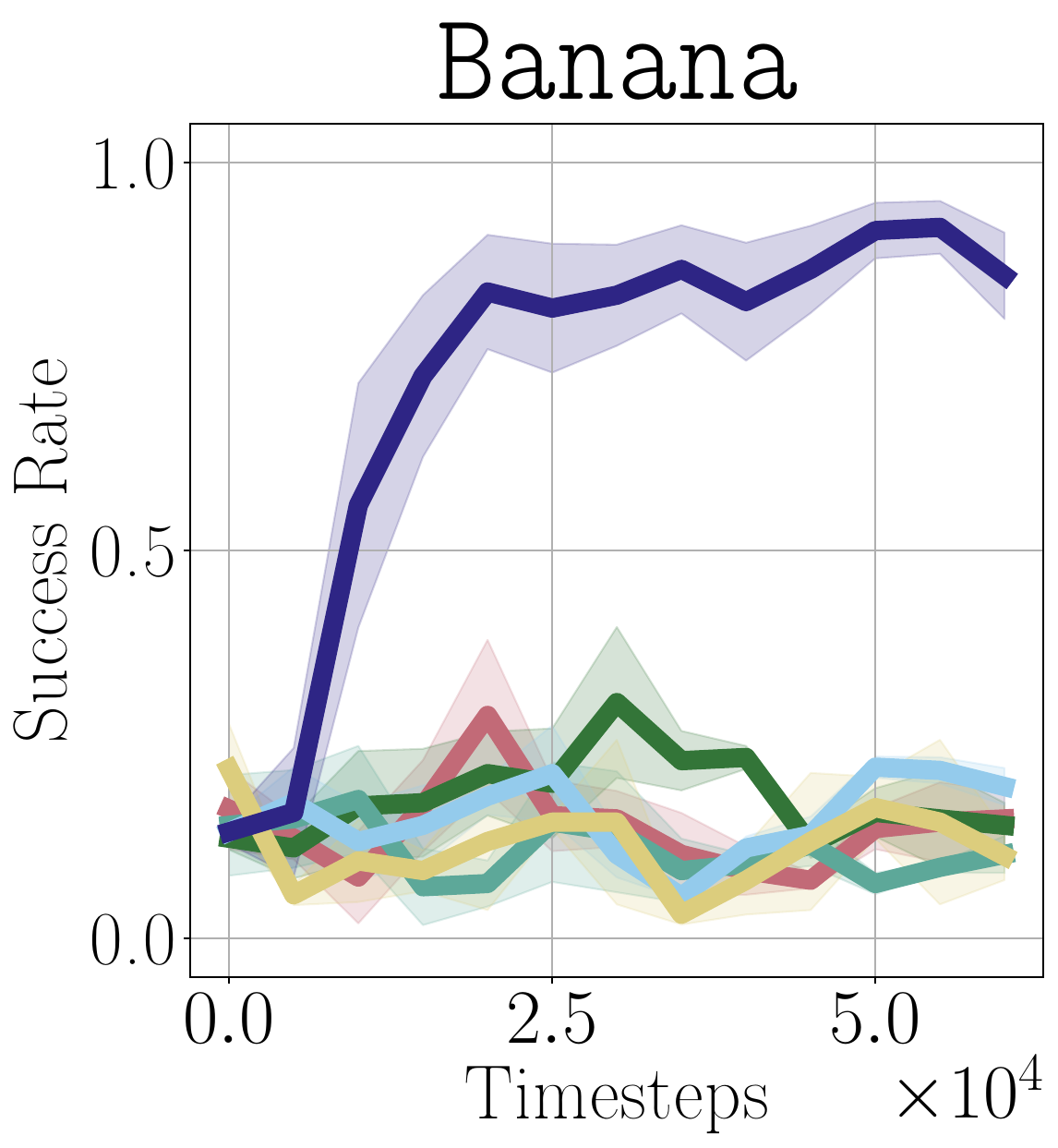}
  \end{subfigure}\hspace{0.015\textwidth}
  \begin{subfigure}[t]{0.12\textwidth}
    \centering
    \includegraphics[width=\linewidth]{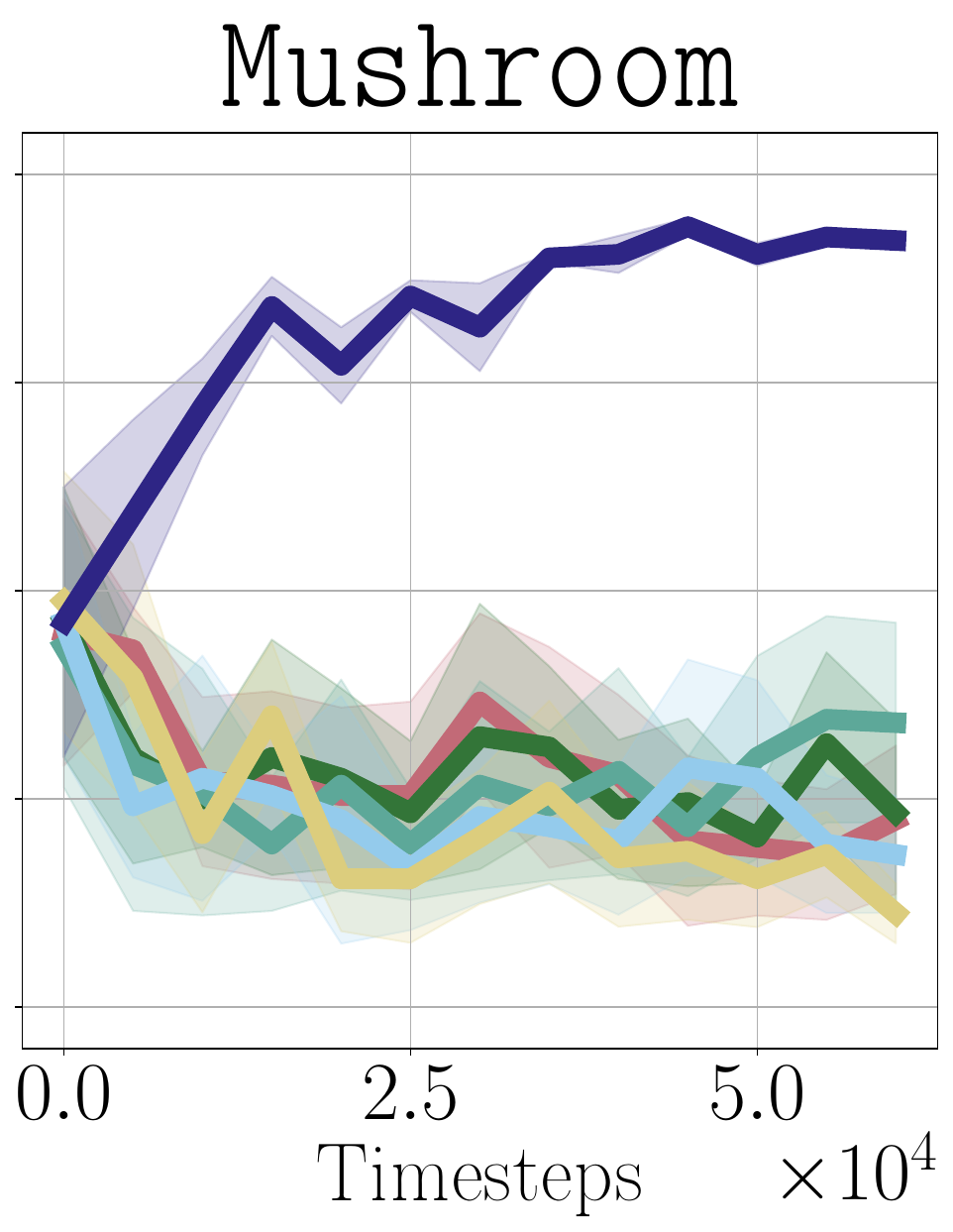}
  \end{subfigure}\hspace{0.015\textwidth}
  \begin{subfigure}[t]{0.12\textwidth}
    \centering
    \includegraphics[width=\linewidth]{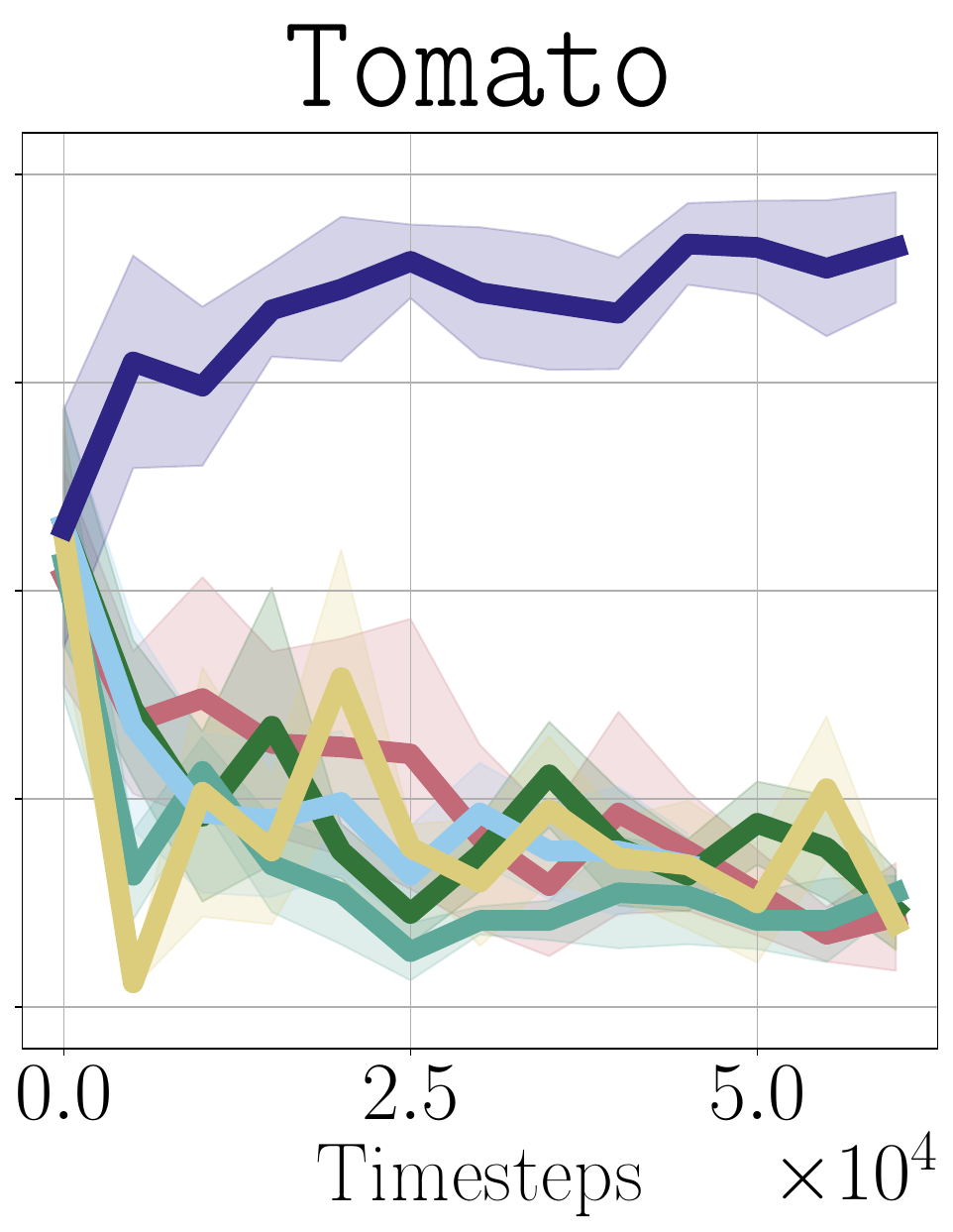}
  \end{subfigure}
  \caption{
    Per task results for RL finetuning with \resrl and $\prvm$ process reward on $\texttt{RoboCasa}$
  }
  \label{fig:all_robocasa}
\end{figure}

\subsection{Individual Results for $\pi_0$ Experiments}

\begin{figure}[H]
  \centering
  \begin{subfigure}[t]{0.18\textwidth}
    \centering
    \includegraphics[width=\linewidth]{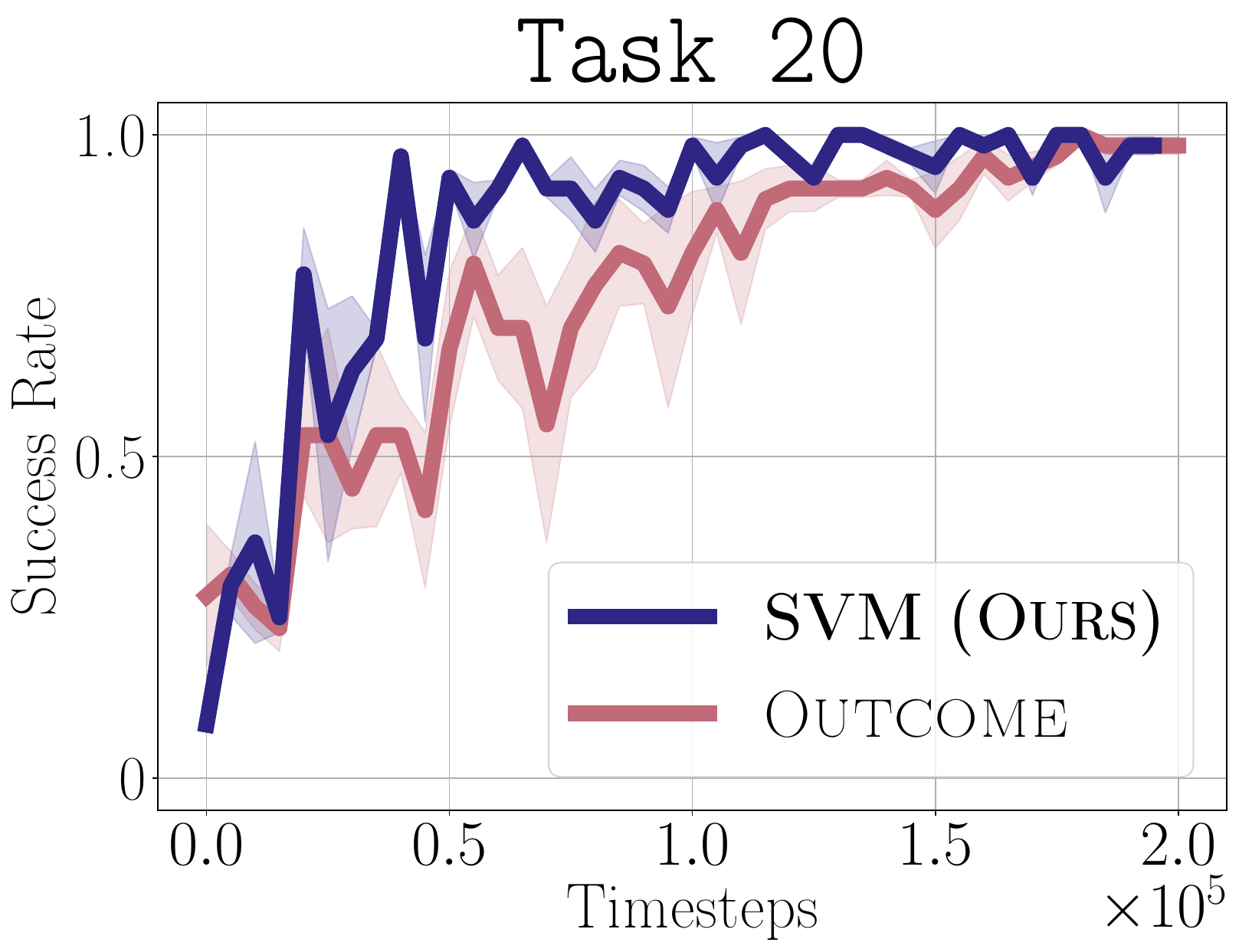}
  \end{subfigure}\hspace{0.015\textwidth}
  \begin{subfigure}[t]{0.16\textwidth}
    \centering
    \includegraphics[width=\linewidth]{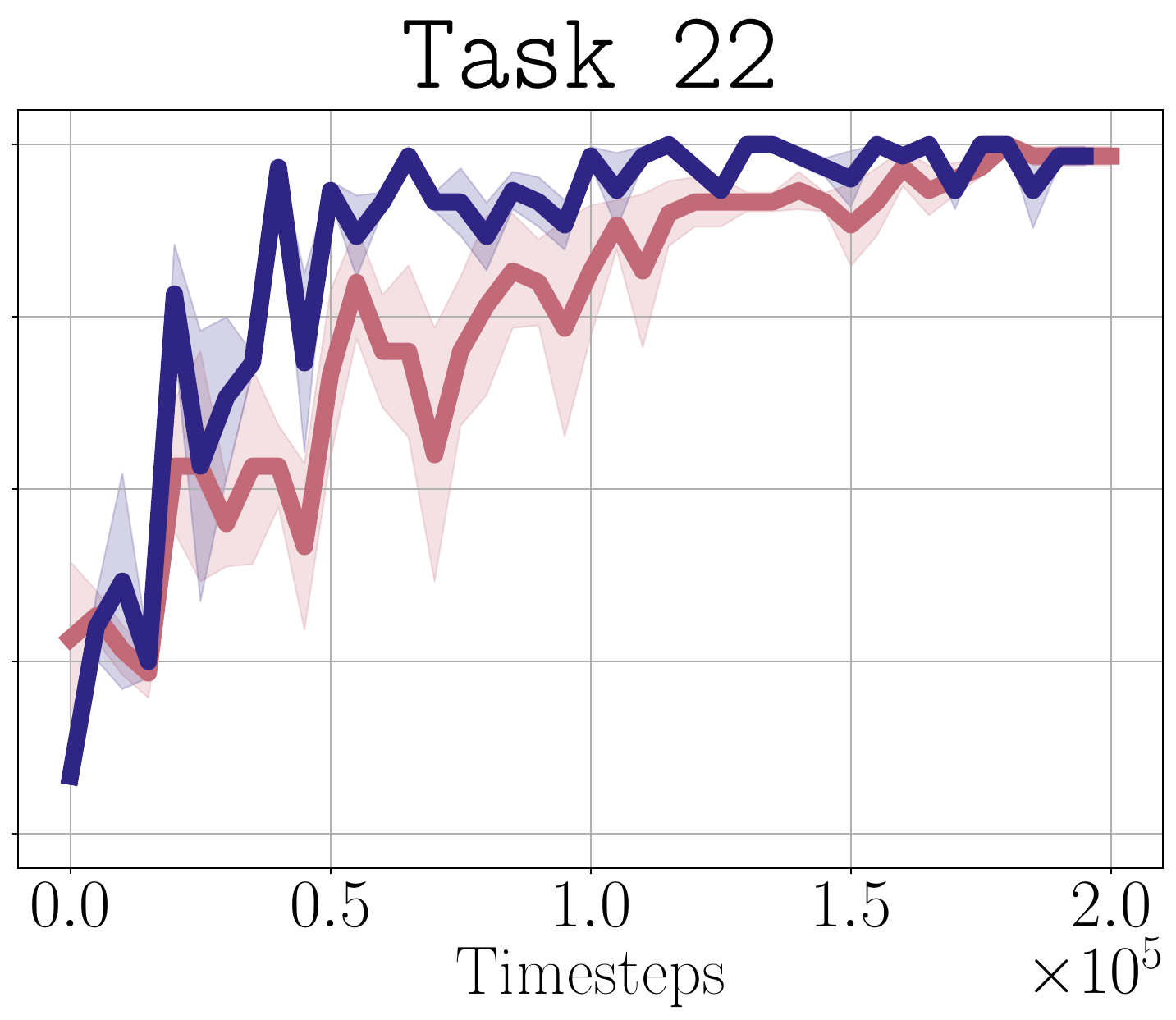}
  \end{subfigure}\hspace{0.015\textwidth}
  \begin{subfigure}[t]{0.16\textwidth}
    \centering
    \includegraphics[width=\linewidth]{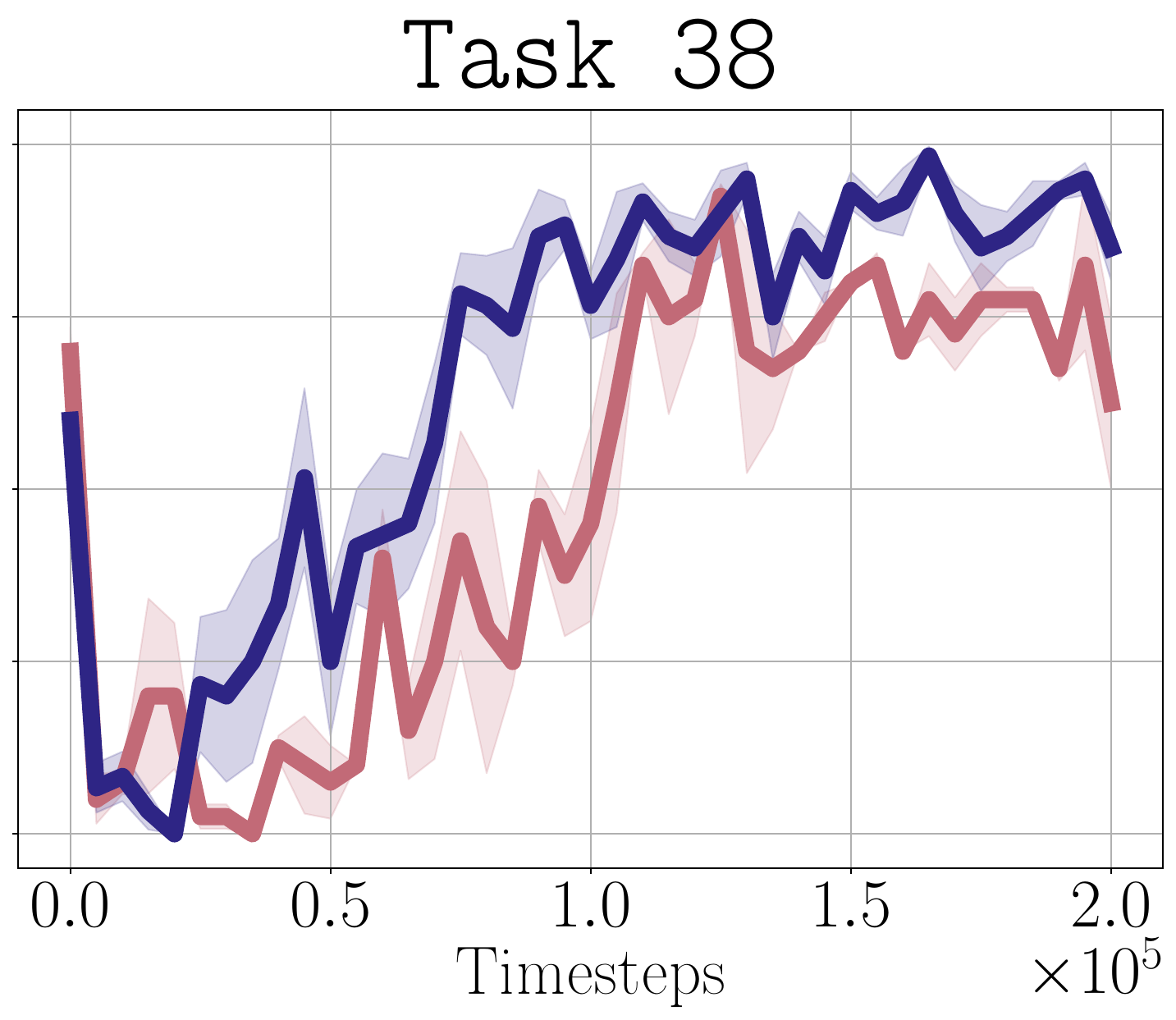}
  \end{subfigure}\hspace{0.015\textwidth}
  \begin{subfigure}[t]{0.16\textwidth}
    \centering
    \includegraphics[width=\linewidth]{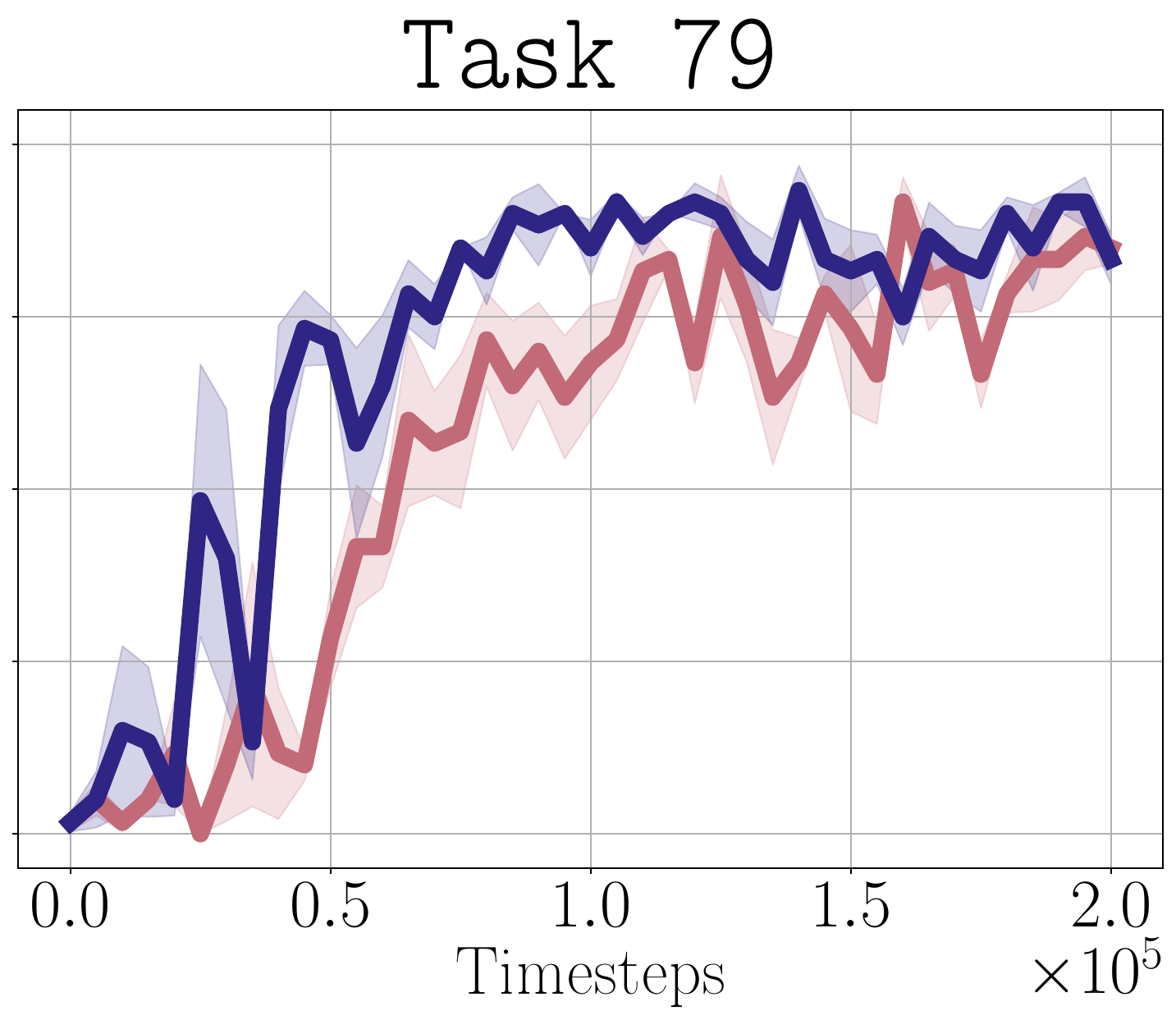}
  \end{subfigure}
  \caption{
    Per task results for RL finetuning of $\pi_0$ with $\dsrl$ and $\prvm$ process reward on $\texttt{LIBERO}$
  }
  \label{fig:all_pi0}
\end{figure}

\subsection{Additional Ablation Results}\label{sec:additional_ablations}

Here we provide several additional ablations of design choices for \prvm process rewards. Unless otherwise stated, the experiments reported here are averaged over tasks from \texttt{LIBERO Kitchen Scene 2}.

\textbf{How can we most effectively extract a policy from $\widehat{f}$?} We provide additional results for this ablation on \texttt{LIBERO Kitchen Scene 1-3} in Figure~\ref{fig:policy_extraction_all}. For each approach, we train the discriminator $\widehat{f}$ online. We describe the ablated policy extraction methods below.

\begin{itemize}
    \item $\widehat{f}$-\textsc{maximization}: We train the policy to directly maximize the discriminator score,
    \[\max_\pi  \mathbb{E}_{(s, a) \sim \pi}\left[\widehat{f}(s, a)\right]\]
    instead of learning a separate $Q$ function.
    
    \item $\widehat{f}$-\textsc{sampling}: During evaluation, at each step, we sample $K=64$ candidate actions from the policy and execute the action with the highest discriminator score. We swept over $K \in \{8, 32, 64, 128\}$ and found that $K \geq 64$ generally performs well.

    \item \dsrl\textsc{+BC}: We train the policy using the standard \dsrl objective with an additional behavior cloning regularizer:
    \[\max_\pi \; \mathbb{E}_{(s, a) \sim \pi}\left[Q(s, a)\right]
    + \lambda \mathbb{E}_{(s, a) \sim \Dpos}\left[\log \pi(a \mid s)\right]\] 
    Here, $\lambda$ controls the strength of the behavior cloning regularization. We swept over $\lambda \in \{0.25, 0.5, 0.75, 1, 2, 3, 10, 30\}$ and found that the method does not perform well across a wide range of values. We report results with $\lambda = 1$.
\end{itemize}

\begin{figure}[h]
  \centering
  \hspace*{1.4cm}%
  \begin{minipage}{\textwidth}
    \centering
    \begin{minipage}[t]{0.175\textwidth}
      \centering
      \includegraphics[width=\linewidth]{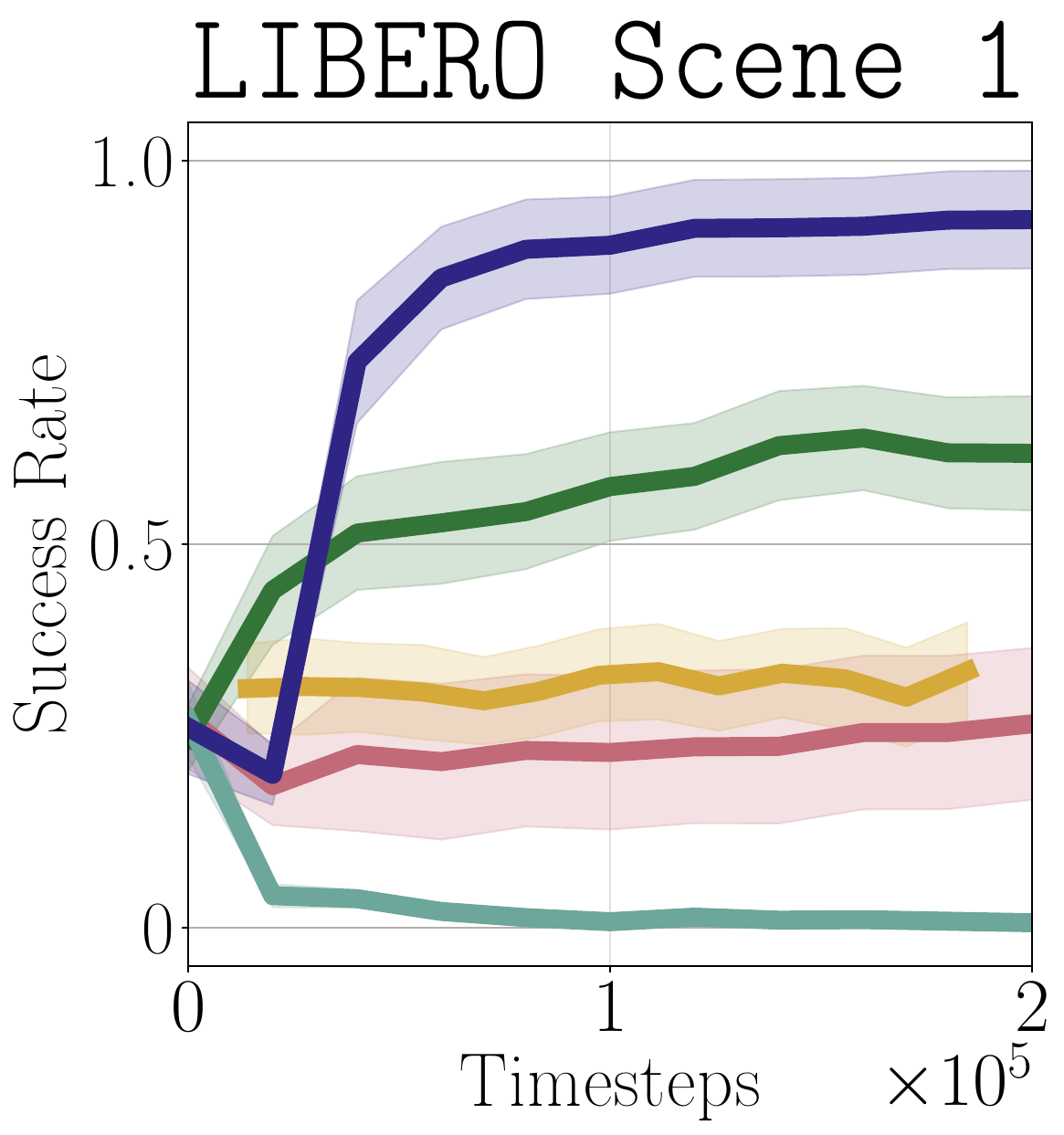}
    \end{minipage}
    \hspace{0.01\textwidth}
    \begin{minipage}[t]{0.15\textwidth}
      \centering
      \includegraphics[width=\linewidth]{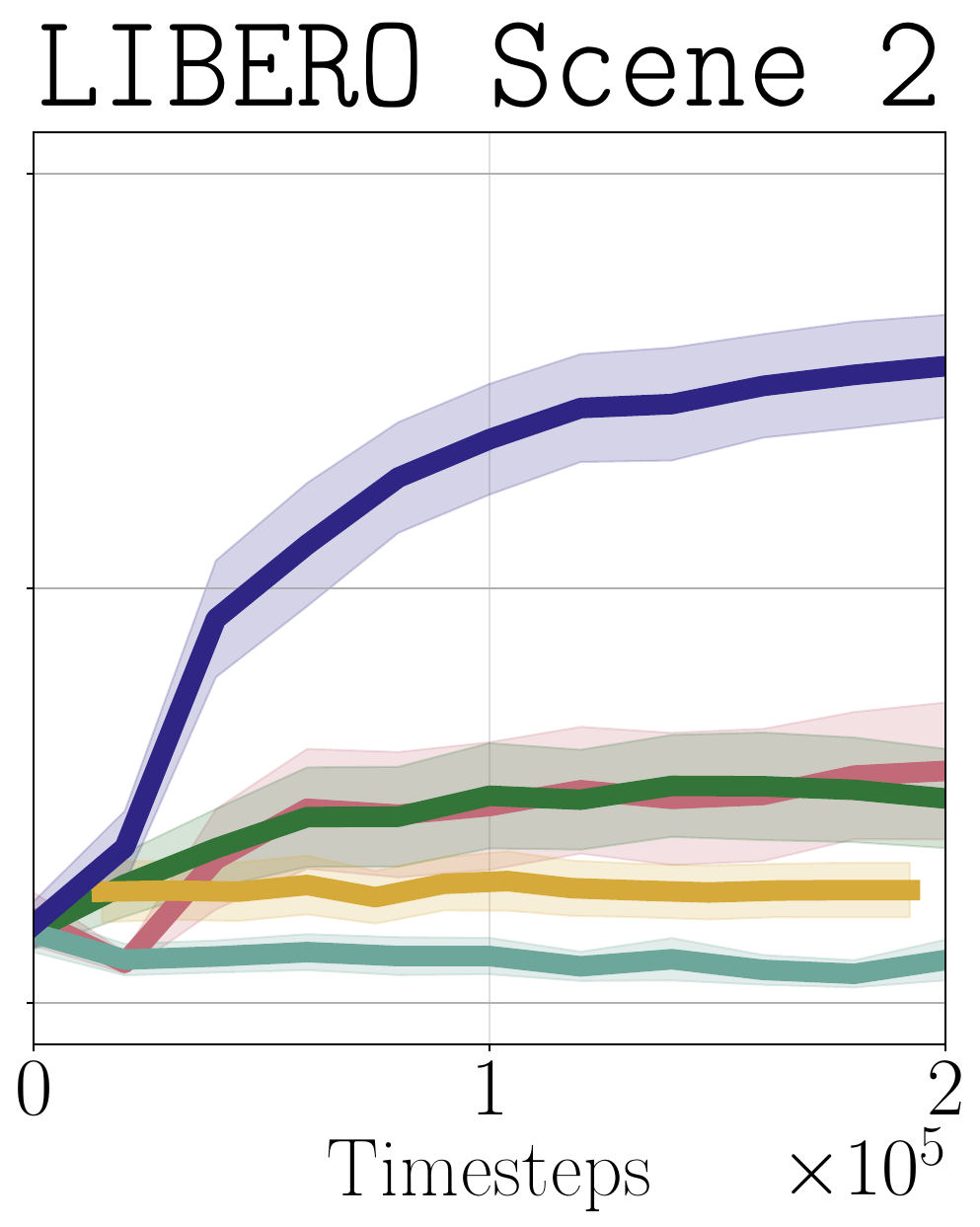}
    \end{minipage}
    \hspace{0.01\textwidth}
    \begin{minipage}[t]{0.15\textwidth}
      \centering
      \includegraphics[width=\linewidth]{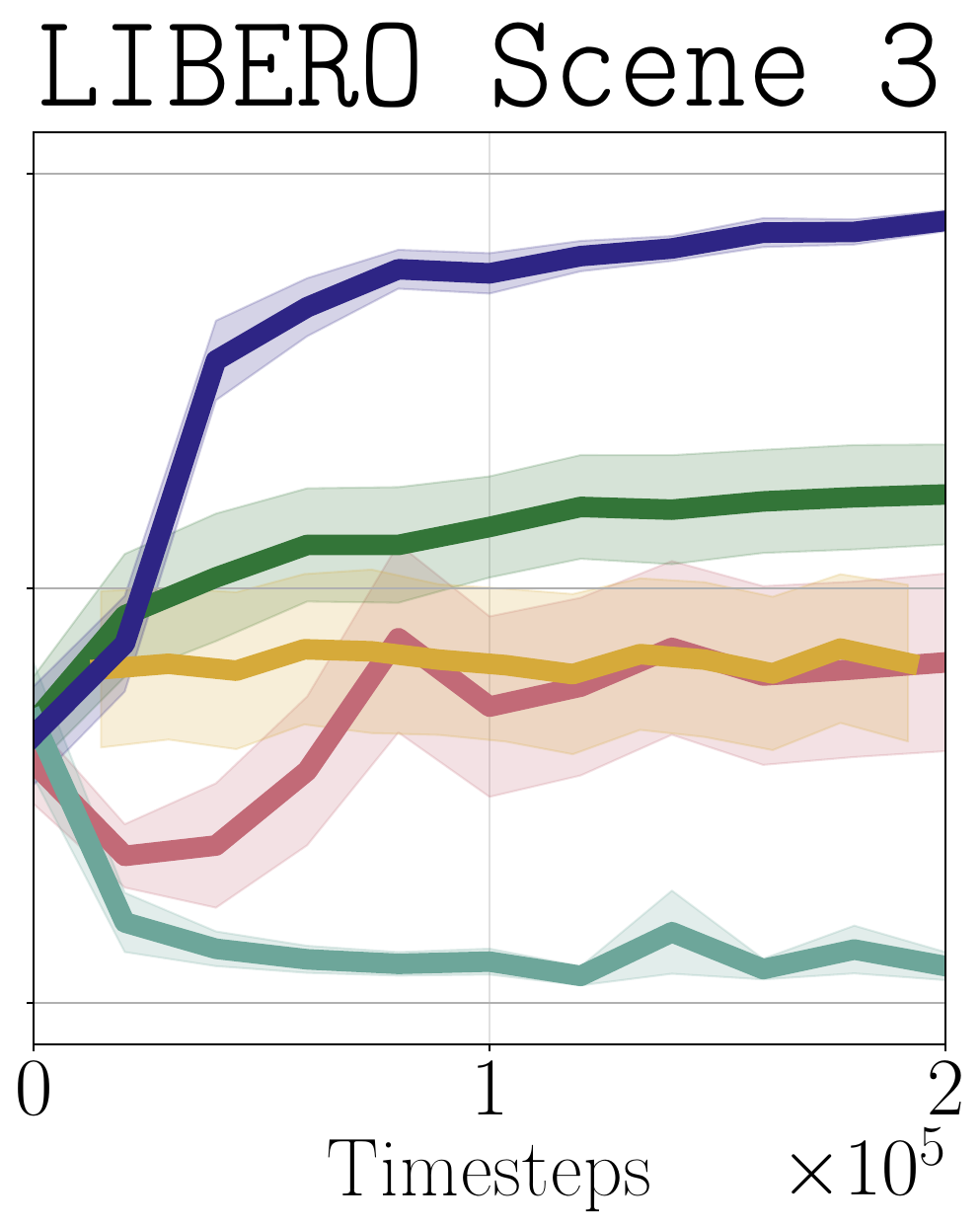}
    \end{minipage}
    \begin{minipage}[t]{0.20\textwidth}
      \centering
      \raisebox{0.35cm}[0pt][0pt]{%
        \makebox[\linewidth][c]{%
          \hspace*{-5mm}\includegraphics[height=3.5cm]{images/classifier_usage_legend.pdf}%
        }%
      }
    \end{minipage}
  \end{minipage}
  \caption{Ablation of policy extraction approach on \texttt{LIBERO Kitchen Scene 1-3}.}
  \label{fig:policy_extraction_all}
\end{figure}

\textbf{How does the functional form of the process reward impact performance?} We provide additional ablation results for \resrl in Figure~\ref{fig:reward_functional_form_all}. In particular, we note that while the current form bears a close relationship to KL-regularized RL and adversarial imitation learning, \Cref{thm:proc_consistent} remains true for a variety of monotonic transformations of the visitation probabilities.

\begin{figure}[h]
  \centering
  \makebox[\textwidth][c]{%
    \hspace*{1.5cm}%
    \begin{minipage}{0.60\textwidth}
      \centering
      \begin{minipage}[t]{0.29\linewidth}
        \centering
        \includegraphics[width=\linewidth]{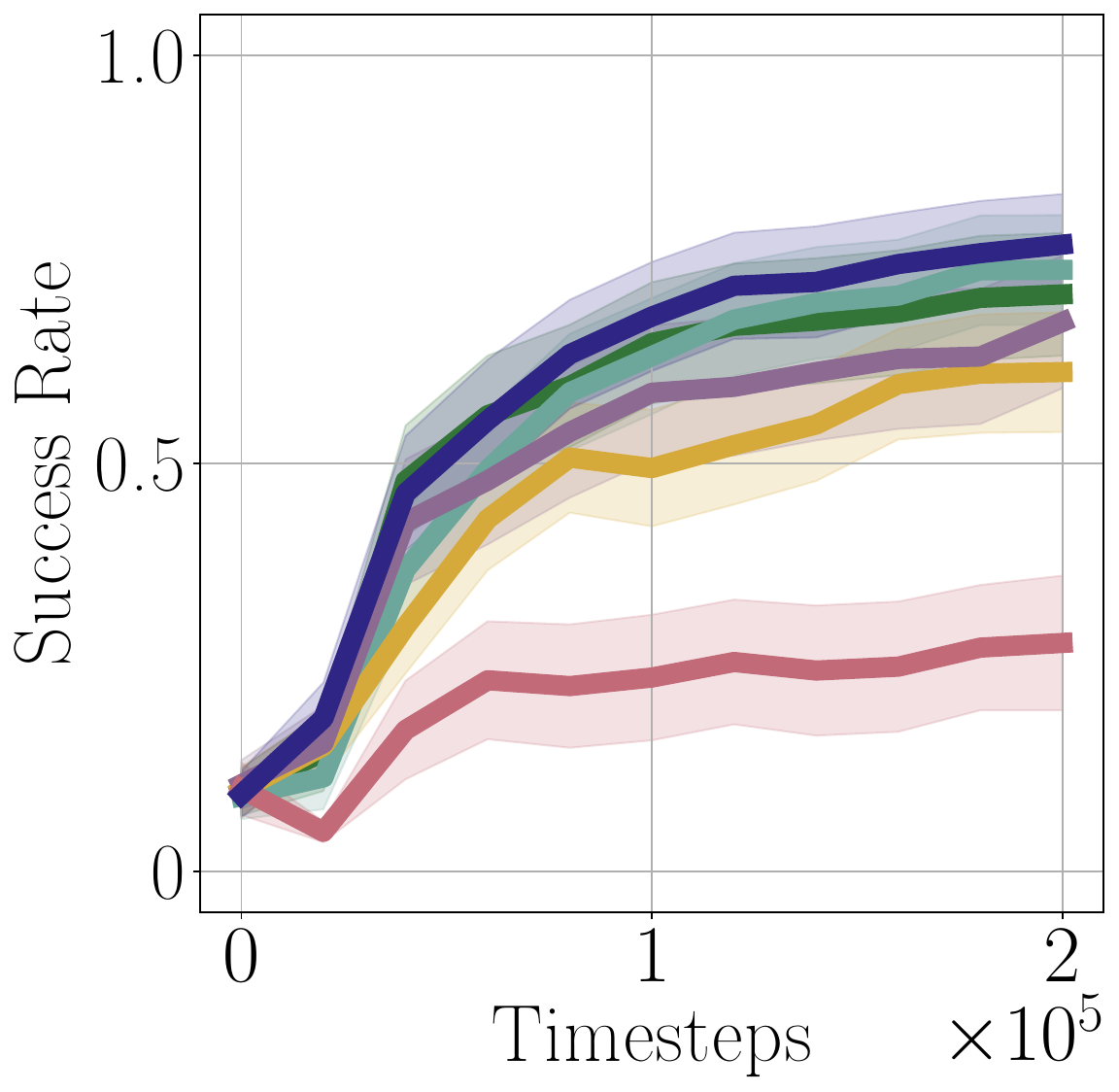}
      \end{minipage}
      \hspace{0.02\linewidth}
      \begin{minipage}[t]{0.25\linewidth}
        \centering
        \includegraphics[width=\linewidth]{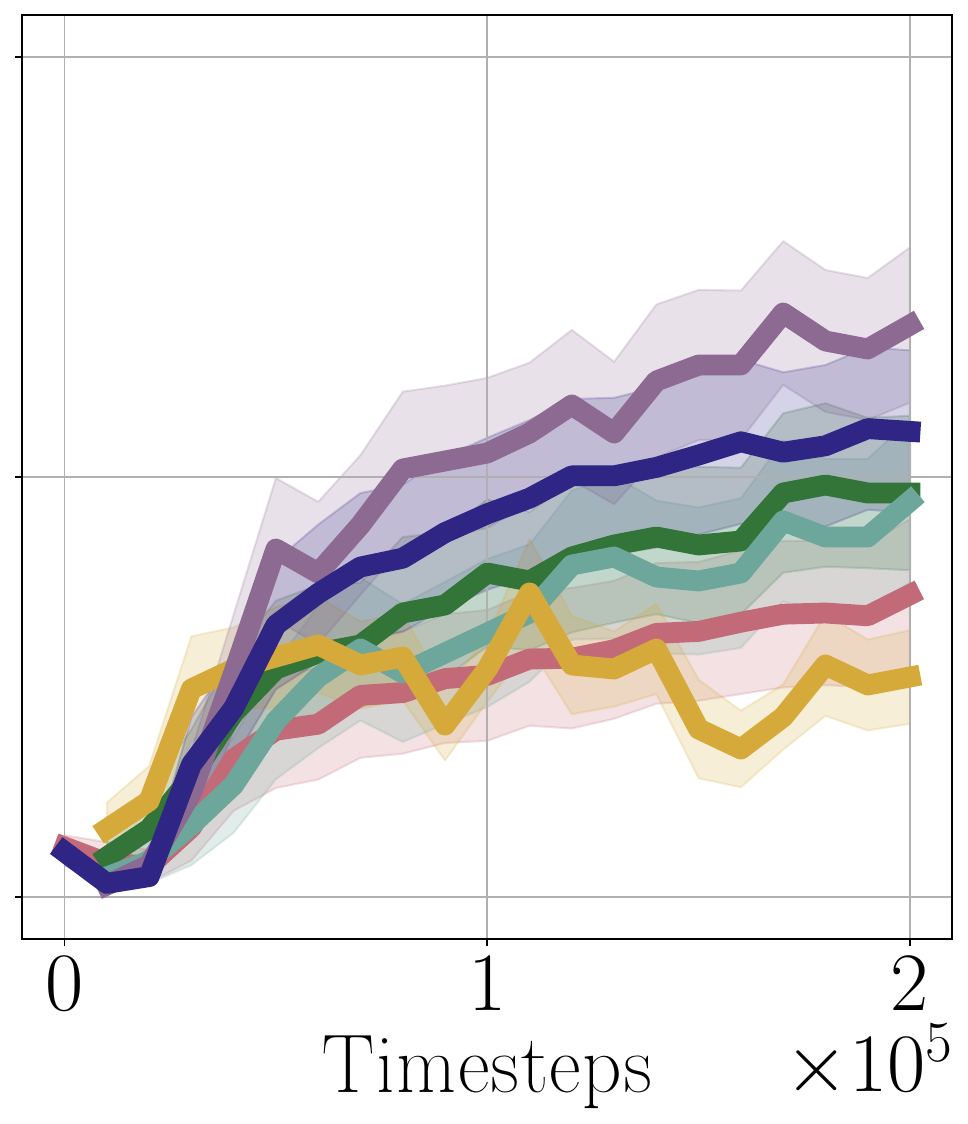}
      \end{minipage}
      \begin{minipage}[t]{0.33\linewidth}
        \centering
        \raisebox{0.265cm}[0pt][0pt]{%
          \makebox[\linewidth][c]{%
            \hspace*{-4mm}\includegraphics[height=3.2cm]{images/ablation_function_legend.pdf}%
          }%
        }
      \end{minipage}
    \end{minipage}%
  }
  \caption{Aggregated success rates for \dsrl (Left) and \resrl (Right) using different monotone transformations of the predicted success probability $\widehat{f}$.}
  \label{fig:reward_functional_form_all}
\end{figure}

\Cref{fig:reward_functional_form_all} illustrates the performance of RL with various transformations of $\widehat{f}$, the discriminator score. We see in particular that, while many such transformations do lead to improved performance over only outcome rewards, the specific form of \prvm, $\log \widehat{f} / (1-\widehat{f})$ leads to the most consistently effective performance.

\textbf{Does \prvm require timestep conditioning?} 
Our theoretical analysis motivates conditioning the classifier on the timestep $h$. Here we explore whether this impacts performance in practice. 
We illustrate \prvm with and without timestep conditioning in Figure~\ref{fig:conditioning_exp}, and see that timestep conditioning does not on average improve performance. 

\begin{figure}[h]
  \centering
  \makebox[\textwidth][c]{%
    \hspace*{1.5cm}%
    \begin{minipage}{0.60\textwidth}
      \centering
      \begin{minipage}[t]{0.29\linewidth}
        \centering
        \includegraphics[width=\linewidth]{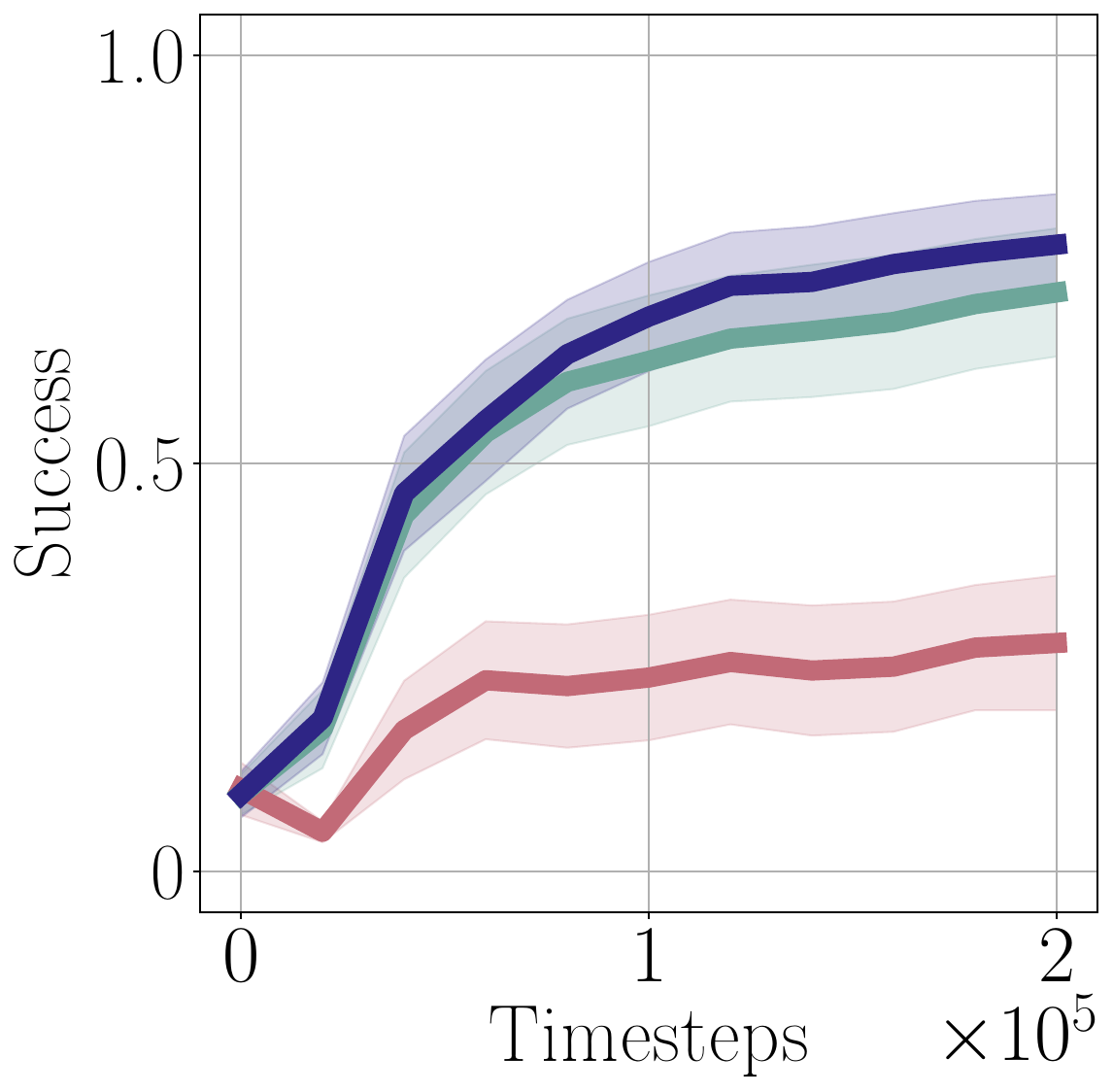}
      \end{minipage}
      \begin{minipage}[t]{0.33\linewidth}
        \centering
        \raisebox{0.265cm}[0pt][0pt]{%
          \makebox[\linewidth][c]{%
            \hspace*{2mm}\includegraphics[height=3.2cm]{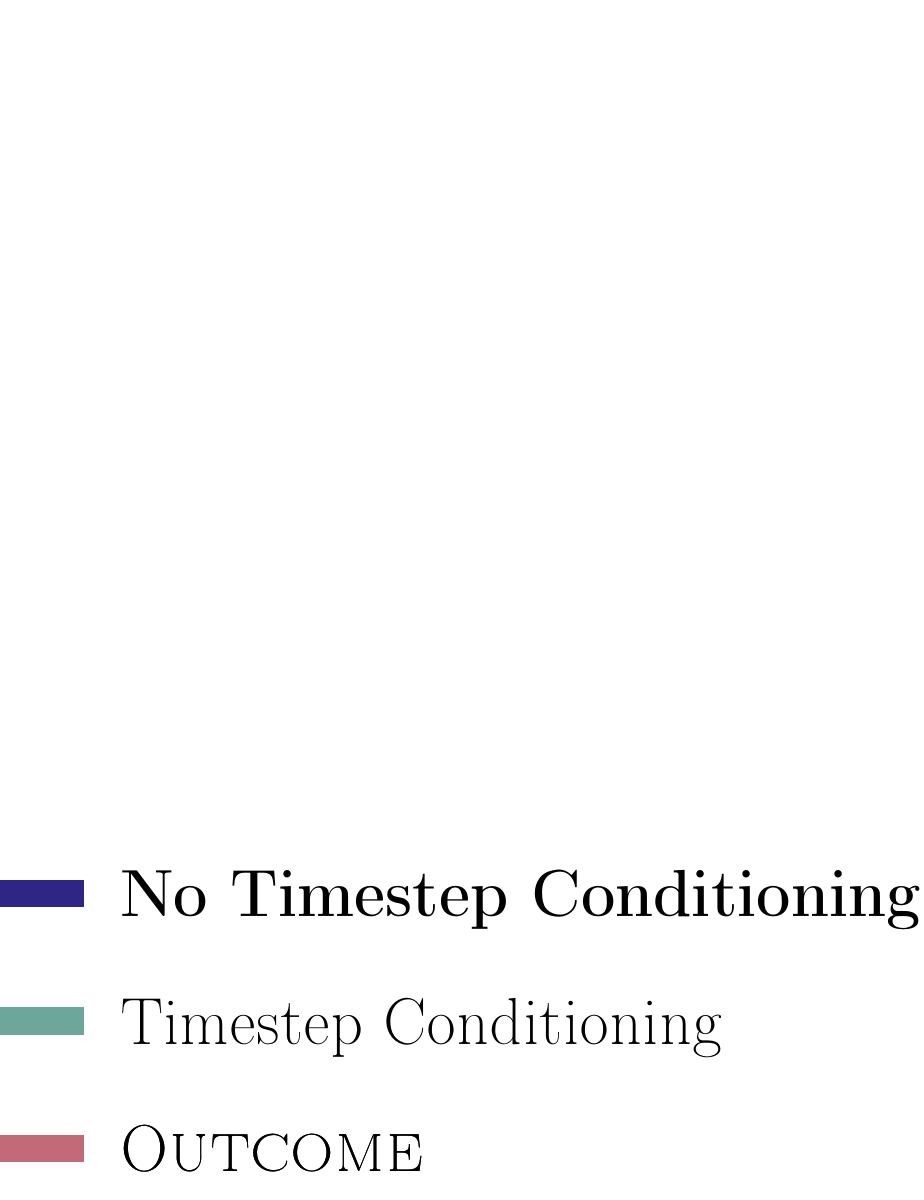}
          }%
        }
      \end{minipage}
    \end{minipage}%
  }
  \caption{Timestep conditioning with \dsrl.}
  \label{fig:conditioning_exp}
\end{figure}

\textbf{Does \prvm require symmetric sampling?} Our theory in Section~\ref{sec:method} suggests that to train the discriminator $\fhat$ we should sample symmetrically from $\Dpos$ and $\Dneg$. If we instead rely on asymmetric sampling, where we sample from $\Dpos$ and $\Dneg$ proportional to their size, the learned $\fhat$ instead corresponds, effectively, to a $Q$-function. 
Here we evaluate whether \prvm is sensitive to the form of sampling.  As shown in Figure~\ref{fig:sampling_exp}, symmetric sampling yields on average slightly higher aggregate performance. However, while the aggregate performance difference is not substantial, on 3/7 tasks in \texttt{Kitchen Scene 2}, we find that symmetric sampling is essential to achieving effective RL improvement and, furthermore, see a larger aggregate gap between symmetric and asymmetric sampling when using \resrl as opposed to \dsrl. Thus, using a ``discriminator'' as opposed to a $Q$-function to shape the reward is in general important for achieving the best performance.

\begin{figure}[H]
  \centering
  \makebox[\textwidth][c]{%
    \hspace*{1.5cm}%
    \begin{minipage}{0.60\textwidth}
      \centering
      \begin{minipage}[t]{0.295\linewidth}
        \centering
        \includegraphics[width=\linewidth]{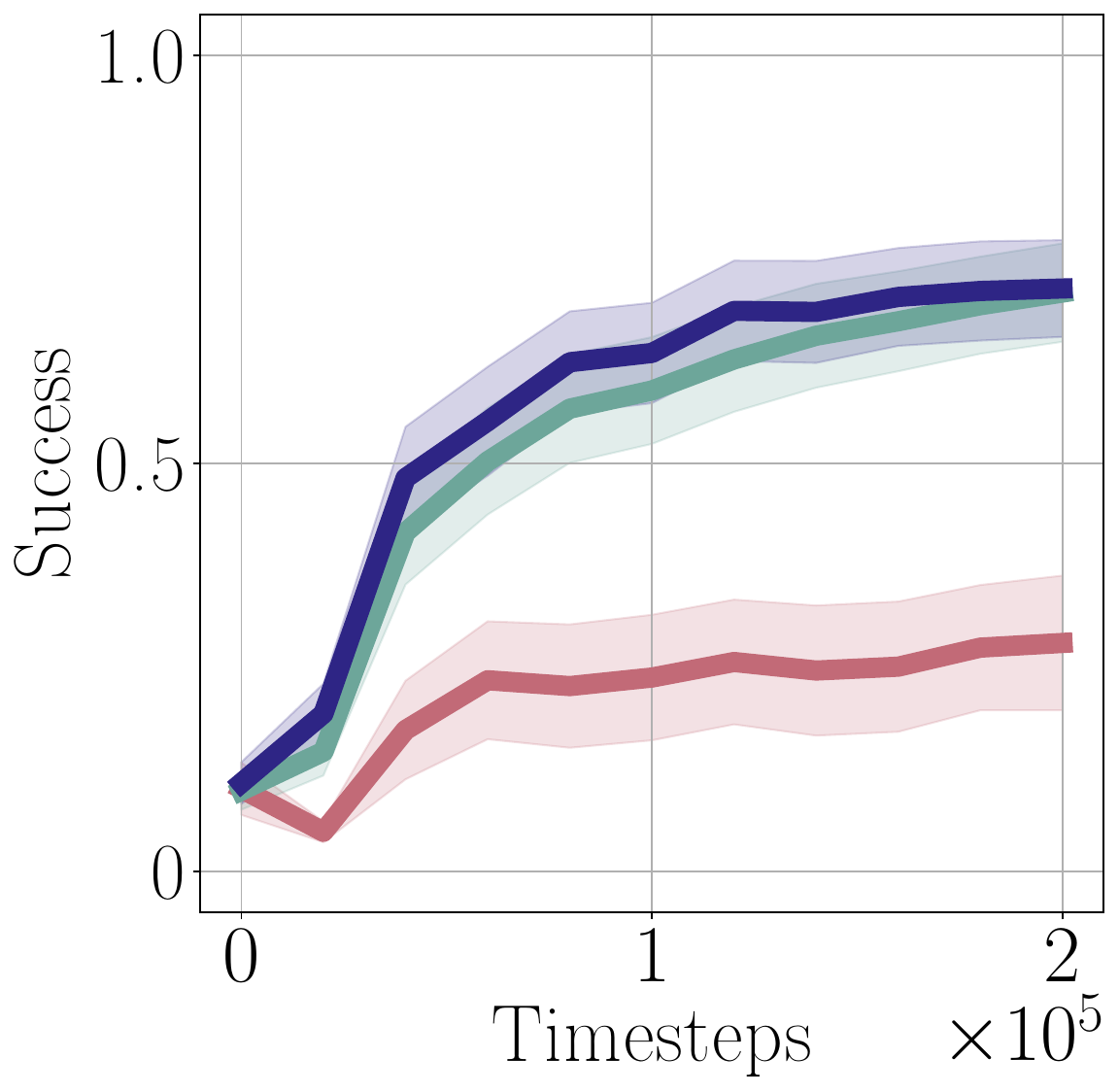}
      \end{minipage}
      \hspace{0.02\linewidth}
      \begin{minipage}[t]{0.25\linewidth}
        \centering
        \includegraphics[width=\linewidth]{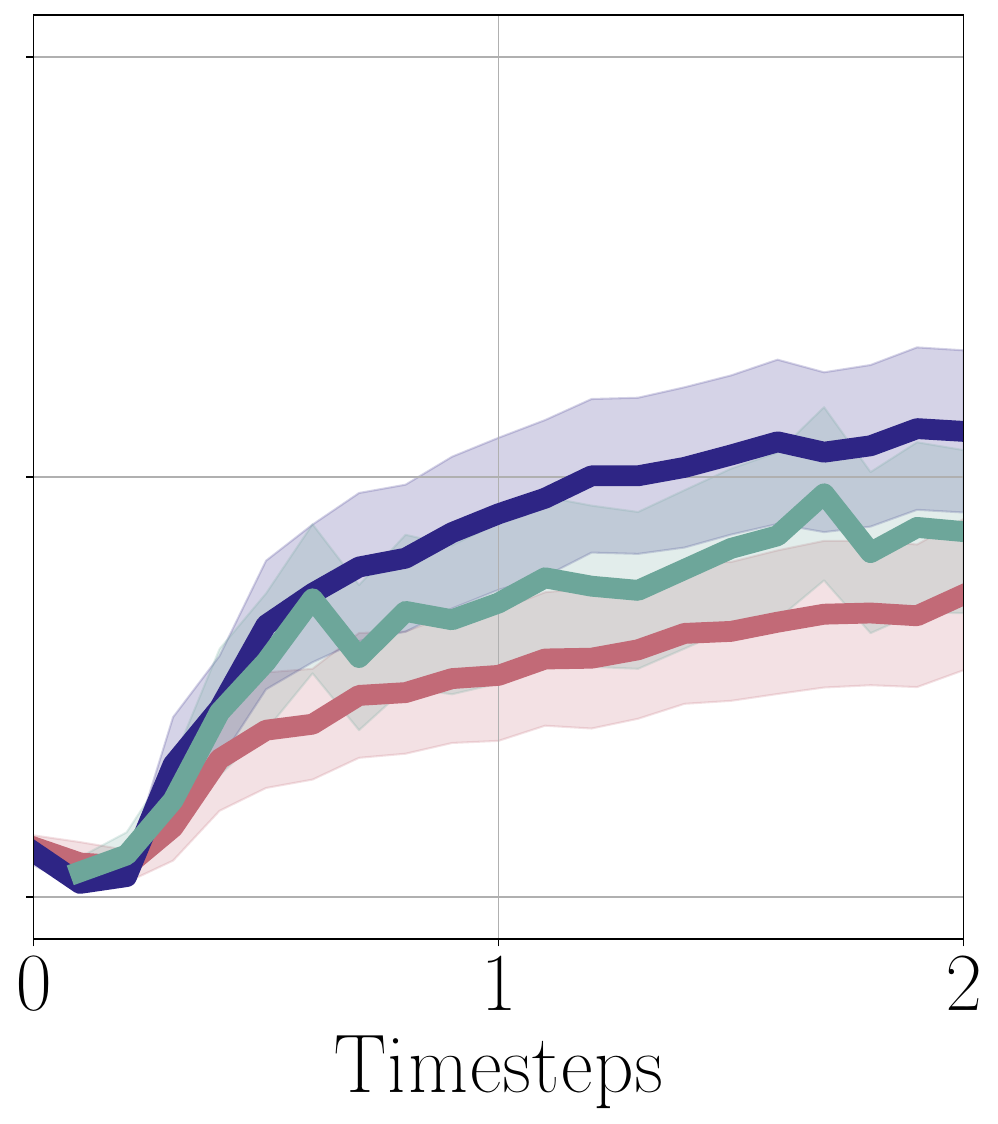}
      \end{minipage}
      \begin{minipage}[t]{0.33\linewidth}
        \centering
        \raisebox{0.265cm}[0pt][0pt]{%
          \makebox[\linewidth][c]{%
            \hspace*{-10mm}\includegraphics[height=3.2cm]{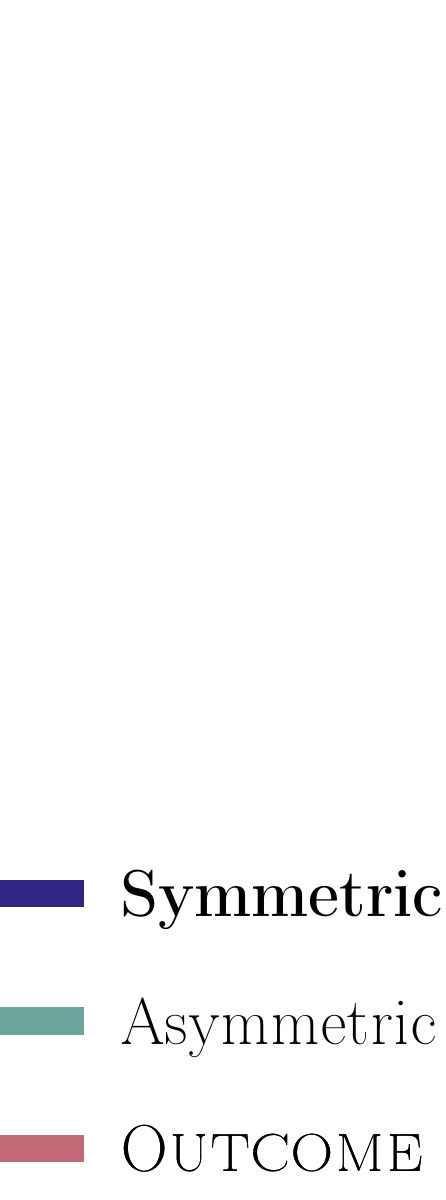}
          }%
        }
      \end{minipage}
    \end{minipage}%
  }
  \caption{Symmetric sampling ablation with \dsrl (Left) and \resrl (Right).}
  \label{fig:sampling_exp}
\end{figure}

\section{Experimental Details}\label{sec:exp_details}

Finally, we provide additional experimental details on each approach we consider.

\subsection{Details of Libero Experiments}\label{sec:libero_exp_details}

\textbf{Base Policy Training.} We instantiate the base policy $\pi_{\text{pre}}$ as a diffusion transformer policy~\cite{dasari2024ingredients}. The policy is pretrained via behavioral cloning on the full \texttt{LIBERO-90} dataset. For task conditioning, we use BERT language embeddings~\cite{devlin2019bertpretrainingdeepbidirectional} of the task descriptions. The hyperparameters used for base policy pretraining are reported in Table~\ref{tab:libero_policy_hyperparam}.

\textbf{DSRL Implementation.} For \dsrl, we use the official implementation from~\cite{wagenmaker2025steering}, specifically the \textsc{DSRL-SAC} algorithm. The noise policy $\pi_{\dsrl}(z \mid s)$ is trained using the standard SAC algorithm, performing $k$ gradient steps per training update. We additionally constrain the noise action magnitude $b_{\mathcal{W}}$ by clipping each output dimension of $\pi_{\dsrl}(z \mid s)$ to the interval $[-b_{\mathcal{W}}, b_{\mathcal{W}}]$.

\textbf{Residual RL Implementation.} For \resrl, the residual policy $\pi_{\text{res}}$ is also trained using \textsc{SAC} \cite{haarnoja2018soft}, likewise performing $k$ gradient steps per training update. The residual policy is conditioned on both the observation $s$ and the base policy action $a_0$. The final action executed in the environment is given by $a=a_0+\gamma \cdot \pi_{\text{res}}(s, a_0)$, with $\gamma$ controlling the residual scale. 

Prior to the main experiments, we perform a hyperparameter sweep over the number of gradient steps $k$, the process reward scale $\lambda$, and the action magnitude $b_{\mathcal{W}}$ for \dsrl, as well as the residual scale $\gamma$ for \resrl. We observe comparable performance across \dsrl and \resrl for the shared hyperparameters, and therefore use the same set of values for both methods. We evaluate \prvm on \texttt{Kitchen Scene 1-3}, and for each task we use the same set of hyperparameters. The full set of hyperparameters used for \dsrl and \resrl is reported in Table~\ref{tab:dsrl_libero_hyperparams}.

\textbf{Discriminator Implementation.} We parameterize the discriminator $\widehat{f}$ as a three-layer MLP with hidden size $256$. The discriminator takes as input the image representations extracted by the ResNet encoder of the pretrained diffusion policy, concatenated with the actions. The discriminator is updated $K$ times every $T$ environment steps using the Adam optimizer with a learning rate of $10^{-4}$. The discriminator training hyperparameters are reported in Table~\ref{tab:libero_classifier_hyper}.

\begin{table}[h]
\caption{
\footnotesize
\textbf{Base policy pretraining hyperparameter for \texttt{LIBERO} experiments.}
}
\label{tab:libero_policy_hyperparam}
\vspace{5pt}
\begin{center}
\scalebox{0.9}
{
\begin{tabular}{ll}
    \toprule
    \textbf{Hyperparameter} & Value \\
    \midrule
Batch size &  $150$ \\
Learning rate & $0.0003$ \\
Training steps & $50000$ \\
LR scheduler & cosine \\
Warmup steps & $2000$ \\
\hline
Action chunk size & $4$ \\
Train denoising steps & $100$ \\
Inference denoising steps & $8$ \\
Image encoder & ResNet-18 \\
Hidden size & $256$ \\
Number of Heads & $8$ \\
Number of Layers & $4$ \\
Feedforward dimension & $512$ \\
Token dimension & $256$ \\
    \bottomrule
\end{tabular}
}
\end{center}
\end{table}

\begin{table}[h]
\caption{
\footnotesize
\textbf{\dsrl and \resrl hyperparameter for \texttt{LIBERO} experiments.}
}
\vspace{5pt}
\label{tab:dsrl_libero_hyperparams}
\begin{center}
\scalebox{0.9}
{
\begin{tabular}{ll}
    \toprule
    \textbf{Hyperparameter} & \textbf{Value} \\
    \midrule
    Learning rate & $0.0003$ \\
    Batch size & $256$ \\
    Activation & Tanh \\
    Target entropy & $0$ \\
    Target update rate ($\tau$) & $0.005$ \\
    Number of actor and critic layers & $3$ \\
    Layer size & $1024$ \\
    Number of critics & $2$ \\
    Number of environments & $1$ \\
    Gradient steps per update $k$ & $20$ \\
    Discount factor & $0.99$ \\
    Initial episode rollouts $N_0$ & $20$ \\
    Reward clipping $\rclip$ & $20$ \\
    Reward scale $\lambda$ & $0.05$ \\
    (\dsrl) Action magnitude $b_{\mathcal{W}}$ & $1.5$ \\
    (\resrl) Residual scale $\gamma$ & $0.01$ \\
    \bottomrule
\end{tabular}
}
\end{center}
\end{table}

\begin{table}[h]
\caption{
\footnotesize
\textbf{Discriminator hyperparameter for \texttt{LIBERO} experiments.}
}
\vspace{5pt}
\label{tab:libero_classifier_hyper}
\begin{center}
\scalebox{0.9}
{
\begin{tabular}{ll}
    \toprule
    \textbf{Hyperparameter} & \textbf{Value} \\
    \midrule
    Number of layers & $3$ \\
    Layer size & $256$ \\
    Batch size & $64$ \\
    Learning rate & $10^{-4}$ \\
    Activation & $\text{ReLU}$ \\
    Optimizer & Adam \\
    Gradient steps per update $K$ & $32$ \\
    Update frequency $T$ & $1000$ \\
    \bottomrule
\end{tabular}
}
\end{center}
\end{table}

For all baseline methods except \textsc{SASR}, we parameterize the reward model using the same architecture as our discriminator and update it at the same frequency during training and used the same training hyperparameters reported in Table~\ref{tab:libero_classifier_hyper}. Similar to \prvm, the reward model takes as input the image representations extracted by the ResNet encoder of the pretrained diffusion policy. Prior to the main experiments, we sweep the reward scale for each baseline over the range $\{0.05, 0.1, 0.15, 0.2\}$ for \prvm, \textsc{SORS}, \textsc{RND}, and \textsc{GAIL-Reward}, and over the range $\{0.25, 0.5, 0.75, 1\}$ for \textsc{SASR}. We find that a reward scale of $0.05$ generally performs best (or yields no significant difference) for the former methods, while a reward scale of $0.5$ works best for \textsc{SASR}. We detail the implementation of each baseline below.

\begin{itemize}
    \item \textsc{SORS}: For each trajectory $\tau$, let $R(\tau)$ denote its discounted return. We train the reward model using the Bradley-Terry preference objective. For trajectory pairs $(\tau_i,\tau_j)\sim\mathcal{D}$ with preference label $y=\mathbb{I}[R(\tau_i) > R(\tau_j)]$, we minimize 
    \[\mathcal{L}(f) = \mathbb{E}_{(\tau_i,\tau_j)\sim\mathcal{D}} \Big[- y \log \sigma(R(\tau_i)-R(\tau_j))- (1-y)\log \sigma(R(\tau_j)-R(\tau_i))\Big].\] 
    \item \textsc{SASR}: We used the publicly available implementation of Ma et al.~\cite{ma2024highly}, using random Fourier features to approximate the Gaussian kernel, with a retention rate of $0.1$ and a KDE bandwidth of $1$.
    \item \textsc{RND}: We instantiate a frozen target network and a trainable predictor network, both following the architecture and update frequency specified in Table~\ref{tab:libero_classifier_hyper}.
    \item \textsc{GAIL-Reward}: We implement \textsc{GAIL-Reward} using the same discriminator architecture and training procedure as \prvm, differing only in the sampling strategy used to construct positive and negative examples.
\end{itemize}

\subsection{Details of RoboCasa Experiments}\label{sec:robocasa_exp_details}

\textbf{Base Policy Training.} We instantiate the base policy $\pi_{\text{pre}}$ as a BC-Transformer policy with a GMM output head, pretrained via behavioral cloning on the \texttt{PnPCounterToCab} task using $50$ demonstrations from the RoboCasa dataset. We use the hyperparameters provided by their codebase without modification~\cite{nasiriany2024robocasalargescalesimulationeveryday}.

\textbf{Residual RL Implementation.} We use the same \resrl implementation as in \Cref{sec:libero_exp_details}. 
Before running the main experiments, we perform a hyperparameter sweep over the number of gradient steps $k$, the process reward scale $\lambda$, and the residual scale $\gamma$. Unless otherwise stated, we use the same hyperparameters as in \Cref{tab:dsrl_libero_hyperparams}, except that the residual scale is set to $0.1$.

\textbf{Discriminator Implementation.} We adopt the same discriminator architecture as in \Cref{sec:libero_exp_details}. The discriminator takes as input the concatenation of the image and proprioceptive representations from the last timestep of the BC transformer's context, the base policy action, and the residual action. We use the same training hyperparameters as before, reported in \Cref{tab:libero_classifier_hyper}.

\subsection{Details of $\pi_0$ Experiments}
We use the public checkpoint \texttt{s3://openpi-assets/checkpoints/pi0\_libero} for our $\pi_0$ experiments and select four tasks for which the base policy has a nonzero success rate: \texttt{Task~20}, \texttt{Task~22}, \texttt{Task~38}, and \texttt{Task~79}. The scene setup and task description for these tasks is shown in Figure~\ref{fig:pi0_scene}. 

\textbf{Discriminator Implementation.} For the $\pi_0$ experiments, we adopt an image-based discriminator. We experiment with two discriminator architectures: a CNN-based model and a pretrained ResNet-based model. We find that the CNN-based architecture generally performs better across most tasks, and therefore use it for all subsequent experiments.

The CNN-based discriminator consists of three convolutional blocks with channel dimensions of $64$, $128$, and $256$. Each block applies a convolution followed by batch normalization and a ReLU activation. The final feature map is projected to a $512$-dimensional embedding, which is concatenated with the action and passed to a three-layer MLP. We additionally sweep the update frequency $T$ and the reward scale $\lambda$, and find that $T=100$ and $\lambda=0.1$ work best for most tasks. We report the hyperparameters for the discriminator and \dsrl in Tables~\ref{tab:pi0_classifier_hyperparam} and~\ref{tab:pi0_hyperparam}, respectively.

\begin{figure}[h]
  \centering
  \begin{minipage}{0.15\textwidth}
    \centering
    \includegraphics[width=\linewidth]{images/scene_images/libero_kitchen_scene_3.png}
  \end{minipage}
  \hspace{0.01\textwidth}
  \begin{minipage}{0.15\textwidth}
    \centering
    \includegraphics[width=\linewidth]{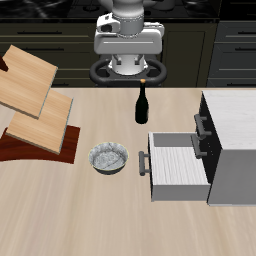}
  \end{minipage}
  \hspace{0.01\textwidth}
  \begin{minipage}{0.15\textwidth}
    \centering
    \includegraphics[width=\linewidth]{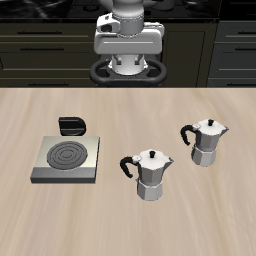}
  \end{minipage}
  \hspace{0.01\textwidth}
  \begin{minipage}{0.15\textwidth}
    \centering
    \includegraphics[width=\linewidth]{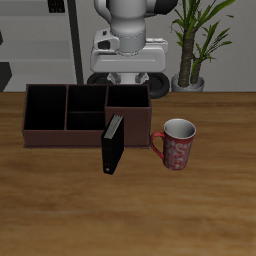}
  \end{minipage}
  \caption{Selected tasks for $\pi_0$ experiments. \texttt{Task 20: Turn on the stove}, \texttt{Task 22: Close the bottom drawer of the cabinet}, \texttt{Task 38: Put the right moka pot on the stove}, \texttt{Task 79: pick up the book and place it in the left compartment of the caddy}.}
  \label{fig:pi0_scene}
\end{figure}

\begin{table}[h]
\caption{
\footnotesize
\textbf{Discriminator hyperparameter for \texttt{LIBERO} $\pi_0$ experiments}
}
\vspace{5pt}
\label{tab:pi0_classifier_hyperparam}
\begin{center}
\scalebox{0.9}
{
\begin{tabular}{ll}
    \toprule
    \textbf{Hyperparameter} & \textbf{Value} \\
    \midrule
    Batch size & $64$ \\
    Learning rate & $10^{-4}$ \\
    Activation & ReLU \\
    Optimizer & Adam \\
    Gradient steps per update $K$ & $8$ \\
    Update frequency $T$ & $100$ \\
    \bottomrule
\end{tabular}
}
\end{center}
\end{table}

\begin{table}[h]
\caption{
\footnotesize
\textbf{\dsrl hyperparameter for \texttt{LIBERO} $\pi_0$ experiments}
}
\vspace{5pt}
\label{tab:pi0_hyperparam}
\begin{center}
\scalebox{0.9}
{
\begin{tabular}{ll}
    \toprule
    \textbf{Hyperparameter} & \textbf{Value} \\
    \midrule
    Number of critics & $10$ \\
    Gradient steps per update & $20$ \\
    Discount factor & $0.99$ \\
    Action Magnitude $b_{\mathcal{W}}$ & $2$ \\
    Reward Scale $\lambda$ & $0.1$ \\
    Initial episode rollouts $N_0$ & $20$ \\
    \bottomrule
\end{tabular}
}
\end{center}
\end{table}

\subsection{Details of Real World Experiments}

For our real-world experiments, we select three tasks:
\begin{itemize}
    \item \texttt{Pick and Place}: The goal is to pick up the corn and place it into the silver pot. 
    \item \texttt{Open Drawer}: The goal is to open the red drawer.
    \item \texttt{Cover Knife with Cloth} task: The goal is to cover the knife on the table with the cloth.
\end{itemize} 
For each task, we use a sparse $0$–$1$ reward and employ \dsrl as the fine-tuning algorithm. The hyperparameters used for \dsrl training are reported in Table~\ref{tab:widowx_hyperparam}.

\textbf{Base Policy Training.} In the WidowX experiments, we again use a diffusion transformer architecture for the base policy~\cite{dasari2024ingredients}. The policy is pretrained on the BridgeData V2 dataset with image-based goal conditioning. The pretraining hyperparameters are provided in Table~\ref{tab:widowx_pretrain_hyperparam}.

\textbf{Discriminator Implementation.} For the real-world experiments, the discriminator follows the same architecture as in the Libero experiments (Table~\ref{tab:libero_classifier_hyper}). The remaining discriminator training hyperparameters are listed in Table~\ref{tab:widowx_hyperparam}.

\begin{table}[!htbp]
\caption{
\footnotesize
\textbf{\dsrl hyperparameter for WidowX experiments.}
}
\vspace{5pt}
\begin{center}
\scalebox{0.9}
{
\begin{tabular}{ll}
    \toprule
    \textbf{Hyperparameter} & \textbf{Value} \\
    \midrule
    Hidden size & $1024$ \\
    Gradient steps per update & $30$ \\
    Discount factor & $0.97$ \\
    Action Magnitude & $2$ \\
    Reward clipping $\rclip$ & $20$ \\
    Reward scale $\lambda$ & $0.1$ \\
    Initial episode rollouts $N_0$ & $20$ \\
    \bottomrule
\end{tabular}
}
\end{center}
\label{tab:widowx_hyperparam}
\end{table}

\begin{table}[!htbp]
\caption{
\footnotesize
\textbf{Base policy pretraining hyperparameters for WidowX experiments.}
}
\vspace{5pt}
\begin{center}
\scalebox{0.9}
{
\begin{tabular}{ll}
    \toprule
    \textbf{Hyperparameter} & \textbf{Value} \\
    \midrule
    Batch size & $2048$ \\
    Learning rate & $0.0003$ \\
    Training steps & $100000$ \\
    LR scheduler & $\text{cosine}$ \\
    Warmup steps & $2000$ \\
    Action chunk size & $1$ \\
    Training denoising steps & $100$ \\
    Inference denoising steps & $8$ \\
    Image encoder & ResNet-34 \\
    Hidden size & $256$ \\
    Number of heads & $1$ \\
    Number of layers & $3$ \\
    Feedforward dimension & $512$ \\
    \bottomrule
\end{tabular}
}
\end{center}
\label{tab:widowx_pretrain_hyperparam}
\end{table}

\begin{table}[!htbp]
\caption{
\footnotesize
\textbf{Discriminator hyperparameter for WidowX experiments.}
}
\vspace{5pt}
\label{tab:widowx_classifier_hyper}
\begin{center}
\scalebox{0.9}
{
\begin{tabular}{ll}
    \toprule
    \textbf{Hyperparameter} & \textbf{Value} \\
    \midrule
    Number of layers & $3$ \\
    Layer size & $1024$ \\
    Batch size & $64$ \\
    Learning rate & $10^{-4}$ \\
    Activation & $\text{ReLU}$ \\
    Optimizer & Adam \\
    Gradient steps per update $K$ & 8 \\
    Update frequency $T$ & $100$ \\
    \bottomrule
\end{tabular}
}
\end{center}
\end{table}

\subsection{Details of Robomimic Experiments}
For \texttt{Robomimic} experiments we follow the \textsc{RLPD} implementation in the \textsc{SERL} codebase~\cite{luo2024serl}. We initialize the buffer with the full set of proficient human demonstrations, which consists of $200$ trajectories.

\textbf{Discriminator Implementation.} Similar to Section~\ref{sec:libero_exp_details}, the discriminator consists of a three-layer MLP that takes state observations and actions as input. We use the same hyperparameter configuration as \prvm, as shown in Table~\ref{tab:libero_classifier_hyper}. \textsc{GAIL} uses the same hyperparameters and differs only in the composition of the training dataset, and does not incorporate the outcome reward. The full set of hyperparameters is provided in Table~\ref{tab:robomimic_hyperparam}.

\begin{table}[!htbp]
\caption{
\footnotesize
\textbf{\textsc{RLPD} hyperparameters for \texttt{Robomimic} experiments.}
}
\vspace{5pt}
\label{tab:robomimic_hyperparam}
\begin{center}
\scalebox{0.9}
{
\begin{tabular}{ll}
    \toprule
    \textbf{Hyperparameter} & \textbf{Value} \\
    \midrule
    Learning rate & $0.0003$ \\
    Batch size & $64$ \\
    Target update rate ($\tau$) & $0.005$ \\
    Critic to actor ratio & $4$ \\
    Number of critics & $2$ \\
    Number of environments & $1$ \\
    Gradient steps per update & $30$ \\
    Discount factor & $0.99$ \\
    Reward clipping $\rclip$ & $20$ \\
    Reward scale $\lambda$ & $0.1$ \\
    \bottomrule
\end{tabular}
}
\end{center}
\end{table}

\subsection{Details of Ablation Experiments}

For all ablation experiments, we follow the same hyperparameter setup as described in Section~\ref{sec:exp_details}.

\subsection{Details of Compute Resource}

For the simulation experiments on \texttt{LIBERO}, \texttt{RoboCasa}, and \texttt{Robomimic}, we use 
NVIDIA RTX A5000 24GB GPU, four CPU cores, and 60GB of CPU memory. Both the \dsrl and \resrl experiments take approximately $10$ hours of wall-clock time to complete per training run. For the real-world experiments, we run each experiment on an NVIDIA GeForce RTX 4090 GPU, taking approximately $1$--$1.5$ hours to complete.

\end{document}